\documentclass{article} 
\usepackage{preprint_style,times}

\usepackage{tikz}

\newcommand{\martiniclass}{%
  \raisebox{-0.3ex}{%
    \begin{tikzpicture}[scale=0.24, baseline=(current bounding box.south), line width=0.7pt]
      \fill[yellow!30!white, opacity=0.4] (-1.3,1.4) -- (1.3,1.4) -- (0,0.2) -- cycle;
      \draw[line width=0.4pt, gray!80] (-1.3,1.4) -- (1.3,1.4); 
      
      \draw (-1.5,1.6) -- (1.5,1.6) -- (0,0.2) -- cycle; 
      \draw (0,0.2) -- (0,-0.8);                         
      \draw (-0.8,-0.8) -- (0.8,-0.8);                   
      
      \draw[line width=0.4pt, brown] (0.2,0.6) -- (1.3,1.7); 
      
      \fill[green!60!black] (0.72,1.25) circle (0.25);
      \fill[red!90!black] (0.77,1.3) circle (0.08); 
    \end{tikzpicture}
  }%
}

\usepackage{algorithm}
\usepackage{algpseudocode}  
\usepackage{wrapfig}
\usepackage{float}

\usepackage[utf8]{inputenc} 
\usepackage[T1]{fontenc}    
\usepackage{hyperref}       
\usepackage{url}            
\usepackage{longtable}
\usepackage{booktabs}       
\usepackage{enumitem}

\usepackage{amsfonts}       
\usepackage{nicefrac}       
\usepackage{microtype}      
\usepackage[table]{xcolor}  
\usepackage{arydshln}

\usepackage{xspace}
\usepackage{lineno}

\usepackage{mathtools}
\usepackage{todonotes}
\usepackage{amssymb,fge}
\usepackage{thm-restate}
\usepackage{wrapfig}
\usepackage{bbm}
\usepackage{mathrsfs}
\usepackage{soul}
\usepackage{array}
\usepackage{multirow}
\usepackage{algorithm}
\usepackage[commentColor=black,beginLComment=/*~, endLComment=~*/]{algpseudocodex}
\usepackage{subcaption}
\usepackage[font=small]{caption}
\usepackage{pifont}
\usepackage{tikz-3dplot}
\usepackage{pgfplots}
\pgfplotsset{compat=newest}
\usepackage{tikzscale}
\usepackage{comment} 
\PassOptionsToPackage{final,pagebackref=true}{hyperref}
\usepackage{hyperref} 
\usepackage{cleveref}
\crefname{figure}{Figure}{Figures}
\crefname{table}{Table}{Tables}
\crefname{equation}{Equation}{Equations}
\usepackage{wasysym}
\usepackage{acronym}
\usepackage{xfrac}
\usepackage{multibib}
\usepackage{float}

\DeclareMathOperator{\diag}{diag}
\newcommand{\PL}{\mathrm{PL}}

\newcommand{\PCC}{\mathrm{PCC}}

\newcommand*{\backrefalt}[4]{%
    \ifcase #1 \footnotesize{(Not cited.)}%
    \or        \footnotesize{(Cited on page~#2)}%
    \else      \footnotesize{(Cited on pages~#2)}%
    \fi}
\newcolumntype{P}[1]{>{\centering\arraybackslash}p{#1}}

\hypersetup{
	colorlinks=true,       
	linkcolor=blue,        
	citecolor=blue,        
	filecolor=magenta,     
	urlcolor=blue         
}

\usepackage{amsmath, amsthm, amsfonts, bm}
\usepackage{mathtools}
\usepackage{thmtools}
\usepackage{thm-restate}
\usepackage{xcolor}
\usepackage{xspace}
\usepackage[framemethod=TikZ]{mdframed}

\newtheorem{proposition}{Proposition}
\newtheorem{corollary}{Corollary}

\newtheorem{lemma}{Lemma}

\newcommand{\gopt}{g^{\mathrm{opt}}}

\def\1{\bm{1}}

\DeclareMathAlphabet{\mathsfit}{\encodingdefault}{\sfdefault}{m}{sl}
\SetMathAlphabet{\mathsfit}{bold}{\encodingdefault}{\sfdefault}{bx}{n}

\def\gD{{\mathcal{D}}}

\def\gG{{\mathcal{G}}}

\def\gL{{\mathcal{L}}}

\def\gN{{\mathcal{N}}}

\def\gX{{\mathcal{X}}}

\def\sP{{\mathbb{P}}}

\newcommand{\E}{\mathbb{E}}

\newcommand{\R}{\mathbb{R}}

\newcommand{\KL}{{\mathrm{KL}}}
\newcommand{\Var}{\mathrm{Var}}

\newcommand{\Cov}{\mathrm{Cov}}

\DeclareMathOperator*{\argmin}{arg\,min}

\newcommand{\xhdr}[1]{{\noindent\bfseries #1}.}
\newcommand{\cut}[1]{}

\newcommand{\namelong}{\textbf{M}ixing \textbf{A}nd \textbf{R}eward \textbf{T}ilting for \textbf{IN}ference-time \textbf{I}mprovement\xspace}
\newcommand{\nameshort}{\textsc{Martini}\xspace}

\mdfdefinestyle{MyFrame}{%
    linecolor=black,
    outerlinewidth=.3pt,
    roundcorner=5pt,
    innertopmargin=1pt,
    innerbottommargin=1pt,
    innerrightmargin=1pt,
    innerleftmargin=1pt,
    backgroundcolor=black!0!white}

\mdfdefinestyle{MyFrame2}{%
    linecolor=white,
    outerlinewidth=1pt,
    roundcorner=1pt,
    innertopmargin=3pt,
    innerbottommargin=2pt,
    innerrightmargin=7pt,
    innerleftmargin=7pt,
    backgroundcolor=black!3!white}

\mdfdefinestyle{MyFrameEq}{%
    linecolor=white,
    outerlinewidth=0pt,
    roundcorner=0pt,
    innertopmargin=0pt,
    innerbottommargin=0pt,
    innerrightmargin=7pt,
    innerleftmargin=7pt,
    backgroundcolor=black!3!white}

\makeatletter
\@ifundefined{assumption}{%
  \theoremstyle{definition}
  \newtheorem{assumption}{Assumption}[section]
  \theoremstyle{plain}
}{}
\makeatother

\crefname{assumption}{Assumption}{Assumptions}
\Crefname{assumption}{Assumption}{Assumptions}

\title{Steered Not Stirred: Effective Reward Tilting using Surrogate Reward Mixing \protect\martiniclass}
\makeatletter
\let\original@fnsymbol\@fnsymbol
\renewcommand{\@fnsymbol}[1]{%
  \ifnum#1=2 **\else\original@fnsymbol{#1}\fi}
\makeatother

\author{%
  \textbf{Alec Sargood}$^{1,2}$\thanks{\texttt{alec.sargood.23@ucl.ac.uk}}
  \quad \textbf{Olivia Gallup}$^{1,3}$
  \quad \textbf{Georg Maierhofer}$^{4}$ \\
  \textbf{Noam Ghenassia}$^{1}$
  \quad \textbf{Paul Duckworth}$^{1}$
  \quad \textbf{Krisztina Sinkovics}$^{1}$\thanks{Joint last authors.}
  \quad \textbf{Avishek Joey Bose}$^{5,6}$\footnotemark[2] \\[1ex]
  {\small
  $^{1}$InstaDeep \quad
  $^{2}$University College London \quad
  $^{3}$University of Oxford} \\
  {\small
  $^{4}$University of Cambridge \quad
  $^{5}$Imperial College London \quad
  $^{6}$Mila--Qu\'{e}bec AI Institute}
}

\iclrfinalcopy 
\begin{document}

\maketitle

\begin{abstract}
\looseness=-1
Continuous-time generative models, from diffusion models, flow matching, and now flow maps, owe their practical value to high-fidelity generation and steerability. In many applications, a ground-truth latent ``oracle'' reward $R^*$ is unavailable or expensive to evaluate, e.g., human judgments or experimental measurements. Although the oracle reward is not directly observable, it can often be approximated through multiple cheaper surrogates that, individually, provide incomplete characterizations but collectively reveal substantial information about the downstream task. In this paper, we present a method for composing such surrogates to steer a pretrained model toward the distribution tilted by an unavailable $R^*$. Given a small feedback set of oracle rankings or numerical rewards, we introduce \namelong (\nameshort), which learns a surrogate composition for plug-and-play use with standard inference-time steering methods. Our theory characterizes surrogate information limits and relates population recovery error to ideal-target KL at small tilt strength. For preferences, the bound additionally accounts for an irreducible mismatch between the Plackett--Luce population target and the conditional mean. Experiments demonstrate improved oracle-tilt fidelity in language generation and ImageNet alignment, and improved joint pose accuracy and validity in protein--ligand cofolding, without inference-time oracle queries.
\end{abstract}

\section{Introduction}
\label{sec:introduction}
\looseness=-1
In many real-world deployments, generative models are judged not only by sample fidelity to the pretraining distribution, but by whether their generations satisfy downstream task-specific objectives~\citep{zhang2026instruction}. Consequently, it remains imperative to align generative models to desired rewards through algorithmic frameworks such as inference-time steering or post-training~\citep{uehara2025inference} that, in effect, sharpen the model distribution on samples with high downstream utility, where that utility is too costly or inaccessible to supervise at pretraining scale. The central discrepancy between pre-training and downstream task alignment, therefore, requires high-quality post-training data under an oracle reward, whose acquisition is often bottlenecked due to high collection costs~\citep{gao2023scaling}, from capturing the nature of human preferences~\citep{rafailov2023direct,ethayarajh2024kto} to experimental measurements needed to evaluate biomolecular design~\citep{pathak2025ai}.

\looseness=-1
In lieu of carte blanche access to a reward oracle that is perfectly predictive of a desired downstream task, practitioners often have access to multiple weakly informative signals in the form of surrogate or proxy reward models. In text-to-image generation, for example, human preferences are approximated by aesthetic, alignment, and learned preference scores such as, ImageReward~\citep{xu2023imagereward}, CLIP~\citep{radford2021learning}, and PickScore~\citep{kirstain2023pick}. In molecular and protein design, physical plausibility checks~\citep{buttenschoen2024posebusters} and machine-learned interatomic potentials~\citep{wood2026family} similarly serve as proxies for physical assays or high-fidelity simulations. In each example, cheap proxy reward signals are individually imperfect and only weakly correlated with the true downstream task, yet collectively they contain substantial information and often complementary evidence for successful candidates. Despite this promise---including the potential to mitigate reward hacking~\citep{coste2024reward,eisenstein2023helping}---existing inference-time steering methods for diffusion and flow-matching models are largely formulated around a single surrogate reward. As a result, a key missing development is a principled framework for composing multiple proxies to recover the latent oracle objective, rather than standard multi-objective optimization that treats surrogates as separate objectives. This motivates a fundamental question: \emph{How can multiple weak reward surrogates be composed to steer a pretrained diffusion or flow model towards the reward-tilted distribution favored by an unavailable oracle?}

\begin{figure}[t]
  \centering
  \includegraphics[width=0.9\textwidth]{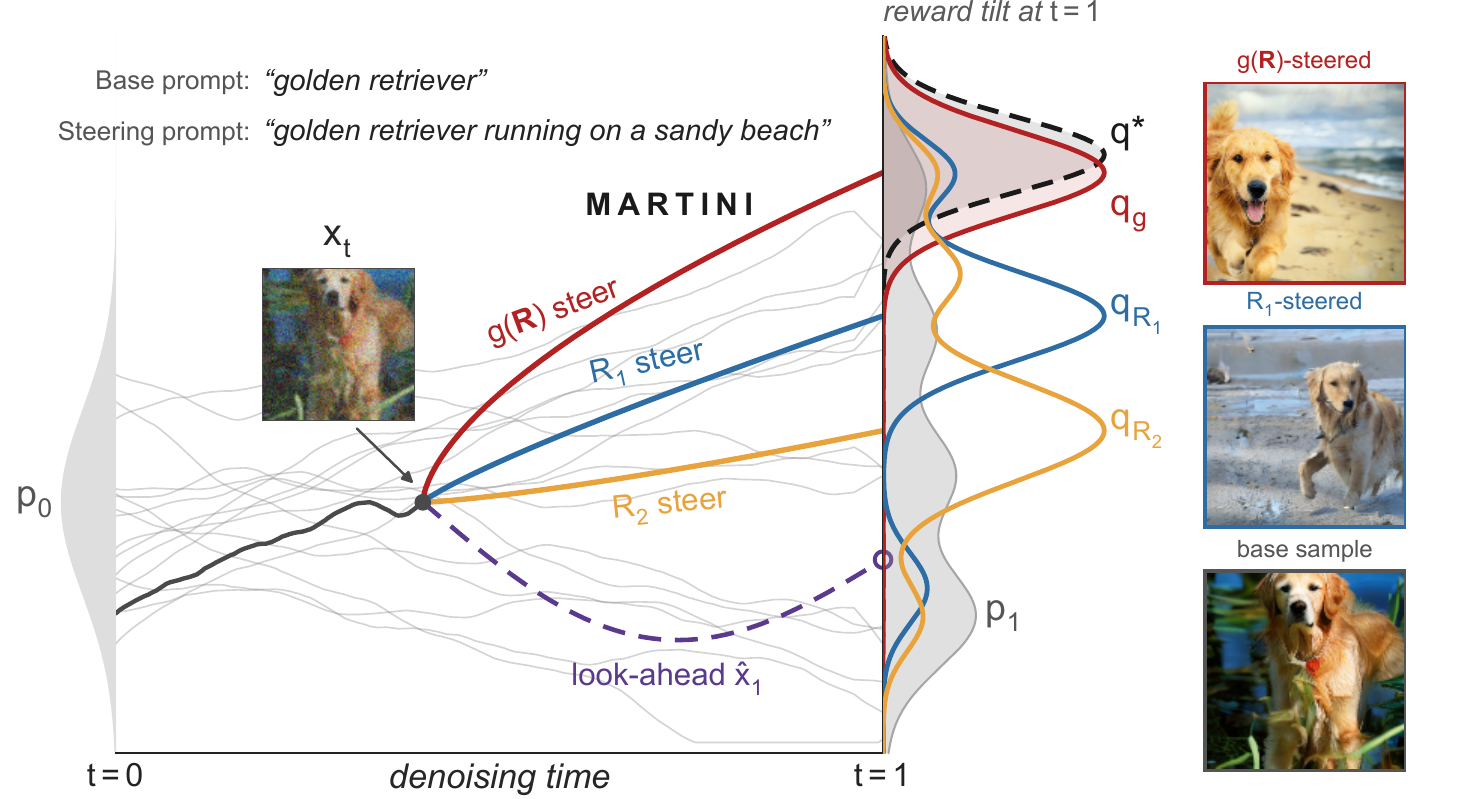}
  \vspace{-5pt}
\caption{\looseness=-1 \small
\nameshort learns a surrogate composition $ g\circ\mathbf R$ from limited offline oracle feedback and uses it to steer toward $q_{ g}$, aiming to approximate the gold reward tilt $q^*$. At inference, \nameshort steers without querying $R^*$.}
  \label{fig:visual-abstract}
  \vspace{-10pt}
\end{figure}
\looseness=-1
\xhdr{Present work}
We study steering toward the reward-tilted distribution of a latent oracle reward, $R^*$, available only on a small annotated feedback set, using multiple cheap surrogates and no query access to $R^*$ at inference. More precisely, we formulate two settings representative of real-world use, in which the feedback set supplies either preference rankings over candidate samples or direct scalar reward values. We introduce \nameshort (\textbf{M}ixing \textbf{A}nd \textbf{R}eward \textbf{T}ilting for \textbf{IN}ference-time \textbf{I}mprovement, \cref{fig:visual-abstract}), which learns a composition of the surrogates from that feedback and allows plug-and-play deployment in existing reward steering frameworks. \nameshort exploits complementary surrogate information to improve steering toward $R^*$, targeting a different solution concept from Pareto optimization. We summarize our main contributions as follows:

\begin{enumerate}[topsep=0pt, partopsep=0pt, itemsep=0pt, parsep=0pt, leftmargin=*]
    \item \looseness=-1 \xhdr{Methods} We develop a novel algorithm in \nameshort that exploits the problem structure of steering towards a gold reward $R^*$, witnessed only on a feedback set, by learning a composition function over cheap reward surrogates that can readily interface with existing steering pipelines. 

    \item \looseness=-1 \xhdr{Theory} We characterize the information limit of the surrogates and identify the exact KL-optimal composition. Conditional-mean regression minimizes recovery error and the leading small-guidance KL coefficient. We also bound reward gain and gap, and relate preference fitting to these quantities through the discrepancy between its population target and the conditional mean.
    
    \item \looseness=-1 \xhdr{Results} We evaluate \nameshort\ on language generation, ImageNet text alignment, and protein--ligand cofolding, as well as toy experiments, spanning FKS~\citep{singhal2025fks}, FMRG~\citep{huang2026fmrg}, MFM~\citep{potaptchik2026mfm}, and It\^o steering~\citep{pan2026itomaps} methods. 
\end{enumerate}

\section{Background and Preliminaries} 
\label{sec:background}
\looseness=-1
\xhdr{Notation}\label{sec:notation}
We denote by $\mathcal X=\mathbb R^d$ the state space of the generative dynamics and by $p_1$ the data distribution on this space, using embedded representations for discrete data. A reward is a scalar function $R:\mathcal X\to\mathbb R$. A superscript asterisk marks the gold reward $R^*$, unqueryable at inference. We use $\mathbf R$ to denote the vector of $k$ surrogate rewards $\mathbf R=(R_1,\dots,R_k):\mathcal X\to\mathbb R^k$, and $\mathbf r=(r_1,\dots,r_k)\in\mathbb R^k$ values it takes. A composition function $g:\mathbb R^k\to\mathbb R$ turns the surrogates into a scalar reward on $\mathcal X$, i.e., $g\circ\mathbf R:x\mapsto g(\mathbf R(x))$ while $\gG$ denotes the function class used to fit the composition. For a reward $R$ and inverse temperature $\beta>0$, we denote $q_R\propto p_1\exp(\beta R)$ as its induced tilt. All such tilts assign zero probability to sets of zero $p_1$-probability. In particular, we define $q^*:=q_{R^*}$ as the gold tilt and  correspondingly $q_g:=q_{g\circ\mathbf R}$ the $g$-composition tilt. $\KL(p\,\|\,q)$ is the Kullback--Leibler divergence, while conditional moments of $R^*$ given $\mathbf R$ are denoted by: $\mu:=\E_{p_1}\!\left[R^*\mid\mathbf R\right]$, $v:=\Var_{p_1}\!\left(R^*\mid\mathbf R\right)$, and $\bar v:=\E_{p_1}[v]$, so that $\mu$ is the best surrogate estimate of $R^*$, $v$ is the variance of $R^*$ that remains once the surrogates are known, and $\bar v$ the expected variance. Specifically, $\bar v$ quantifies the additional information about $R^*$ that the surrogates do not contain.
Finally, SCC and PCC denote the Spearman and Pearson correlation coefficients, respectively.

\subsection{Stochastic Interpolants and Inference Steering}
\label{sec:flowmodel}
\looseness=-1
Stochastic interpolants, such as flow matching~\citep{lipman_flow_2022,liu_rectified_2022,albergo_building_2023} and diffusion models~\citep{ho2020denoising,song2020score}, transport samples from a tractable prior $p_0$, e.g., $\gN(0,I)$, to a data distribution $p_1 \in \sP(\gX)$, typically represented by a dataset $\gD_{\text{train}} = \{ x^i\}_{i=1}^N$ of $N$ samples. These models then construct an interpolant $I_t:=\alpha_t x_1+\sigma_t x_0$, 
that connects independent samples $x_0\sim p_0$, $x_1\sim p_1$ and induces the marginal law $p_t:=\mathrm{Law}(I_t)$. This interpolant can be constructed in different ways: 
%
%
Flow matching regresses a velocity network $u_\phi:\mathbb R^d\times[0,1]\to\mathbb R^d$ against the interpolant velocity $\dot I_t=\dot\alpha_t x_1+\dot\sigma_t x_0$. Diffusion models learn endpoint denoisers $\hat x_1(x_t,t):=D_\phi(x_t,t)\approx\mathbb E[x_1\mid I_t=x_t]$, or implicitly parametrize the Stein score $\nabla \log p_t(x_t)$. 
Sampling can then be performed by integrating the probability flow ODE $\dot x_t=u_\phi(x_t,t)$, or by simulating an SDE with the same marginals and terminal law $p_{1,\phi}$, respectively:
\begin{equation}
\label{eq:reverse_diffusion}
    \mathrm{d}x_t=\left[\,u_\phi(x_t,t)+\tfrac{c_t^2}{2}\,\nabla\!\log p_t(x_t)\,\right]\mathrm{d}t+c_t\,\mathrm{d}W_t, \quad x_0 \sim p_0,
\end{equation}
\looseness=-1
where $c_t > 0$  is the diffusion coefficient and $W_t$ is the standard Wiener process on $\R^d$. 

\looseness=-1
\xhdr{Steering}
Given a perfectly trained diffusion model with terminal law $p_{1,\phi} \overset{d}{=}p_1$, inference-time steering seeks to target the reward-tilted $q_R$ without further retraining. SDE methods achieve this by adding a Doob  $h$-transform drift $\beta c_t^2\,\nabla_{x_t}V_t$ with $V_t(x_t)=\beta^{-1}\log\E[\exp(\beta R(x_1))\mid x_t]$, while sequential Monte-Carlo schemes~\citep{doucet2001introduction} propagate particles and resample by look-ahead reward potentials $R(\hat{x}_1 (x_t))$. Finally, few-step optimal-control methods treat the velocity as a control and modify the velocity by pushing samples along $\nabla R$, also using a look-ahead reward estimate.

\looseness=-1
\xhdr{Flow maps}
Rather than integrating the probability-flow ODE step-by-step, one can train a flow map, $X:[0,1]^2 \times \mathcal X\to\mathcal X$, $X_{t,s}(x):=X(t,s,x)$. A flow map jumps a state from time $t$ to $s$ along the ODE trajectory, $X_{t,s}(x_t)=x_s$, and is naturally parametrized via the average velocity $u_{t,s}$ for $s\ne t$, $X_{t,s}(x_t)=x_t+(s-t)\,u_{t,s}(x_t)$, with $u_{t,s}$ given by: $u_{t,s}(x_t)=(s-t)^{-1}\int_t^s u_\phi(x_\tau,\tau)\,\mathrm d\tau$. As $s\to t$ the average velocity recovers the instantaneous one, $\lim_{s\to t}u_{t,s}(x_t)=u_\phi(x_t,t)$. We can distill flow maps from a pretrained model, or directly from data, using Lagrangian, Eulerian, or Progressive distillation or losses. The one-step map $\hat x_1(x_t)=X_{t,1}(x_t)$ provides a fast look-ahead to the endpoint, so we can evaluate a reward on a predicted sample without integrating the remaining trajectory.


\looseness=-1
\xhdr{Continuous treatment of discrete data}
Let $\mathcal V$ be a vocabulary and $L$ a sequence length, such that token sequences satisfy $y\in\mathcal V^L$. Here, $\mathcal X=\mathbb R^{L\times|\mathcal V|}$, where $d=L|\mathcal V|$, and we use tokenwise one-hot encoder $\mathrm{Enc}:\mathcal V^L\to\mathcal X$ with the tokenwise argmax decoder $\mathrm{Dec}:\mathcal X\to\mathcal V^L$~\citep{lee2026flow}, such that $p_1$ is the distribution of one-hot embeddings $\mathrm{Enc}(\mathcal V^L)\subset\mathcal X$. Since $\mathrm{Dec}(\mathrm{Enc}(y))=y$, this representation preserves the text distribution exactly. We apply the preceding interpolant and flow-map construction to $p_1$, with $x_1=\mathrm{Enc}(y)$ and $x_0\sim\mathcal N(0,I)$. The intermediate marginals $p_t$ have smooth densities and converge in distribution to $p_1$ as $t\to1$. Learned continuous terminal predictions are then decoded by tokenwise argmax. Rewards on one-hot encodings are defined by the corresponding text rewards.

\subsection{Preference modelling}
\label{sec:preference}
\looseness=-1
Preference learning aims to fit a latent utility $h$ from relative rankings rather than scalar-valued scores (as in, e.g., regression tasks), by minimizing a negative log-likelihood (NLL). The Bradley--Terry (BT) model~\citep{bradley1952rank} models pairwise rankings $p(x_i\succ x_j)=\sigma(h(x_i)-h(x_j))$, with $\sigma$ the logistic sigmoid function. The Plackett--Luce (PL) model~\citep{plackett1975analysis,luce1959individual} extends this to a full ranking $\pi$ of a candidate set $S$ of size $m$, decomposed into a sequence of softmax choices over the items not yet selected, $p_{\PL}(\pi\mid h,S)=\prod_{i=1}^{m-1}\exp(h\left(x_{\pi(i)}\right))\big/\sum_{i'=i}^{m}\exp\left(h\left(x_{\pi(i')}\right)\right)$, recovering BT at $m=2$ and is invariant to additive constants. Such models are pervasive in LLM post-training: annotators rank responses to a prompt, and the goal is to align the model to human judgements of quality using RL or direct preference optimization~\citep{schulman2017proximal,rafailov2023direct}.


\section{Method}
\looseness=-1
\nameshort\ learns a composition of fixed surrogates from
an offline feedback set $\gD_{\text{feedback}}$, distinct
from the generator's training data $\gD_{\text{train}}$.
Feedback consists of numerical oracle rewards or rankings;
the fitted composition drives inference-time steering
without oracle queries.

\subsection{Problem Setup}
\label{sec:setup}

\looseness=-1
At a given tilt strength $\beta>0$, we seek a surrogate tilt $q_g$
approximating $q^*\propto p_1\exp(\beta R^*)$, using $k$ cheap surrogates stacked together as $\mathbf R$. We reiterate that $R^*$ is not available at inference; instead, we only have knowledge of the gold reward through a feedback set $\gD_{\text{feedback}}$. We seek $\gopt\in\argmin_g\KL(q^*\|q_g)$ over measurable compositions with finite, positive tilt normalizers. In practice, we fit $g\in\gG$ from $\gD_{\mathrm{feedback}}$. Performance depends on the surrogate information, composition class, and feedback budget. We consider two feedback regimes.

\looseness=-1
\xhdr{(i) Value feedback} $R^*$ returns scores: $\mathcal D_{\text{feedback}}=\{(x_i,R^*(x_i))\}_{i=1}^n$, where $\mathcal D_{\text{feedback}}$ represents a budget of expensive evaluations. The feedback is cardinal, namely the scale of $R^*$ is known.

\looseness=-1
\xhdr{(ii) Preference feedback} $R^*$ is seen only through order: $\mathcal D_{\text{feedback}}=\{(S^{(j)},\pi^{(j)})\}_{j=1}^M$ comprises $M$ query groups of $m\geq2$ samples, $S^{(j)}\sim p_1^{\otimes m}$, and $\pi^{(j)}$ is the preferred order. Rankings are drawn independently across groups from the PL model of~\S\ref{sec:preference} at the gold reward score:
\begin{equation}
  \label{eq:pl-oracle}
  \pi^{(j)}\sim p^*_{\PL}(\cdot\mid S^{(j)}),
  \qquad
  p^*_{\PL}(\pi\mid S):=p_{\PL}(\pi\mid R^*,S)
  =\prod_{i=1}^{m-1}\frac{\exp(R^*\left(x_{\pi(i)}\right))}{\sum_{i'=i}^{m}\exp\left(R^*\left(x_{\pi(i')}\right)\right)}.
\end{equation} 
Equivalently, by the Gumbel-sort construction~\citep{yellott1977relationship}, $\pi^{(j)}$ is obtained by perturbing each candidate's oracle score with independent standard Gumbel noise and sorting in descending order. By the shift invariance of $p_{\PL}$, preferences constrain $R^*$ only up to an additive constant.

\looseness=-1
\xhdr{Relation to multi-objective optimization}
In multi-objective optimization (MOO), the $R_j$ are
distinct objectives. For maximized rewards, $x$ dominates
$x'$ if $R_j(x)\geq R_j(x')$ for every $j$, with at least
one strict inequality. Undominated feasible candidates form
the Pareto set, where their reward vectors form the Pareto front, and scalarizations express trade-offs among the objectives.

\looseness=-1
Our surrogates instead provide imperfect information about
one oracle $R^*$. A candidate can dominate another on every
surrogate yet receive a lower oracle reward, so the
surrogate Pareto set can exclude the oracle
optimum~(\cref{lem:monotone_g,lem:offront,fig:moo-comparison}).
We therefore allow nonmonotone compositions, although no
composition can recover information absent from its inputs.
Further discussions are in~\S\ref{app:pareto}.

\begin{figure}[htbp]
    \centering
    \includegraphics[width=0.9\linewidth]{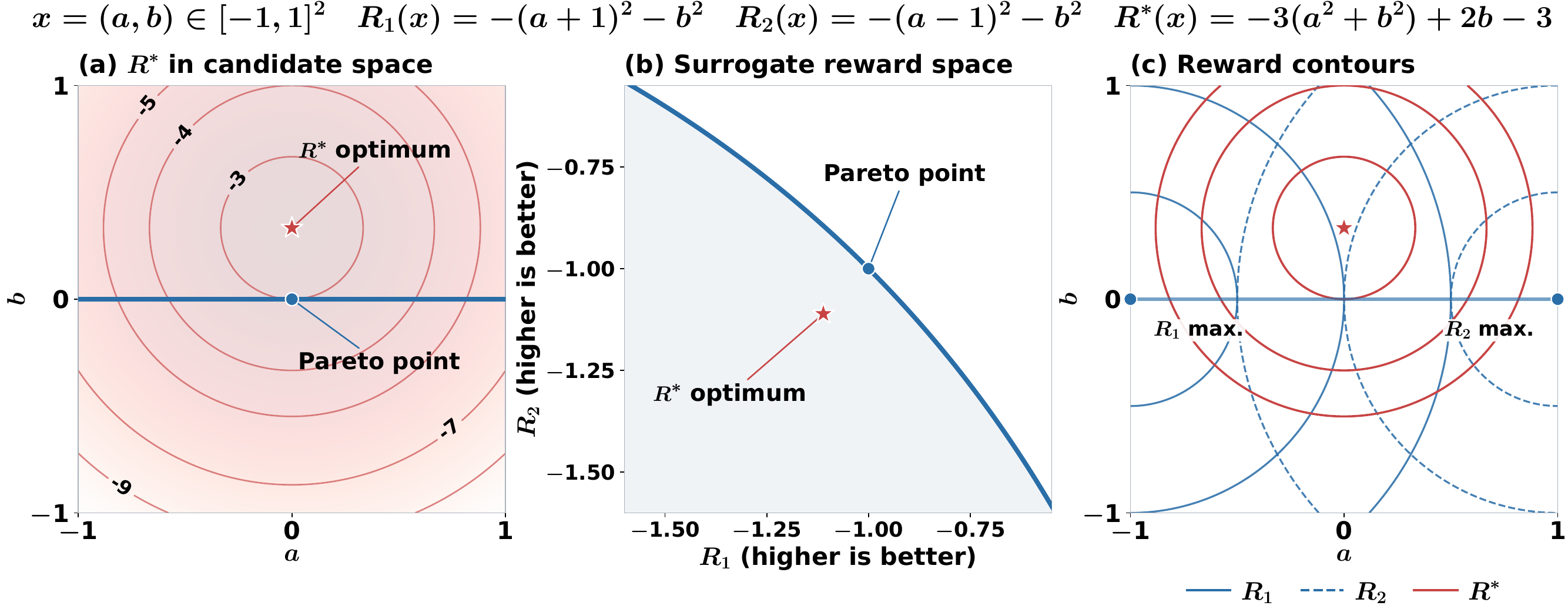}
    \vspace{-10pt}
    \caption{ \small \looseness=-1
    \textbf{The oracle optimum can be dominated by the surrogates.}
    \textbf{(a)} Contours of $R^*$ in candidate space. Its maximizer
    lies outside the Pareto set.
    \textbf{(b)} Visualization in surrogate reward space, Pareto front in blue. The Pareto point scores
    higher on both surrogates but lower on $R^*$.
    \textbf{(c)} Visualization of surrogate reward and $R^*$ contours.}
    \vspace{-10pt}
    \label{fig:moo-comparison}
\end{figure}

\subsection{Training \nameshort}
\label{sec:practical}
\looseness=-1
Given $\gD_{\text{feedback}}$ drawn from $p_1$,
\nameshort learns a composition $g$ offline using a value or
preference loss, with the map $g\circ\mathbf R$ compatible with any inference steering approach without querying $R^*$.

\looseness=-1
\xhdr{Value feedback}
For numerical evaluations of $R^*$, the fitting objective is the mean squared error:
\begin{equation} 
\label{eq:mse-loss}
  \gL_{\mathrm{MSE}}(g)=\frac1n\sum_{i=1}^n\left(g(\mathbf R(x_i))-R^*(x_i)\right)^2 \quad x_i \sim \gD_{\text{feedback}}.
\end{equation}
\looseness=-1
Over unrestricted measurable compositions, the population counterpart of \cref{eq:mse-loss} is minimized by $g\circ\mathbf R=\mu$, a quantity whose steering advantage to $q^*$ we theoretically establish in~\S\ref{sec:theory}.

\looseness=-1
\xhdr{Preference feedback}
For rankings, we fit the composition by minimizing the PL NLL. Let $x_{(i)}^{(j)}$ denote the candidate ranked $i$th in group $j$, with the most preferred candidate first. The objective is
\begin{equation}
\label{eq:pl-loss}
\gL_{\PL}(g)
=-\frac1M\sum_{j=1}^{M}\sum_{i=1}^{m-1}
\left[
g\!\left(\mathbf R(x_{(i)}^{(j)})\right)
-\log\!\sum_{l=i}^{m}
\exp\!\left(g\!\left(\mathbf R(x_{(l)}^{(j)})\right)\right)
\right].
\end{equation}
\looseness=-1
\Cref{eq:pl-loss} fits the observed rankings without requiring numerical gold rewards. Its population target need not equal $\mu$, whose implications for recovery and steering are discussed in~\S\ref{sec:theory}.

\looseness=-1
\xhdr{Constructing the composition}
We consider three ways to construct the composition: correlation weighting, fitted linear models, and MLPs. Linear and MLP compositions are trained using the appropriate empirical loss with family-specific regularization, while correlation weights are computed directly from Pearson correlations with observed oracle values or Spearman correlations with observed oracle rankings. Correlation and linear models describe a weighted sum of surrogate scores, while MLPs can capture nonlinear relationships. Derivations and implementation details are provided in~\S\ref{app:learn}.

\looseness=-1
\xhdr{Inference-time steering}
After fitting, we hold $g$ fixed and use $g\circ\mathbf R$ to steer toward $q_g$. Adding the Doob drift correction to~\eqref{eq:reverse_diffusion} gives:
\begin{equation}
\label{eq:composed-steering}
\mathrm d x_t = \left[ u_\phi(x_t,t) +\frac{c_t^2}{2}\nabla\log p_t(x_t) +\beta c_t^2\nabla_{x_t}V_t^{g}(x_t) \right]\mathrm dt +c_t\,\mathrm dW_t,
\end{equation}
\looseness=-1
where $V_t^g(x_t) =\frac1\beta\log\E\!\left[ \exp\!\left(\beta g(\mathbf R(x_1))\right) \mid x_t\right]$. In practice, an endpoint prediction $\hat x_1 := D_\phi(x_t,t)$ gives the look-ahead approximation. In practice, we use the look-ahead reward
$g(\mathbf R(\hat x_1(x_t)))$. When the composition, surrogates, and endpoint predictor are differentiable, its gradient follows by the chain rule (c.f,~\cref{alg:martini} for details).

\subsection{Theoretical Results}
\label{sec:theory}
\looseness=-1
Our setting raises two questions: \emph{\textbf{Q1.} when does learning to predict $R^*$ from surrogates lead to useful steering}, and \emph{\textbf{Q2.} how well can any composition of the surrogates do?} The objectives in~\cref{eq:mse-loss,eq:pl-loss} use oracle feedback, whereas our goal is to approximate $q^*$. Our results connect these questions and identify the limits imposed by information about $R^*$ missing from the surrogates. 

\looseness=-1
We examine four quantities: \emph{Steering divergence}, $\KL(q^*\|q_g)$, measures fidelity to the oracle tilt, and \emph{reward gain} and \emph{reward gap} compare the induced distribution's expected $R^*$ reward with that of $p_{1}$ and $q^*$ respectively. Finally, \emph{recovery error} measures how accurately $g\circ\mathbf R$ predicts $R^*$ under $p_1$.

\looseness=-1
We analyze population quantities and ideal tilts under the assumptions stated
in~\S\ref{app:assumptions}. The KL and reward expansions in the theorems hold for fixed compositions as $\beta\downarrow0$, while the recovery identities and bounds are exact. For both theorems, we use mean-aligned representatives (Assumption~\ref{ass:means}), which preserve the induced tilt. For conditional generation, i.e. $p_1(\cdot|c)$ we fix the context $c$ in the analysis and drop the notation for brevity. 

\looseness=-1
\xhdr{Value feedback}
We write $\mathrm{MSE}(g):=\E_{p_1}[(g\circ\mathbf R-R^*)^2]$ for population recovery error and $\omega_R:=\exp(\beta R)/\E_{p_1}[\exp(\beta R)]$ for a tilt's relative weight.
\begin{mdframed}[
    style=MyFrame2,
    innertopmargin=3pt,
    innerbottommargin=3pt
]
\begin{restatable}[Value Regime]{theorem}{thmvalue}
\label{thm:value}
\setlength{\abovedisplayskip}{6pt}
\setlength{\belowdisplayskip}{6pt}
\setlength{\abovedisplayshortskip}{1pt}
\setlength{\belowdisplayshortskip}{3pt}

Suppose $R^*$ and $g\circ\mathbf R$ satisfy
Assumptions~\ref{ass:moments}--\ref{ass:means}. Then:
\begin{enumerate}[
    label=(\roman*),
    leftmargin=2em,
    itemsep=2pt,
    topsep=1pt,
    parsep=0pt
]
\item \emph{(Steering.)} As $\beta\downarrow0$, with
Assumption~\ref{ass:means} used for the second equality,
\[
\KL(q^*\|q_g)
=\frac{\beta^2}{2}\Var_{p_1}(R^*-g\circ\mathbf R)
+O(\beta^3)
=\frac{\beta^2}{2}\mathrm{MSE}(g)+O(\beta^3).
\]

\item \emph{(Reward gain.)}
For any $R\in L^\infty(p_1)$ and $\beta>0$:
$\E_{q_R}[R^*]-\E_{p_1}[R^*]
=\Cov_{p_1}(R^*,\omega_R)$.
Furthermore, as $\beta\downarrow0$:
\[
\E_{q_g}[R^*]-\E_{p_1}[R^*]
\geq\beta\left[
\Var_{p_1}(\mu)
-\sqrt{
\Var_{p_1}(\mu)(\mathrm{MSE}(g)-\bar v)
}
\right]+O(\beta^2).
\]

\item \emph{(Reward gap.)} As $\beta\downarrow0$,
$\E_{q^*}[R^*]\!-\!\E_{q_g}[R^*]
\!\leq\!\beta[\bar v
+\sqrt{\Var_{p_1}(\mu)(\mathrm{MSE}(g)-\bar v)}]
\!+\!O(\beta^2)$.
\setlength{\itemsep}{5pt}
\item \emph{(Best recovery.)}
Under Assumption~\ref{ass:moments} alone,
\[
\mathrm{MSE}(g)
=\bar v+\|g\circ\mathbf R-\mu\|^2_{L^2(p_1)}
\geq\bar v,
\qquad
\Var_{p_1}(R^*)=\Var_{p_1}(\mu)+\bar v.
\]
\end{enumerate}
\end{restatable}
\end{mdframed}

\looseness=-1
Part~(i) answers \textbf{Q1} by connecting the value-fitting objective directly to steering: reducing population MSE reduces the leading KL error. This also explains why the unrestricted regression target $\mu$ is useful even though the exact KL-optimal composition $\gopt$ can differ from it. Proposition~\ref{prop:klopt} and Corollary~\ref{cor:mu} show that $\KL(q^*\|q_\mu)-\KL(q^*\|q_{\gopt})=O(\beta^4)$ as $\beta\downarrow0$. Thus, using $\mu$ instead of $\gopt$ incurs only fourth-order excess KL, providing a distributional justification for MSE fitting in this regime.

\looseness=-1
Part~(iv) answers \textbf{Q2}: fitting can reduce the discrepancy from $\mu$, but recovery error remains at least $\bar v$, the variation in $R^*$ unexplained by the surrogates. Improving the fit of $g$ therefore addresses only one source of error, and reducing this floor requires more informative surrogate inputs (Corollary~\ref{cor:mono}). 

\looseness=-1
Parts~(ii)--(iii) relate recovery error to expected reward after steering under $R^*$. At leading order, reducing recovery error raises the gain lower bound and lowers the gap upper bound.  When
$\Var_{p_1}(\mu)>0$, sufficiently accurate approximation of $\mu$ guarantees positive gain for small enough $\beta>0$, even if $\bar v>0$ (Corollary~\ref{cor:positive-gain}). Useful steering therefore does not require perfect oracle recovery. The exact gain identity in part~(ii) appears concurrently in~\citet{paes2026theoretical}, while here we relate it to learned surrogate compositions through asymptotic gain and gap bounds.

\looseness=-1
\xhdr{Preference feedback} 
When numerical oracle rewards are unavailable, we instead ask whether fitting oracle rankings controls the same recovery and steering quantities. 

\begin{mdframed}[style=MyFrame2]
\begin{restatable}[Preference Regime]{theorem}{thmpref}
\label{thm:pref}
Under Assumptions~\ref{ass:moments}, \ref{ass:means},
and~\ref{ass:preferences}, define the population loss
$\PL(g):=\E_{S,\pi}[-\log p_{\PL}(\pi\mid g\circ\mathbf R,S)]$
and let $g_{\PL}\in\argmin_{g\in\gG}\PL(g)$.
With the score-range bound $\Delta$ from~\ref{ass:preferences} and
$c_\Delta=(1-1/m)e^{-\Delta}$, every $g\in\gG$
satisfies the bound
\[
\mathrm{MSE}(g)\leq B(g):=\bar v+
\left(
\sqrt{\Var_{p_1}(g_{\PL}\circ\mathbf R-\mu)}
+\sqrt{\frac{2}{c_\Delta}
\bigl(\PL(g)-\PL(g_{\PL})\bigr)}
\right)^2.
\]
Parts~(i)--(iii) further assume the oracle bound in
Assumption~\ref{ass:bounded} and hold for fixed $g$ as
$\beta\downarrow0$:
\begin{enumerate}[
    label=(\roman*),
    leftmargin=2em,
    itemsep=1pt,
    topsep=1pt
]
\item \emph{(Steering.)}
$\ \ \ \ \  \KL(q^*\|q_g)\leq\frac{\beta^2}{2}B(g)+O(\beta^3)$.

\item \emph{(Reward gain.)}
$\ \ \ \ \ 
\E_{q_g}[R^*]-\E_{p_1}[R^*]
\geq\beta\left[
\Var_{p_1}(\mu)
-\sqrt{\Var_{p_1}(\mu)(B(g)-\bar v)}
\right]+O(\beta^2).
$
\item \emph{(Reward gap.)}
$
\ \ \ \ \  \E_{q^*}[R^*]-\E_{q_g}[R^*]
\leq\beta\left[
\bar v+\sqrt{\Var_{p_1}(\mu)(B(g)-\bar v)}
\right]+O(\beta^2).
$

\item \emph{(Minimum bound.)}
$\ \ \ \ \ B(g)\geq B(g_{\PL})=\mathrm{MSE}(g_{\PL})$.
\end{enumerate}
\end{restatable}
\end{mdframed}

\looseness=-1
The empirical objective~\eqref{eq:pl-loss} estimates $\PL(g)$, whose excess controls recovery through $B(g)$. This bound combines the information floor $\bar v$, the discrepancy between $g_{\PL}\circ\mathbf R$ and $\mu$, and excess PL loss. For a fixed problem and class, reducing excess loss tightens the leading KL, gain, and gap, bounds in parts~(i)--(iii), justifying preference fitting without numerical oracle rewards. At zero excess loss, the discrepancy remains unless $g_{\PL}\circ\mathbf R=\mu$, in which case $B(g_{\PL})=\bar v$~(iv). Thus $g_{\PL}$ minimizes the bound, without necessarily minimizing recovery error or exact steering divergence. The guarantee requires a convex class, a uniform score-range bound $\Delta$, and an attained population minimum (Assumption~\ref{ass:preferences}); nonconvex classes require separate analysis. Proofs and supporting results appear in~\S\ref{app:extra_theory}--\ref{app:proofs}.

\section{Experiments}
\label{sec:experiments}
\looseness=-1
We evaluate \nameshort\ on language generation, text-to-image alignment, and protein--ligand cofolding, fitting $g$ offline and steering without inference-time oracle queries, with synthetic experiments in \S\ref{app:toy}. 

\looseness=-1
\xhdr{Evaluation}
For language and images, we assess recovery using held-out PCC and MSE for values, or SCC and PL loss for preferences. Experimental tilts use the sampler law $p_{1,\phi}$ in place of $p_1$. We evaluate fixed finite target tilts $\beta=1$ and sweep sampler guidance separately, probing behavior beyond our small-$\beta$ guarantees. Composition fidelity is measured by $\KL(q^*\|q_g)$, estimated from unsteered samples. We assess generated samples using embedding-space energy distance and distances on the $R^*$ marginal, including Wasserstein distances and Gaussian and KDE estimates of KL. We also report the absolute reward gap and a diversity-adjusted cost $Q$ (prior: $1$ and lower is better), which evaluates reward gains accompanied by diversity loss. Definitions and estimators appear in~\S\ref{app:metric_est}. Specific protein cofolding metrics are described in~\S\ref{subsec:cofolding}.

\xhdr{Language and Image Baselines}
Best-of-$N$ (BoN), reported in Appendix ablation tables, selects the highest-$R^*$ candidate, requiring $N$ oracle evaluations per sample. Deployable steering baselines are the best single surrogate (BSS), uniform weighting, and ParetoFlow~\citep{yuan2025paretoflow}, plus ReNO's fixed reward weights for images~\citep{eyring2024renoenhancingonesteptexttoimage}. BSS is selected on training feedback using either PCC or SCC depending on regime. We also evaluate SLOP's calibrated selection protocol~\citep{wang2026temper}. SLOP requires $R^*$ values, so its preference results serve only as a value-calibrated reference.

\begin{figure}[htbp]
  \centering
  \includegraphics[width=\textwidth]{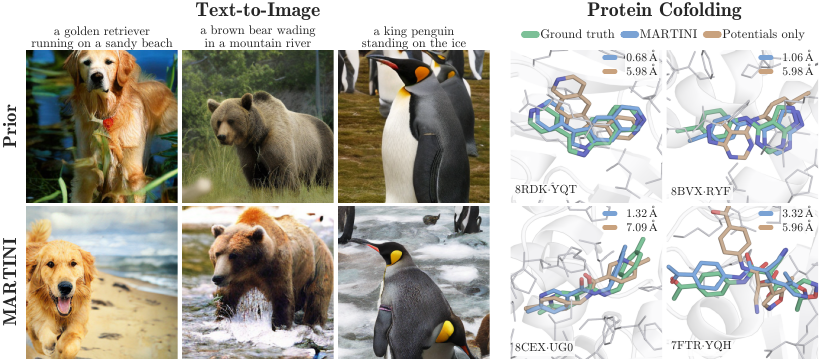}
  \vspace{-15pt}
    \caption{\small\looseness=-1 \textbf{\nameshort\ steering.} Left: ImageNet prior samples (top) and steered samples (bottom). Right: Predicted and reference ligand poses. Annotations give ligand RMSD in \AA{}.}
\label{fig:qualitative}
\end{figure}

\subsection{Language Generation}
\looseness=-1
We use FMRG to steer the released flow map language model (FMLM)~\citep{lee2026flow}, trained on One Billion Words~\citep{chelba2013one}. For fitting and evaluation, we generate unconditional sequences of $128$ tokens in $8$ steps using the \texttt{bert-base-uncased} vocabulary.

\looseness=-1
\xhdr{Setup}
We use the negative generative perplexity under Qwen2.5-1.5B~\citep{qwen2025qwen25technicalreport} as the oracle,
$R^*=-\mathrm{gPPL}$. Our four differentiable surrogates are
BERT~\citep{devlin2019bert} and DistilBERT~\citep{sanh2019distilbert} scores, CoLA acceptability~\citep{warstadt2019cola}, and empirical unigram entropy $H_{\mathrm{tok}}$. Entropy supplies a diversity signal to address repetitive outputs under gPPL minimization. Feedback comprises $n=256$ numerical rewards or $M=256$ rankings of $m=8$ candidates. We generate $256$ samples per method, regime, and seed,
and average metrics over five seeds.

\looseness=-1
\xhdr{Results}
Held-out loss selects an MLP for values and a linear composition for preferences~(\cref{fig:lang-recovery}), and both attain the lowest reported composition KL~(\cref{fig:lang-klg}). At steering strength $0.03$, both outperform the displayed deployable baselines in energy distance and $Q$~(\cref{fig:lang}). Increasing BSS guidance (the same BSS was chosen in both regimes) worsens reward-marginal KL and $Q$, with both exceeding the prior. Stronger \nameshort\ guidance improves reward-marginal KL but eventually raises energy distance above the prior: closer reward-marginal matching need not improve embedding-space fidelity. Additional results and generated examples appear in~\cref{tab:lang-full,tab:lang-qualitative}.

\begin{figure}[t]
  \centering
  \includegraphics[width=\linewidth]{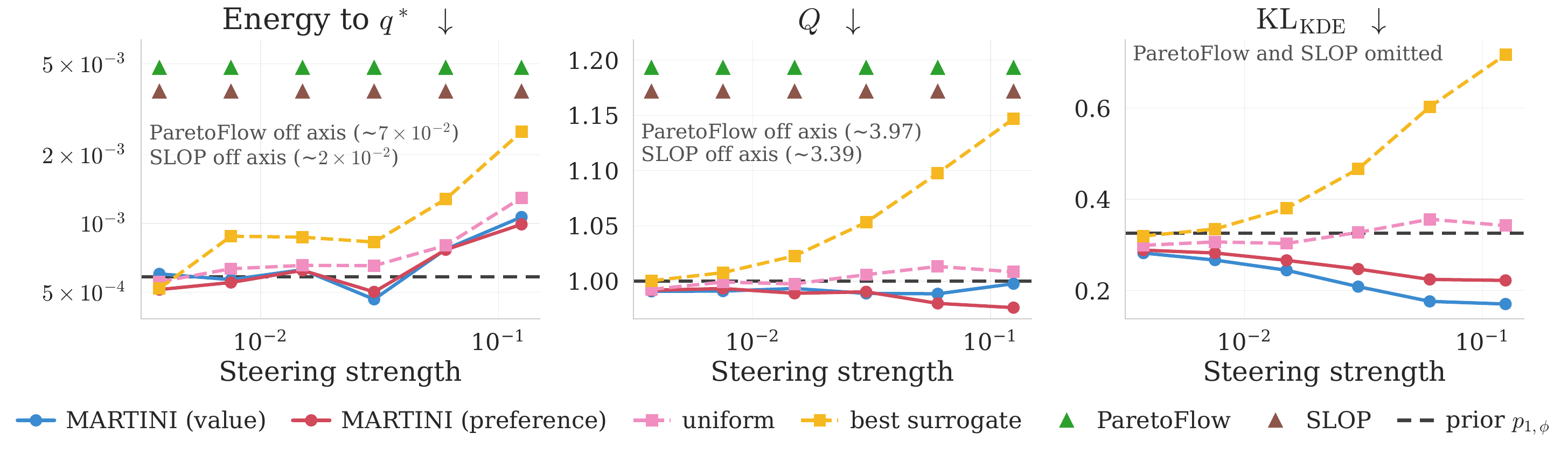}
  \vspace{-20pt} \caption{\small\looseness=-1 \textbf{Language steering.} Lower is better for all metrics. Triangles indicate values beyond the plotted range.} \label{fig:lang}
\end{figure}

\subsection{Text-to-Image}

\looseness=-1
We use FMRG~\citep{huang2026fmrg} to steer a distilled, class-conditioned ImageNet flow map toward text prompts, fitting a separate composition for each class--prompt pair, as shown in~\cref{fig:qualitative} (left).

\looseness=-1\looseness=-1
\xhdr{Setup}
Our surrogates are HPSv2~\citep{wu2023human}, ImageReward~\citep{xu2023imagereward}, PickScore~\citep{kirstain2023pick}, and CLIP~\citep{radford2021learning}, matching ReNO's reward set~\citep{eyring2024renoenhancingonesteptexttoimage}. The oracle $R^*$ measures image--prompt alignment using
Skywork-VL-Reward 7B~\citep{wang2025skywork}. Feedback comprises $n=32$ scalar values or $M=128$ rankings. Recovery uses all $19$ class--prompt pairs while the main steering
comparison uses seven with reference $\mathrm{ESS}>1000$
in both regimes. For each pair and method, sample metrics
use the lowest-energy evaluated strength. All-pair results
and cutoff sensitivity appear in~\S\ref{app:ablations}, and implementation details in~\S\ref{app:in_details}.

\looseness=-1
\xhdr{Results}
Held-out loss selects a linear composition in both regimes.
On these seven pairs, \nameshort\ achieves the lowest
embedding-space and reward-marginal distances and absolute
reward gap in both regimes~(\cref{tab:in-combined}).
With values, $\mathrm{KL}_{\mathcal N}$ and
$\mathrm{KL}_{\mathrm{KDE}}$ are reduced by factors of
$4.9$ and $3.4$, respectively, relative to the best baselines. Only uniform weighting achieves lower $\mathrm{KL}_g$ ($0.436$ versus $0.466$). With preferences, \nameshort\ also achieves the lowest $\mathrm{KL}_g$ ($0.423$). Both selected compositions also outperform BSS, uniform weighting, and ReNO in mean energy at shared strengths $\lambda\in\{2,3,5\}$~(\cref{tab:in-strength-sweep}). Full sweeps and additional feedback budget comparisons appear in~\S\ref{app:ablations}. Qualitative results are in~\S\ref{app:qual-imagenet}.

\definecolor{martinival}{RGB}{219,234,249}
\definecolor{martinipref}{RGB}{176,206,240}
\definecolor{oraclerow}{RGB}{253,229,203}

\providecommand{\INcell}[2][]{%
  \ensuremath{#2}\makebox[1.1em][l]{\ensuremath{{}#1}}%
}
\providecommand{\INhead}[1]{%
  #1\makebox[1.1em]{}%
}

\begin{table}[t]
\centering
\caption{\small\looseness=-1\label{tab:in-combined}
\textbf{ImageNet steering.}
Bold/underline mark best/second-best deployable results
within each regime. On shared rows, $*$/$\dagger$ denote
values/preferences, with superscripts/subscripts
indicating first/second.}
\footnotesize
\setlength{\tabcolsep}{3pt}
\renewcommand{\arraystretch}{1.08}
\begin{tabular}{@{}l*{7}{c}@{}}
\toprule
Method
& \INhead{$\mathrm{KL}_g$}
& \INhead{Energy $\downarrow$}
& \INhead{Sl.\,$W_1$ $\downarrow$}
& \INhead{$\mathrm{KL}_{\mathcal N}$ $\downarrow$}
& \INhead{$\mathrm{KL}_{\mathrm{KDE}}$ $\downarrow$}
& \INhead{$|\mathrm{gap}|$ $\downarrow$}
& \INhead{$Q$ $\downarrow$} \\
\midrule
Prior $p_{1,\phi}$
& \INcell{\text{--}}
& \INcell{0.0683}
& \INcell{0.0080}
& \INcell{0.752}
& \INcell{0.658}
& \INcell{7.51}
& \INcell{1.000} \\
\midrule
ParetoFlow
& \INcell{\text{--}}
& \INcell{0.0885}
& \INcell{0.0088}
& \INcell{0.825}
& \INcell{0.937}
& \INcell{5.35}
& \INcell{0.844} \\
SLOP
& \INcell{0.520}
& \INcell{0.1200}
& \INcell{0.0098}
& \INcell[_{*}]{\underline{0.168}}
& \INcell{0.170}
& \INcell[_{*\dagger}]{\underline{1.98}}
& \INcell{0.364} \\
ReNO
& \INcell{0.553}
& \INcell{0.0503}
& \INcell{0.0069}
& \INcell{0.197}
& \INcell[_{*}]{\underline{0.166}}
& \INcell{2.10}
& \INcell[_{*\dagger}]{\underline{0.244}} \\
Uniform
& \INcell[^{*}_{\dagger}]{\mathbf{0.436}}
& \INcell{0.0603}
& \INcell{0.0074}
& \INcell{0.605}
& \INcell{0.635}
& \INcell{2.73}
& \INcell[^{*\dagger}]{\mathbf{0.238}} \\
\midrule
BSS (Val.)
& \INcell{0.514}
& \INcell{\underline{0.0498}}
& \INcell{\underline{0.0068}}
& \INcell{0.217}
& \INcell{0.183}
& \INcell{2.57}
& \INcell{0.354} \\
\rowcolor{martinival}
\nameshort\ (Val.)
& \INcell{\underline{0.466}}
& \INcell{\mathbf{0.0321}}
& \INcell{\mathbf{0.0058}}
& \INcell{\mathbf{0.034}}
& \INcell{\mathbf{0.049}}
& \INcell{\mathbf{1.27}}
& \INcell{0.359} \\

\midrule
BSS (Pref.)
& \INcell{0.440}
& \INcell{\underline{0.0423}}
& \INcell{\underline{0.0065}}
& \INcell{\underline{0.135}}
& \INcell{\underline{0.133}}
& \INcell{2.04}
& \INcell{0.316} \\
\rowcolor{martinipref}
\nameshort\ (Pref.)
& \INcell{\mathbf{0.423}}
& \INcell{\mathbf{0.0372}}
& \INcell{\mathbf{0.0061}}
& \INcell{\mathbf{0.093}}
& \INcell{\mathbf{0.121}}
& \INcell{\mathbf{1.71}}
& \INcell{0.297} \\
\bottomrule
\end{tabular}
\vspace{-7pt}
\end{table}

\subsection{Protein-Ligand Co-folding}
\label{subsec:cofolding}

\looseness=-1
Protein--ligand cofolding provides a setting with experimental scalar feedback. We ask whether a surrogate composition learned from binding-affinity measurements can improve the accuracy and physical validity of the jointly generated poses. Qualitative results shown in~\cref{fig:qualitative} (right) and~\cref{fig:prot_gcorr15_top20_barebones}.

\looseness=-1
\xhdr{Setup}
Using the value regime, we fit a composition of five surrogates to experimental $\mathrm{p}K_D$ measurements from HiQBind~\citep{wang2025hiqbind}, restricting the feedback set to PoseBusters-valid monomers and multimers~\citep{buttenschoen2024posebusters}. The fitted composition $g\circ\mathbf R$ then steers the DeCAF-Boltz flow map~\citep{scarpellini2026fewstepcofoldingallatomflow}, distilled from Boltz-1~\citep{wohlwend2025boltz}, without further experimental feedback. We evaluate on $186$ complexes from the post-2023 subset of Runs N' Poses~\citep{skrinjar2026runsnposes}, beyond Boltz-1's training cutoff. As these evaluation complexes lack measured affinities, we assess transfer through pose accuracy and physical validity, reporting best@1 metrics averaged over five seeds. Dataset filtering, surrogate definitions, and steering details are in~\S\ref{app:protein}.

\looseness=-1
\xhdr{Metrics}
Our \emph{success rate} is the percentage of predictions with RMSD below $2$\,\AA{} that pass PoseBusters validity checks~\citep{buttenschoen2024posebusters}. We also report accuracy and validity rates alongside \emph{lDDT-PLI}, which measures local protein--ligand distance agreement within a $6$\,\AA{} pocket~\citep{mariani2013lddt}, as well as structural diversity detailed in~\S\ref{app:protein-metrics}.


\looseness=-1
\xhdr{Results}
The BSS is smina (\cref{fig:prot_ceil_bars_hiqbind}). \nameshort\ (correlation $g$) achieves the highest reported best@1 (\S\ref{app:protein-metrics}) success rate ($64.5\%$) over both potentials-only steering and smina\,+\,potentials by $2.7$ percentage points~(\cref{tab:prot_steering_combined}). Relative to potentials-only, \emph{lDDT-PLI} improves by $1.9$ points, indicating better local geometry, with similar diversity ($0.0133$ versus $0.0132$). PB-validity falls by $6.5$ points, even though \nameshort trades validity for accuracy, all the accurate predictions also satisfy validity criteria. Ablation study results and success rates for reward-selected poses are presented in Appendix~\ref{app:protein-results}.

\begin{table*}
    \centering
    \caption{\small \looseness=-1 \textbf{Protein--ligand steering on post-2023 Runs N' Poses}. 15 particles, 5 seeds, rewards un-normalised.}
        \vspace{-10pt}
    \label{tab:prot_steering_combined}



\definecolor{martinirow}{RGB}{219,234,249}   

\centering
\small
\setlength{\tabcolsep}{4.5pt}
\resizebox{1.0\linewidth}{!}{
\begin{tabular}{lrrrrr}
\toprule
Method & Success Rate $\uparrow$ & RMSD$<2$\,\AA{} $\uparrow$ & PB Valid $\uparrow$ & lDDT-PLI $\uparrow$ & Diversity $\uparrow$ \\
\midrule
Boltz1 potentials & ${59.5}$ & ${59.5}$ & $\mathbf{99.9}$ & ${65.2}$ & ${0.0108}$ \\
DeCAF potentials & ${61.8}$ & ${61.8}$ & $\underline{98.3}$ & ${65.7}$ & ${0.0132}$ \\
smina + potentials & ${61.8}$ & ${61.8}$ & ${98.2}$ & ${65.4}$ & ${0.0129}$ \\
\midrule
\rowcolor{martinirow}DeCAF-\nameshort(smina, potentials) & $\underline{62.7}$ & $\underline{62.8}$ & ${93.3}$ & $\underline{66.5}$ & $\mathbf{0.0139}$ \\
\rowcolor{martinirow}DeCAF-\nameshort(5 surrogates) & $\mathbf{64.5}$ & $\mathbf{64.5}$ & ${91.8}$ & $\mathbf{67.6}$ & $\underline{0.0133}$ \\
\bottomrule
\end{tabular}
}




\end{table*}

\section{Related work}
\looseness=-1
\xhdr{Reward-based control}
Generative models can be aligned through reinforcement learning~\citep{black2024training} or preference optimisation~\citep{rafailov2023direct,wallace2024diffusion,bai2026towards}, and steered at inference through guidance~\citep{bansal2024universal,singhal2025fks}, posterior sampling~\citep{chung2023diffusion}, or flow-map optimal control~\citep{huang2026fmrg}. \nameshort\ complements these methods by learning the steering reward from offline oracle feedback.

\looseness=-1
\xhdr{Reward compositions and limited feedback}
Reward ensembles can mitigate over-optimisation~\citep{coste2024reward}.
ReNO uses fixed reward weights~\citep{eyring2024renoenhancingonesteptexttoimage},
while SLOP calibrates pooled scores to maximize estimated gold reward~\citep{wang2026temper}. We study approximation of a fixed oracle tilt from values or rankings, characterizing the information limit imposed by the surrogates. Multi-objective methods balance distinct
rewards~\citep{rame2023rewarded,yuan2025paretoflow}.
Limited oracle access also motivates query-efficient
fine-tuning~\citep{uehara2024feedback}, protein
steering~\citep{yang2026steering}, and optimisation over
generated samples~\citep{willis2025sample}.
We use feedback to learn how fixed surrogates predict a
single oracle, keeping the generator unchanged.
Complementing analyses of proxy error and ensembling~\citep{paes2026theoretical},
we relate fitting loss and surrogate information limits
to steering fidelity and reward bounds.

\section{Conclusion}
\looseness=-1
\nameshort\ learns surrogate compositions offline from oracle values or rankings for reuse across compatible steerers without inference-time oracle queries. Our theory links fitting to ideal-tilt fidelity and reward improvement, identifying surrogate information limits and conditions for useful steering despite imperfect recovery. Experiments demonstrate benefits across language, images, and protein--ligand cofolding. KL and reward expansions are local in $\beta$, with larger tilts assessed empirically. Preference guarantees require convex, range-bounded classes, and their tightness remains unassessed. Robustness beyond synthetic PL rankings remains untested. ImageNet guidance uses per-pair reference energy, leaving selection without oracle-scored references an open task.\clearpage

\subsection*{AI use statement}
We used generative AI tools to develop and review mathematical statements
and proofs and to support research execution, including implementation and feedback
on experimental methodology, and interpretation of results. Additionally, we used these tools to review, draft, and improve the clarity of the manuscript. The authors take responsibility for the final content of this work, including all text, mathematical claims, results, and artifacts produced with AI assistance.

\subsection*{Ethics statement}
Steering pretrained generative models, including language
models, can amplify biases or errors in the base model,
oracle, or surrogate rewards, and can generate harmful content. Improvements in measured reward do not establish safety, fairness, or factual correctness. Likewise, computational protein--ligand metrics do not replace experimental validation. Applications require task-specific evaluation and appropriate safeguards.

\subsection*{Reproducibility statement}
The fitting objectives and steering procedure are described in~\S\ref{sec:practical}, with composition-fitting details in~\S\ref{app:learn}. Experimental implementation and metric estimation are described in~\S\ref{app:in_details} and~\S\ref{app:metric_est}. Formal assumptions and proofs are provided in~\S\ref{app:assumptions} and~\S\ref{app:proofs}. We commit to publicly releasing the code and supporting materials needed to reproduce our methods and experimental results upon acceptance.


\clearpage
\bibliography{iclr2027_conference}

@inproceedings{bai2026towards,
  title={Towards better optimization for listwise preference in diffusion models},
  author={Bai, Jiamu and Yu, Xin and Xu, Meilong and Lu, Weitao and Pan, Xin and Maeng, Kiwan and Kifer, Daniel and Wang, Jian and Wang, Yu},
  booktitle={International Conference on Learning Representations},
  volume={2026},
  pages={133254--133282},
  year={2026}
}

@misc{bansal2024universal,
      title={Universal Guidance for Diffusion Models}, 
      author={Arpit Bansal and Hong-Min Chu and Avi Schwarzschild and Soumyadip Sengupta and Micah Goldblum and Jonas Geiping and Tom Goldstein},
      year={2023},
      eprint={2302.07121},
      archivePrefix={arXiv},
      primaryClass={cs.CV},
      url={https://arxiv.org/abs/2302.07121}, 
}

@inproceedings{black2024training,
  title={Training Diffusion Models with Reinforcement Learning},
  author={Black, Kevin and Janner, Michael and Du, Yilun and Kostrikov, Ilya and Levine, Sergey},
  booktitle={International Conference on Learning Representations},
  year={2024}
}

@article{bradley1952rank,
      title={Rank Analysis of Incomplete Block Designs: I. The Method of Paired Comparisons},
      author={Bradley, Ralph Allan and Terry, Milton E.},
      journal={Biometrika}, volume={39}, number={3/4}, pages={324--345}, year={1952}
}

@article{buttenschoen2024posebusters,
  author  = {Buttenschoen, Martin and Morris, Garrett M. and
             Deane, Charlotte M.},
  title   = {{PoseBusters}: {AI}-based docking methods fail to generate
             physically valid poses or generalise to novel sequences},
  journal = {Chemical Science},
  volume  = {15},
  number  = {9},
  pages   = {3130--3139},
  year    = {2024},
  doi     = {10.1039/D3SC04185A}
}

@inproceedings{chung2023diffusion,
      title={Diffusion Posterior Sampling for General Noisy Inverse Problems},
      author={Hyungjin Chung and Jeongsol Kim and Michael T. McCann and Marc L. Klasky and Jong Chul Ye},
      year={2023}, eprint={2209.14687}, archivePrefix={arXiv}, primaryClass={stat.ML},
      booktitle={International Conference on Learning Representations},
      url={https://arxiv.org/abs/2209.14687}
}

@inproceedings{coste2024reward,
  title={Reward Model Ensembles Help Mitigate Overoptimization},
  author={Coste, Thomas and Anwar, Usman and Kirk, Robert and Krueger, David},
  booktitle={International Conference on Learning Representations},
  year={2024}
}

@article{eisenstein2023helping,
  title={Helping or herding? reward model ensembles mitigate but do not eliminate reward hacking},
  author={Eisenstein, Jacob and Nagpal, Chirag and Agarwal, Alekh and Beirami, Ahmad and D'Amour, Alex and Dvijotham, DJ and Fisch, Adam and Heller, Katherine and Pfohl, Stephen and Ramachandran, Deepak and others},
  journal={arXiv preprint arXiv:2312.09244},
  year={2023}
}

@article{ethayarajh2024kto,
  title={{KTO}: Model alignment as prospect theoretic optimization},
  author={Ethayarajh, Kawin and Xu, Winnie and Muennighoff, Niklas and Jurafsky, Dan and Kiela, Douwe},
  journal={arXiv preprint arXiv:2402.01306},
  year={2024}
}

@misc{eyring2024renoenhancingonesteptexttoimage,
      title={ReNO: Enhancing One-step Text-to-Image Models through Reward-based Noise Optimization}, 
      author={Luca Eyring and Shyamgopal Karthik and Karsten Roth and Alexey Dosovitskiy and Zeynep Akata},
      year={2024},
      eprint={2406.04312},
      archivePrefix={arXiv},
      primaryClass={cs.CV},
      url={https://arxiv.org/abs/2406.04312}, 
}

@inproceedings{gao2023scaling,
  title={Scaling Laws for Reward Model Overoptimization},
  author={Gao, Leo and Schulman, John and Hilton, Jacob},
  booktitle={International Conference on Machine Learning},
  year={2023}
}

@misc{huang2026fmrg,
      title={How to Guide Your Flow: Few-Step Alignment via Flow Map Reward Guidance},
      author={Jerry Y. Huang and Justin Lin and Sheel Shah and Kartik Nair and Nicholas M. Boffi},
      year={2026}, eprint={2604.27147}, archivePrefix={arXiv}, primaryClass={cs.LG},
      url={https://arxiv.org/abs/2604.27147}
}

@article{kirstain2023pick,
  title={Pick-a-pic: An open dataset of user preferences for text-to-image generation},
  author={Kirstain, Yuval and Polyak, Adam and Singer, Uriel and Matiana, Shahbuland and Penna, Joe and Levy, Omer},
  journal={Advances in neural information processing systems},
  volume={36},
  pages={36652--36663},
  year={2023}
}

@book{luce1959individual,
      title={Individual Choice Behavior: A Theoretical Analysis},
      author={Luce, R. Duncan},
      publisher={Wiley}, address={New York}, year={1959}
}

@article{paes2026theoretical,
  title={Theoretical Limits of Language Model Alignment},
  author={Paes, Lucas Monteiro and Mackraz, Natalie and Theobald, Barry-John and Danieli, Federico},
  journal={arXiv preprint arXiv:2605.07105},
  year={2026}
}

@misc{pan2026itomaps,
      title={It\^o maps for any-step SDEs},
      author={Zhengkai Pan and Peter Potaptchik and Wenxi Yao and Michael S. Albergo and Jakiw Pidstrigach},
      year={2026}, eprint={2606.11156}, archivePrefix={arXiv}, primaryClass={stat.ML},
      url={https://arxiv.org/abs/2606.11156}
}

@article{pathak2025ai,
  title={AI-enabled drug and molecular discovery: computational methods, platforms, and translational horizons},
  author={Pathak, Aditya and Theagarajan, Radhika and Rizqi, Muhammad Miftakhur and Nugraha, Ari Satia and Boruah, Tridip and Pratiksha and Kumar, Himanshu and Naik, Bindu and Yadav, Shweta and Jha, Avinash Kumar and others},
  journal={Discover Molecules},
  volume={2},
  number={1},
  pages={32},
  year={2025},
  publisher={Springer}
}

@article{plackett1975analysis,
      title={The Analysis of Permutations},
      author={Plackett, R. L.},
      journal={Journal of the Royal Statistical Society. Series C (Applied Statistics)},
      volume={24}, number={2}, pages={193--202}, year={1975}
}

@misc{potaptchik2026mfm,
      title={Meta Flow Maps enable scalable reward alignment},
      author={Peter Potaptchik and Adhi Saravanan and Abbas Mammadov and Alvaro Prat and Michael S. Albergo and Yee Whye Teh},
      year={2026}, eprint={2601.14430}, archivePrefix={arXiv}, primaryClass={stat.ML},
      url={https://arxiv.org/abs/2601.14430}
}

@misc{radford2021learning,
      title={Learning Transferable Visual Models From Natural Language Supervision},
      author={Alec Radford and Jong Wook Kim and Chris Hallacy and Aditya Ramesh and Gabriel Goh and Sandhini Agarwal and Girish Sastry and Amanda Askell and Pamela Mishkin and Jack Clark and Gretchen Krueger and Ilya Sutskever},
      year={2021}, eprint={2103.00020}, archivePrefix={arXiv}, primaryClass={cs.CV},
      url={https://arxiv.org/abs/2103.00020}
}

@misc{rafailov2023direct,
      title={Direct Preference Optimization: Your Language Model is Secretly a Reward Model},
      author={Rafael Rafailov and Archit Sharma and Eric Mitchell and Stefano Ermon and Christopher D. Manning and Chelsea Finn},
      year={2023}, eprint={2305.18290}, archivePrefix={arXiv}, primaryClass={cs.LG},
      url={https://arxiv.org/abs/2305.18290}
}

@inproceedings{rame2023rewarded,
  title={Rewarded Soups: Towards Pareto-Optimal Alignment by Interpolating Weights Fine-Tuned on Diverse Rewards},
  author={Ram\'e, Alexandre and Couairon, Guillaume and Shukor, Mustafa and Dancette, Corentin and Gaya, Jean-Baptiste and Soulier, Laure and Cord, Matthieu},
  booktitle={Advances in Neural Information Processing Systems},
  year={2023},
  note={arXiv:2306.04488}
}

@misc{singhal2025fks,
 archivePrefix = {arXiv},
 author = {Singhal, Raghav and Horvitz, Zachary and Teehan, Ryan and Ren, Mengye and Yu, Zhou and McKeown, Kathleen and Ranganath, Rajesh},
 eprint = {2501.06848},
 primaryClass = {cs.LG},
 publisher = {arXiv},
 title = {{A} {General} {Framework} for {Inference-time} {Scaling} and {Steering} of {Diffusion} {Models}},
 url = {https://arxiv.org/abs/2501.06848},
 year = {2025}
}

@inproceedings{uehara2024feedback,
  title={Feedback Efficient Online Fine-Tuning of Diffusion Models},
  author={Uehara, Masatoshi and Zhao, Yulai and Black, Kevin and Hajiramezanali, Ehsan and Scalia, Gabriele and Diamant, Nathaniel Lee and Tseng, Alex M. and Levine, Sergey and Biancalani, Tommaso},
  booktitle={International Conference on Machine Learning},
  year={2024}
}

@article{uehara2025inference,
  title={Inference-time alignment in diffusion models with reward-guided generation: Tutorial and review},
  author={Uehara, Masatoshi and Zhao, Yulai and Wang, Chenyu and Li, Xiner and Regev, Aviv and Levine, Sergey and Biancalani, Tommaso},
  journal={arXiv preprint arXiv:2501.09685},
  year={2025}
}

@inproceedings{wallace2024diffusion,
  title={Diffusion Model Alignment Using Direct Preference Optimization},
  author={Wallace, Bram and Dang, Meihua and Rafailov, Rafael and Zhou, Linqi and Lou, Aaron and Purushwalkam, Senthil and Ermon, Stefano and Xiong, Caiming and Joty, Shafiq and Naik, Nikhil},
  booktitle={IEEE/CVF Conference on Computer Vision and Pattern Recognition},
  year={2024}
}

@article{wang2026temper,
  title={Temper and Tilt Lead to SLOP: Reward Hacking Mitigation with Inference-Time Alignment},
  author={Wang, Ye and Liu, Jing and Koike-Akino, Toshiaki},
  journal={arXiv preprint arXiv:2605.13537},
  year={2026}
}

@misc{willis2025sample,
      title={Sample-Efficient Optimisation over the Outputs of Generative Models}, 
      author={Samuel Willis and Paul Duckworth and Jack Simons and Aleksandra Kalisz and Krisztina Sinkovics and Noam Ghenassia and Shikha Surana and Henry T. Oldroyd and Alexandru I. Stere and Dragos D Margineantu and Carl Henrik Ek and Henry Moss and Erik Bodin},
      year={2025},
      eprint={2509.23800},
      archivePrefix={arXiv},
      primaryClass={stat.ML},
      url={https://arxiv.org/abs/2509.23800}, 
}

@article{wood2026family,
  title={UMA: A family of universal models for atoms},
  author={Wood, Brandon and Dzamba, Misko and Fu, Xiang and Gao, Meng and Shuaibi, Muhammed and Barroso-Luque, Luis and Abdelmaqsoud, Kareem and Gharakhanyan, Vahe and Kitchin, John and Levine, Daniel S. and others},
  journal={Advances in Neural Information Processing Systems},
  volume={38},
  pages={129391--129427},
  year={2025}
}

@misc{xu2023imagereward,
      title={ImageReward: Learning and Evaluating Human Preferences for Text-to-Image Generation},
      author={Jiazheng Xu and Xiao Liu and Yuchen Wu and Yuxuan Tong and Qinkai Li and Ming Ding and Jie Tang and Yuxiao Dong},
      year={2023}, eprint={2304.05977}, archivePrefix={arXiv}, primaryClass={cs.CV},
      url={https://arxiv.org/abs/2304.05977}
}

@article{yang2026steering,
  title={Steering generative models with experimental data for protein fitness optimization},
  author={Yang, Jason and Chu, Wenda and Khalil, Daniel and Astudillo, Raul and Wittmann, Bruce and Arnold, Frances and Yue, Yisong},
  journal={Advances in Neural Information Processing Systems},
  volume={38},
  pages={145383--145421},
  year={2025}
}

@inproceedings{yuan2025paretoflow,
  title={ParetoFlow: Guided Flows in Multi-Objective Optimization},
  author={Yuan, Ye and Chen, Can and Pal, Christopher and Liu, Xue},
  booktitle={International Conference on Learning Representations},
  year={2025},
  note={arXiv:2412.03718}
}

@article{zhang2026instruction,
  title={Instruction tuning for large language models: A survey},
  author={Zhang, Shengyu and Dong, Linfeng and Li, Xiaoya and Zhang, Sen and Sun, Xiaofei and Wang, Shuhe and Li, Jiwei and Hu, Runyi and Zhang, Tianwei and Wang, Guoyin and others},
  journal={ACM Computing Surveys},
  volume={58},
  number={7},
  pages={1--36},
  year={2026},
  publisher={ACM New York, NY}
}

@article {wohlwend2025boltz,
	author = {Wohlwend, Jeremy and Corso, Gabriele and Passaro, Saro and Getz, Noah and Reveiz, Mateo and Leidal, Ken and Swiderski, Wojtek and Atkinson, Liam and Portnoi, Tally and Chinn, Itamar and Silterra, Jacob and Jaakkola, Tommi and Barzilay, Regina},
	title = {Boltz-1 Democratizing Biomolecular Interaction Modeling},
	year = {2025},
	doi = {10.1101/2024.11.19.624167},
	publisher = {Cold Spring Harbor Laboratory},
	URL = {https://www.biorxiv.org/content/early/2025/05/06/2024.11.19.624167},
	eprint = {https://www.biorxiv.org/content/early/2025/05/06/2024.11.19.624167.full.pdf},
	journal = {bioRxiv}
}

@article{passaro2025boltz2,
  title={Boltz-2: Towards Accurate and Efficient Binding Affinity Prediction},
  author={Passaro, Saro and Corso, Gabriele and Wohlwend, Jeremy and Reveiz, Mateo and Thaler, Stephan and Somnath, Vignesh Ram and Getz, Noah and Portnoi, Tally and Roy, Julien and Stark, Hannes and Kwabi-Addo, David and Beaini, Dominique and Jaakkola, Tommi and Barzilay, Regina},
  journal={bioRxiv},
  year={2025},
  doi={10.1101/2025.06.14.659707},
  publisher={Cold Spring Harbor Laboratory},
  URL={https://www.biorxiv.org/content/10.1101/2025.06.14.659707v1}
}

@misc{scarpellini2026fewstepcofoldingallatomflow,
      title={Few-step Cofolding with All-Atom Flow Maps}, 
      author={Gianluca Scarpellini and Ron Shprints and Peter Holderrieth and Juno Nam and Pranav Murugan and Rafael Gómez-Bombarelli and Tommi Jaakkola and Maruan Al-Shedivat and Nicholas Matthew Boffi and Avishek Joey Bose},
      year={2026},
      eprint={2606.08375},
      archivePrefix={arXiv},
      primaryClass={cs.LG},
      url={https://arxiv.org/abs/2606.08375}, 
}

@article{skrinjar2026runsnposes,
  title={Evaluating generalization in protein–ligand cofolding methods},
  author={Peter \v{S}krinjar and J{\'e}r{\^o}me Eberhardt and Gabriel Studer and Gerardo Tauriello and Torsten Schwede and Janani Durairaj},
  journal={Nature Structural \& Molecular Biology},
  year={2026},
  volume={33},
  pages={782--794},
  doi={10.1038/s41594-026-01797-5},
  url={https://doi.org/10.1038/s41594-026-01797-5}
}

@article{wang2025hiqbind,
  title={A workflow to create a high-quality protein--ligand binding dataset for training, validation, and prediction tasks},
  author={Wang, Yingze and Sun, Kunyang and Li, Jie and Guan, Xingyi and Zhang, Oufan and Bagni, Dorian and Zhang, Yang and Carlson, Heather A and Head-Gordon, Teresa},
  journal={Digital Discovery},
  volume={4},
  number={5},
  pages={1209--1220},
  year={2025},
  publisher={The Royal Society of Chemistry}
}

@article{chen2010molprobity,
  title={MolProbity: all-atom structure validation for macromolecular crystallography},
  author={Chen, Vincent B and Arendall, W Bryan and Headd, Jeffrey J and Keedy, Daniel A and Immormino, Robert M and Kapral, Gary J and Murray, Laura W and Richardson, Jane S and Richardson, David C},
  journal={Acta Crystallographica Section D: Biological Crystallography},
  volume={66},
  number={1},
  pages={12--21},
  year={2010},
  publisher={International Union of Crystallography},
  doi       = {10.1107/S0907444909042073}
}

@article{rausch2021frustratometeR,
    author = {Rausch, Atilio O and Freiberger, Maria I and Leonetti, Cesar O and Luna, Diego M and Radusky, Leandro G and Wolynes, Peter G and Ferreiro, Diego U and Parra, R Gonzalo},
    title = {FrustratometeR: an R-package to compute local frustration in protein structures, point mutants and MD simulations},
    journal = {Bioinformatics},
    volume = {37},
    number = {18},
    pages = {3038-3040},
    year = {2021},
    month = {09},
    issn = {1367-4803},
    doi = {10.1093/bioinformatics/btab176},
    url = {https://doi.org/10.1093/bioinformatics/btab176},
    eprint = {https://academic.oup.com/bioinformatics/article-pdf/37/18/3038/50579563/btab176.pdf},
}

@article{koes2013smina,
author = {Koes, David and Baumgartner, Matthew and Camacho, Carlos},
year = {2013},
volume = {53},
number = {8}, 
pages = {1893--1904},
title = {Lessons Learned in Empirical Scoring with smina from the CSAR 2011 Benchmarking Exercise},
journal = {Journal of chemical information and modeling},
doi = {10.1021/ci300604z}
}

@article{liu_rectified_2022,
  author  = {Liu, Qiang},
  title   = {Rectified Flow: A Marginal Preserving Approach to Optimal Transport},
  journal = {arXiv preprint arXiv:2209.14577},
  year    = {2022},
  doi     = {10.48550/arXiv.2209.14577}
}

@misc{lipman_flow_2022,
      title={Flow Matching for Generative Modeling}, 
      author={Yaron Lipman and Ricky T. Q. Chen and Heli Ben-Hamu and Maximilian Nickel and Matt Le},
      year={2022},
      eprint={2210.02747},
      archivePrefix={arXiv},
      primaryClass={cs.LG},
      url={https://arxiv.org/abs/2210.02747}, 
}

@article{mariani2013lddt,
  author  = {Mariani, Valerio and Biasini, Marco and Barbato, Alessandro and Schwede, Torsten},
  title   = {{lDDT}: a local superposition-free score for comparing protein structures and models using distance difference tests},
  journal = {Bioinformatics},
  volume  = {29}, number = {21}, pages = {2722--2728}, year = {2013},
  doi     = {10.1093/bioinformatics/btt473}
}

@misc{lee2026flow,
      title={Flow Map Language Models: One-step Language Modeling via Continuous Denoising}, 
      author={Chanhyuk Lee and Jaehoon Yoo and Manan Agarwal and Sheel Shah and Jerry Huang and Aditi Raghunathan and Seunghoon Hong and Nicholas M. Boffi and Jinwoo Kim},
      year={2026},
      eprint={2602.16813},
      archivePrefix={arXiv},
      primaryClass={cs.CL},
      url={https://arxiv.org/abs/2602.16813}, 
}

@misc{schulman2017proximal,
      title={Proximal Policy Optimization Algorithms}, 
      author={John Schulman and Filip Wolski and Prafulla Dhariwal and Alec Radford and Oleg Klimov},
      year={2017},
      eprint={1707.06347},
      archivePrefix={arXiv},
      primaryClass={cs.LG},
      url={https://arxiv.org/abs/1707.06347}, 
}

@article{yellott1977relationship,
title = {The relationship between Luce's Choice Axiom, Thurstone's Theory of Comparative Judgment, and the double exponential distribution},
author = {John I. Yellott},
journal = {Journal of Mathematical Psychology},
volume = {15},
number = {2},
pages = {109-144},
year = {1977},
issn = {0022-2496},
doi = {https://doi.org/10.1016/0022-2496(77)90026-8},
url = {https://www.sciencedirect.com/science/article/pii/0022249677900268},
}

@misc{albergo_building_2023,
      title={Building Normalizing Flows with Stochastic Interpolants}, 
      author={Michael S. Albergo and Eric Vanden-Eijnden},
      year={2023},
      eprint={2209.15571},
      archivePrefix={arXiv},
      primaryClass={cs.LG},
      url={https://arxiv.org/abs/2209.15571}, 
}

@inproceedings{song2020score,
      title={Score-Based Generative Modeling through Stochastic Differential Equations}, 
      author={Yang Song and Jascha Sohl-Dickstein and Diederik P. Kingma and Abhishek Kumar and Stefano Ermon and Ben Poole},
      year={2021},
      booktitle={International Conference on Learning Representations},
      eprint={2011.13456},
      archivePrefix={arXiv},
      primaryClass={cs.LG},
      url={https://arxiv.org/abs/2011.13456}, 
}

@misc{ho2020denoising,
      title={Denoising Diffusion Probabilistic Models}, 
      author={Jonathan Ho and Ajay Jain and Pieter Abbeel},
      year={2020},
      eprint={2006.11239},
      archivePrefix={arXiv},
      primaryClass={cs.LG},
      url={https://arxiv.org/abs/2006.11239}, 
}

@incollection{doucet2001introduction,
author="Doucet, Arnaud
and de Freitas, Nando
and Gordon, Neil",
editor="Doucet, Arnaud
and de Freitas, Nando
and Gordon, Neil",
title="An Introduction to Sequential Monte Carlo Methods",
booktitle="Sequential Monte Carlo Methods in Practice",
year="2001",
publisher="Springer New York",
address="New York, NY",
pages="3--14",
isbn="978-1-4757-3437-9",
doi="10.1007/978-1-4757-3437-9_1",
url="https://doi.org/10.1007/978-1-4757-3437-9_1"
}

@article{biasini2013openstructure,
author = "Biasini, M. and Schmidt, T. and Bienert, S. and Mariani, V. and Studer, G. and Haas, J. and Johner, N. and Schenk, A. D. and Philippsen, A. and Schwede, T.",
title = "{{OpenStructure}: an integrated software framework for computational structural biology}",
journal = "Acta Crystallographica Section D",
year = "2013",
volume = "69",
number = "5",
pages = "701--709",
month = "May",
doi = {10.1107/S0907444913007051},
url = {https://doi.org/10.1107/S0907444913007051},
}

@article{zheng2023judging,
  title = {Judging {LLM}-as-a-Judge with {MT-Bench} and {Chatbot Arena}},
  author = {Zheng, Lianmin and Chiang, Wei-Lin and Sheng, Ying and
            Zhuang, Siyuan and Wu, Zhanghao and Zhuang, Yonghao and
            Lin, Zi and Li, Zhuohan and Li, Dacheng and Xing, Eric P. and
            Zhang, Hao and Gonzalez, Joseph E. and Stoica, Ion},
  journal = {arXiv preprint arXiv:2306.05685},
  year = {2023},
  url = {https://arxiv.org/abs/2306.05685}
}

@article{wang2025skywork,
  title={Skywork-vl reward: An effective reward model for multimodal understanding and reasoning},
  author={Wang, Xiaokun and Wang, Peiyu and Pei, Jiangbo and Shen, Wei and Peng, Yi and Hao, Yunzhuo and Qiu, Weijie and Jian, Ai and Xie, Tianyidan and Song, Xuchen and others},
  journal={arXiv preprint arXiv:2505.07263},
  year={2025}
}

@article{wu2023human,
  title={Human preference score v2: A solid benchmark for evaluating human preferences of text-to-image synthesis},
  author={Wu, Xiaoshi and Hao, Yiming and Sun, Keqiang and Chen, Yixiong and Zhu, Feng and Zhao, Rui and Li, Hongsheng},
  journal={arXiv preprint arXiv:2306.09341},
  year={2023}
}

@article{chelba2013one,
  title={One billion word benchmark for measuring progress in statistical language modeling},
  author={Chelba, Ciprian and Mikolov, Tomas and Schuster, Mike and Ge, Qi and Brants, Thorsten and Koehn, Phillipp and Robinson, Tony},
  journal={arXiv preprint arXiv:1312.3005},
  year={2013}
}

@misc{qwen2025qwen25technicalreport,
      title={Qwen2.5 Technical Report}, 
      author={Qwen and : and An Yang and Baosong Yang and Beichen Zhang and Binyuan Hui and Bo Zheng and Bowen Yu and Chengyuan Li and Dayiheng Liu and Fei Huang and Haoran Wei and Huan Lin and Jian Yang and Jianhong Tu and Jianwei Zhang and Jianxin Yang and Jiaxi Yang and Jingren Zhou and Junyang Lin and Kai Dang and Keming Lu and Keqin Bao and Kexin Yang and Le Yu and Mei Li and Mingfeng Xue and Pei Zhang and Qin Zhu and Rui Men and Runji Lin and Tianhao Li and Tianyi Tang and Tingyu Xia and Xingzhang Ren and Xuancheng Ren and Yang Fan and Yang Su and Yichang Zhang and Yu Wan and Yuqiong Liu and Zeyu Cui and Zhenru Zhang and Zihan Qiu},
      year={2025},
      eprint={2412.15115},
      archivePrefix={arXiv},
      primaryClass={cs.CL},
      url={https://arxiv.org/abs/2412.15115}, 
}

@inproceedings{devlin2019bert,
  title={Bert: Pre-training of deep bidirectional transformers for language understanding},
  author={Devlin, Jacob and Chang, Ming-Wei and Lee, Kenton and Toutanova, Kristina},
  booktitle={Proceedings of the 2019 conference of the North American chapter of the association for computational linguistics: human language technologies, volume 1 (long and short papers)},
  pages={4171--4186},
  year={2019}
}

@article{sanh2019distilbert,
  title={DistilBERT, a distilled version of BERT: smaller, faster, cheaper and lighter},
  author={Sanh, Victor and Debut, Lysandre and Chaumond, Julien and Wolf, Thomas},
  journal={arXiv preprint arXiv:1910.01108},
  year={2019}
}

@article{warstadt2019cola,
  title={Cola: The corpus of linguistic acceptability (with added annotations)},
  author={Warstadt, Alex and Singh, Amanpreet and Bowman, Samuel R},
  year={2019}
}

@article{oquab2023dinov2,
  title={Dinov2: Learning robust visual features without supervision},
  author={Oquab, Maxime and Darcet, Timoth{\'e}e and Moutakanni, Th{\'e}o and Vo, Huy and Szafraniec, Marc and Khalidov, Vasil and Fernandez, Pierre and Haziza, Daniel and Massa, Francisco and El-Nouby, Alaaeldin and others},
  journal={arXiv preprint arXiv:2304.07193},
  year={2023}
}

@inproceedings{reimers2019sentence,
  title={Sentence-bert: Sentence embeddings using siamese bert-networks},
  author={Reimers, Nils and Gurevych, Iryna},
  booktitle={Proceedings of the 2019 conference on empirical methods in natural language processing and the 9th international joint conference on natural language processing (EMNLP-IJCNLP)},
  pages={3982--3992},
  year={2019}
}

@inproceedings{lee2026decoupled,
  title={Decoupled meanflow: Turning flow models into flow maps for accelerated sampling},
  author={Lee, Kyungmin and Yu, Sihyun and Shin, Jinwoo},
  booktitle={International Conference on Learning Representations},
  volume={2026},
  pages={64555--64583},
  year={2026}
}

@article{zhang2004tmscore,
author = {Zhang, Yang and Skolnick, Jeffrey},
title = {Scoring function for automated assessment of protein structure template quality},
journal = {Proteins: Structure, Function, and Bioinformatics},
volume = {57},
number = {4},
pages = {702-710},
doi = {https://doi.org/10.1002/prot.20264},
url = {https://onlinelibrary.wiley.com/doi/abs/10.1002/prot.20264},
eprint = {https://onlinelibrary.wiley.com/doi/pdf/10.1002/prot.20264},
year = {2004}
}
\bibliographystyle{abbrvnat}

\appendix
\newcommand{\apptocline}[1]{%
  \noindent\hyperref[#1]{\makebox[1.9em][l]{\ref*{#1}}\nameref*{#1}}%
  \nobreak\leaders\hbox{$\;\cdot\;$}\hfill\nobreak\hyperref[#1]{\pageref*{#1}}\par\nobreak}

\section*{Appendix contents}
\vspace{-0.4em}
{\small
\apptocline{app:learn}
\apptocline{app:in_details}
\apptocline{app:metric_est}
\apptocline{app:add-quan}
\apptocline{app:ablations}
\apptocline{app:protein}
\apptocline{app:extra_theory}
\apptocline{app:proofs}
\apptocline{app:toy}
\apptocline{app:qual}
\apptocline{app:pareto}
}
\vspace{0.6em}

\section{Learning the composition $g$: derivations}
\label{app:learn}

We consider correlation weighting, linear compositions, and MLPs.
Linear and MLP compositions minimize the value or preference loss
in~\cref{eq:mse-loss,eq:pl-loss}, with regularization.
Correlation weights are computed directly from the feedback.
The fitted composition is then held fixed during steering.

\subsection{Fitting objectives}
For value feedback, write $y_i=R^*(x_i)$ and
$f_i=g(\mathbf R(x_i))$. The objective is
\[
\gL_{\mathrm{MSE}}(g)=\frac1n\sum_{i=1}^n(f_i-y_i)^2.
\]
For preferences, let $x_{(i)}^{(j)}$ be the candidate ranked
$i$th in query $j$, with the most preferred first, and write
$f_{(i)}^{(j)}=g(\mathbf R(x_{(i)}^{(j)}))$. The PL loss is
\[
\gL_{\PL}(g)
=-\frac1M\sum_{j=1}^M\sum_{i=1}^{m-1}
\left[f_{(i)}^{(j)}-
\log\sum_{l=i}^m e^{f_{(l)}^{(j)}}\right].
\]
It is convex in the candidate scores and invariant to a common
additive shift. Proposition~\ref{prop:plhess} gives its derivatives.

\subsection{Correlation weighting}
The composition is $\hat g(\mathbf R)=\sum_{j=1}^k w_jR_j$.
For value feedback, $\rho_j$ is the empirical Pearson
correlation between $R_j$ and the oracle values. For preferences,
it is the Spearman correlation with the observed rankings,
assigning larger ordinal scores to more preferred candidates.
We retain signed correlations and use
\[
w_j=
\begin{cases}
\rho_j/\sum_{l=1}^k\rho_l, & \sum_{l=1}^k\rho_l>0,\\
1/k, & \text{otherwise}.
\end{cases}
\]
This is a weighting heuristic rather than, in general, a minimizer
of either fitting loss.

\subsection{Linear compositions}
Let $g_{\mathbf w}(\mathbf R)=\mathbf w^\top\mathbf R$ and
let $\Phi$ have rows $\mathbf R(x_i)^\top$. For values,
ridge regression minimizes
\[
\frac12\|\Phi\mathbf w-\mathbf y\|_2^2
+\frac\lambda2\|\mathbf w\|_2^2,
\qquad \lambda>0,
\]
with solution
\[
\hat{\mathbf w}
=(\Phi^\top\Phi+\lambda I)^{-1}\Phi^\top\mathbf y.
\]
For preferences, we minimize
\[
\gL_{\PL}(g_{\mathbf w})
+\frac{\lambda_{\PL}}2\|\mathbf w\|_2^2,
\qquad \lambda_{\PL}>0.
\]
Both objectives are strictly convex and coercive, so each has
a unique minimizer. The fitted score is
$\hat g(\mathbf R)=\hat{\mathbf w}^\top\mathbf R$.

\subsection{MLP compositions}
We fit a neural composition $g_\theta$ using the appropriate
value or preference objective with weight decay, and set
$\hat g=g_{\hat\theta}$ after training. Unlike the linear
family, this parameterization can represent nonlinear
relationships between the surrogates and the oracle.

\looseness=-1
In practice, we use the same small network in all image and language experiments. The MLP family maps the four surrogate scores, standardised over the pool, through two hidden layers of $64$ units with $\tanh$ activations to a scalar ($4\to64\to64\to1$, no dropout). It is trained full-batch from a single initialisation using Adam with learning rate $3\times10^{-3}$ and weight decay $10^{-3}$.
\clearpage
\section{Implementation details for image and language generation}
\label{app:in_details}

\begin{algorithm}[htbp]
\caption{\nameshort: fit a surrogate composition and steer}
\label{alg:martini}
\begin{algorithmic}[1]
\Require pretrained base model with terminal law $p_{1,\phi}$;
surrogates $\mathbf R$; feedback $\gD_{\text{feedback}}$;
composition family $\gG$; tilt strength $\beta$;
steerer \textsc{Steer}
\State \textbf{Offline fitting:} Evaluate $\mathbf R$ on the feedback candidates in $\gD_{\text{feedback}}$
\If{correlation weighting is selected} 
\State Construct $\hat g$ from empirical correlations
    (\S\ref{app:learn})
\Else
    \State Set $\ell\gets\gL_{\mathrm{MSE}}$ for values~\eqref{eq:mse-loss}
    or $\ell\gets\gL_{\PL}$ for preferences~\eqref{eq:pl-loss}
    \State Fit $\hat g\in\gG$ using $\ell$ and
    family-specific regularization
\EndIf
\State \textbf{Inference:}
\State \Return $\textsc{Steer}(\text{base model}, \hat g\circ\mathbf R, \beta)$
\end{algorithmic}
\end{algorithm}

We give implementation details for the experiments in~\S\ref{sec:experiments}. Additional results appear in~\S\ref{app:ablations}.

\xhdr{Compute}
Experiments were run on NVIDIA H100-80GB GPUs, with
independent experiments distributed across
devices.

\subsection{ImageNet}
We use the released model of \citet{potaptchik2026mfm} (\texttt{mfm-xl2}), a SiT-XL/2 transformer distilled from the Decoupled MeanFlow ImageNet model~\citep{lee2026decoupled}. It generates
class-conditional ImageNet images at $256\times256$ in the latent space of the Stable Diffusion VAE
(\texttt{sd-vae-ft-mse}; latents of shape $4\times32\times32$, scaled by $0.18215$), with noise at
$t=0$, data at $t=1$ and the linear interpolant $x_t=(1-t)\varepsilon+t\,x_1$. The model is
conditioned only on the ImageNet class label: the text prompt never enters the model, and acts only through the rewards --- $R^*$ and
the surrogates score the decoded image against it.

\xhdr{Reward models}
The oracle uses \texttt{Skywork/Skywork-VL-Reward-7B}.
The four surrogates are \texttt{RE-N-Y/hpsv21},
\texttt{RE-N-Y/ImageReward},
\texttt{openai/clip-vit-base-patch16}, and
\texttt{yuvalkirstain/PickScore\_v1}.
For each image and steering prompt, the oracle score is the
Yes/No logit margin,
\[
R^*(x)=\mathrm{logit}_{\mathrm{Yes}}(x)
-\mathrm{logit}_{\mathrm{No}}(x).
\]

\xhdr{Target} $q^*\propto p_{1,\phi}\exp(\beta z(R^*))$, $\beta=1$, is estimated by self-normalised
reweighting of $8192$ independent draws from the model's reverse SDE ($32$ steps); the effective
sample size of these weights is reported per class in Table~\ref{tab:in-classes}.

\xhdr{Feedback and fitting}
ImageNet uses $n{=}32$ scalar $R^*$ evaluations (value) or $M{=}128$ ranking queries of $m{=}8$
candidates each (preference), drawn from the class's pool. We fit separate composition coefficients
for each class--prompt pair. A fifth of the value labels ($6$) and a quarter of the preference queries
($32$, built from images disjoint from the fitting queries) are held out; every family is fit on the
remainder, and the family is selected by held-out MSE (value) or Plackett--Luce loss (preference),
averaged over all $19$ pairs. The linear family is selected in both regimes. The held-out part is used
only for this selection and is not refit.
loss.

\xhdr{Preference utilities}
The synthetic annotator sorts $z(R^*)+\sigma\epsilon$ with independent standard Gumbel noise
$\epsilon$ and $\sigma=0.5$, which samples Plackett--Luce rankings with utilities $z(R^*)/\sigma$.
We standardise $R^*$ once over the reference pool, fixing its scale across queries. In our
eight-candidate calibration (a standard-normal oracle), $\sigma=0.5$ gives $82.7\%$ pairwise
agreement with the unperturbed oracle order. This noise level is loosely informed by the
approximately $81\%$ non-tie human agreement reported on MT-Bench~\citep{zheng2023judging};
agreement with an oracle order and agreement between annotators are different quantities. A fitted
utility $f$ is expressed in PL units, so the tilt it implies on the reference scale is
$\exp(\beta\sigma f)$. In practice we standardise the fitted composition over the pool before use,
so the steering reward and $\mathrm{KL}_g$ both use $\exp(\beta z(f))$, which coincides with
$\exp(\beta\sigma f)$ when $f$ recovers the utility exactly.

\xhdr{Pools and data splits}
For each class--prompt pair and feedback regime, we draw
an independent pool of $8192$ samples and score the oracle
and surrogates. Within each regime, the pool supplies
feedback candidates, normalization statistics, and the
importance-weighted reference~\S\ref{app:metric_est}. Held-out recovery is
measured on pool rows that no fit or selector touched, subsampled to at most $4000$. Oracle and
surrogate scores are computed once per pool sample and cached. The prior $p_{1,\phi}$ is $64$ fresh
draws from the FMRG sampler at zero strength, and BoN, sBoN and SLOP select from an independent pool
of $4096$ such draws, never from the reference they are scored against; these rows come from a
separate run with its own $8192$-draw reference.

\xhdr{Steering and strength selection}
We use FMRG for ImageNet and sweep its guidance strength $\lambda\in\{2,3,5,8\}$ separately from the
target tilt strength $\beta=1$. Each run takes $32$ kernel transitions with a reward-gradient step
after each, the reward evaluated on one look-ahead endpoint (one noise draw, one flow-map jump), and
each setting is a single run of $64$ samples, scored afresh. For each method and class--prompt pair,
all generated-sample metrics are reported at the evaluated strength with the lowest embedding-space
energy distance to $q^*$. BoN is reported at fixed rungs ($n{=}8,32$), and SLOP at $16$ candidates per
sample, the rung matched to \nameshort's inference cost, calibrated on the same $n{=}32$ labels.
This selection uses the oracle-scored reference and describes performance within the evaluated sweep.

\xhdr{Baselines}
BSS is selected from training feedback via either PCC or SCC for value and preference regimes respectively. Uniform weighting combines the
normalized surrogate scores with equal weights.
For images, ReNO uses its published fixed reward weights
within the same FMRG pipeline~\citep{eyring2024renoenhancingonesteptexttoimage}.
This compares reward mixtures rather than reproducing
ReNO's original noise-optimization procedure.
We also compare ParetoFlow~\citep{yuan2025paretoflow}
and SLOP's calibrated selection protocol~\citep{wang2026temper}.
SLOP is calibrated using numerical oracle values; its
preference-regime results are therefore a value-calibrated
reference. BoN selects the highest-$R^*$ candidate among
$N$ candidates, requiring $N$ oracle evaluations per output.

\subsection{Language generation}
We use the released FMLM checkpoint~\citep{lee2026flow} (\texttt{david3684/FMLM-B-LM1B}), trained on
One Billion Words~\citep{chelba2013one}. Generation is unconditional, with $128$ tokens from the
$30{,}522$-token \texttt{bert-base-uncased} vocabulary generated in $8$ flow-map steps, at softmax
temperature $1$ and without per-step noise injection ($\gamma=0$). We use FMRG steering and feedback
budgets of $n=256$ numerical rewards or $M=256$ rankings of $m=8$ candidates, labelled by the same
Plackett--Luce annotator as ImageNet ($\sigma=0.5$). Each evaluated setting is run with $5$ seeds of
$256$ generated samples; metrics are computed per seed and averaged, and $|\mathrm{gap}|$ is taken
from the seed-averaged mean.

The oracle is negative per-sequence generative perplexity under Qwen2.5-1.5B~\citep{qwen2025qwen25technicalreport},
scored in fp32 after removing the \texttt{[CLS]} separators that LM1B's sentence packing leaves inside
each sample. The surrogates are unmasked pseudo-log-likelihoods from BERT and DistilBERT
(\texttt{bert-base-uncased}, \texttt{distilbert-base-uncased}; the mean log-probability the MLM head
assigns to each token actually present, in a single forward pass without masking), a CoLA
acceptability score (the acceptable-minus-unacceptable logit margin of
\texttt{textattack/bert-base-uncased-CoLA}), and the empirical unigram entropy of the sample. All four
share the generator's vocabulary, so their differentiable extensions read a distribution over tokens
directly: it is multiplied into each scorer's word-embedding matrix (between the \texttt{[CLS]} and
\texttt{[SEP]} embeddings), and the entropy is taken of the position-averaged distribution. On one-hot
inputs these reduce exactly to the discrete scores. Held-out loss selects an MLP for values and a
linear composition for preferences.

\xhdr{Pools and steering} For each run, $2048$ unsteered samples are scored by the oracle and all
surrogates: $1536$ rows supply the feedback and $512$ are held out for recovery, and the whole pool,
importance-weighted, forms the $q^*$ reference (ESS $1155$ of $2048$). FMRG follows each flow-map step
with two unit-normalised reward-gradient steps, $x\leftarrow x+\lambda\,\nabla_x r/\lVert\nabla_x
r\rVert$, where $r$ is evaluated on the endpoint prediction $X_{t,1}(x)$ --- one call of the
two-time denoiser, on the vocabulary simplex --- at every step. The strength grid is
$\lambda\in\{0.00375, 0.0075, 0.015, 0.03, 0.06, 0.125, 0.25, 0.5, 1, 2, 4\}$; figures and tables show
$\lambda\le0.125$, where the methods separate. Candidates may recur across ranking queries, so the
$M=256$ groups contain $2048$ candidate occurrences
rather than $2048$ distinct samples.

\subsection{Licenses and attribution}
The \href{https://huggingface.co/adh1s/mfm}{MFM} and
\href{https://huggingface.co/kyungmnlee/DMF}{DMF} checkpoints,
\href{https://github.com/willisma/SiT/blob/main/LICENSE.txt}{SiT},
\href{https://huggingface.co/stabilityai/sd-vae-ft-mse}{SD-VAE-ft-MSE},
\href{https://huggingface.co/stabilityai/sd-vae-ft-ema}{SD-VAE-ft-EMA},
\href{https://huggingface.co/Skywork/Skywork-VL-Reward-7B}{Skywork-VL-Reward-7B},
and \href{https://github.com/openai/CLIP/blob/main/LICENSE}{CLIP}
are released under MIT licenses; the
\href{https://github.com/yuvalkirstain/PickScore/blob/main/LICENSE}{PickScore project repository}
also uses MIT.
The \href{https://huggingface.co/Qwen/Qwen2.5-VL-7B-Instruct}{Qwen2.5-VL-7B-Instruct}
backbone and the upstream
\href{https://huggingface.co/xswu/HPSv2}{HPSv2.1} and
\href{https://huggingface.co/zai-org/ImageReward}{ImageReward}
models use Apache-2.0; the upstream reward-model licenses apply to the
\texttt{RE-N-Y} repackagings identified above.

For the underlying image dataset,
\href{https://www.image-net.org/download.php}{ImageNet}
is governed by its research and educational-use agreement.
Dataset terms are distinct from the licenses of the
pretrained models used to generate the experimental pools.

\clearpage
\section{Metric estimation for image and language generation}
\label{app:metric_est}

We distinguish the tilt induced by a fitted composition from
the samples produced by a steerer. Image and language references
use the sampler law $p_{1,\phi}$ and target strength $\beta=1$
on standardized oracle rewards. The sampler's guidance parameter
is swept separately and does not redefine this target.

\xhdr{Reference distribution and effective sample size}
\label{app:ess}
Let $x_i\sim p_{1,\phi}$, $i=1,\ldots,N_{\mathrm{ref}}$,
be reference samples. With oracle mean $a_*$ and standard
deviation $s_*>0$ estimated on this pool, define
\[
z_*(x)=\frac{R^*(x)-a_*}{s_*},
\qquad
w_i^*=\frac{e^{\beta z_*(x_i)}}
{\sum_{j=1}^{N_{\mathrm{ref}}}e^{\beta z_*(x_j)}},
\qquad
\widehat q^*=\sum_iw_i^*\delta_{x_i}.
\]
The same standardization is used for all compared methods
within a task. The effective sample size is
\[
\mathrm{ESS}=\frac1{\sum_i(w_i^*)^2}.
\]
Small ESS indicates that a few reference samples carry most
of the weight. The main ImageNet comparison uses the seven
class--prompt pairs with reference ESS above $1000$ in both
feedback regimes, limiting sensitivity to concentrated
weights. Sensitivity to this restriction is reported
in~\S\ref{app:ablations}.

\xhdr{Composition KL}
Evaluate the fitted score on the same reference samples and standardize it there,
$z_{\hat g}(x)=(\hat g(\mathbf R(x))-a_{\hat g})/s_{\hat g}$, with $a_{\hat g}$ and $s_{\hat g}$ its mean
and standard deviation over the pool; this is the same standardization applied to the steering reward.
Define
\[
w_i^{\hat g}
=\frac{e^{\beta z_{\hat g}(x_i)}}
{\sum_j e^{\beta z_{\hat g}(x_j)}},
\qquad
\widehat{\KL}_g
=\sum_iw_i^*\log\frac{w_i^*}{w_i^{\hat g}}.
\]
Because $\widehat{\KL}_g$ depends on $\hat g$ only through $z_{\hat g}$, it is unchanged by positive affine rescaling of the fitted score. This is the exact KL between the two empirical tilts on the common pool and estimates the corresponding population KL. It requires no steered samples and remains fixed across the sampler-guidance sweep when the score and $\beta$ are fixed. Standardization changes the induced tilt. With population normalization, conditional-mean regression nevertheless minimizes its leading KL coefficient among standardized surrogate scores~(\cref{cor:standardized}); the fourth-order excess-KL result applies to the unscaled conditional mean. Our implementation estimates normalization moments from the reference pool.

\xhdr{Embedding-space energy distance}
Let $P$ be the generated distribution and $\psi$ the
evaluation embedding. For independent $X,X'\sim P$ and
$Y,Y'\sim q^*$, define
\[
\mathcal E_\psi(P,q^*)
=2\E\|\psi(X)-\psi(Y)\|_2
-\E\|\psi(X)-\psi(X')\|_2
-\E\|\psi(Y)-\psi(Y')\|_2.
\]
This compares the embedding distributions: it vanishes when
those distributions agree, but need not identify the complete
image or text distribution. It requires no bandwidth and
includes a within-sample diversity term.
We use DINOv2-base~\citep{oquab2023dinov2} for images
and token-mean-pooled
\texttt{all-mpnet-base-v2}~\citep{reimers2019sentence} for text,
with $L_2$-normalized embeddings.
Reference expectations use self-normalized importance
weights. The resulting finite-sample estimate is not
generally unbiased and can be negative when self-pairs
are excluded.

\xhdr{Sliced Wasserstein distance}
For $B=256$ unit directions $\theta_b$, sampled once using
a fixed seed shared across methods, we report
\[
\widehat{\mathrm{SW}}_1
=\frac1B\sum_{b=1}^B
W_1\!\left(
(\theta_b^\top\psi)_\#\widehat P,
(\theta_b^\top\psi)_\#\widehat q^*
\right),
\]
where $T_\#P$ denotes the distribution of $T(X)$ for $X\sim P$.
Each projected distance is computed between the unweighted
generated samples and weighted reference samples.

\xhdr{Oracle-reward marginal KL}
Let $\widehat f_P$ and $\widehat f_*$ be Gaussian KDEs
of the generated and reference oracle rewards, with the
reference KDE using weights $w_i^*$. We estimate
\[
\KL_{\mathrm{KDE}}
=\int\widehat f_*(r)
\log\frac{\widehat f_*(r)}{\widehat f_P(r)}\,\mathrm dr.
\]
The numerical grid contains $200$ points spanning the
combined observed reward range with $5\%$ padding.
This is a truncated numerical approximation to KL on
the oracle-reward marginal, rather than on the full
generated distribution.

\xhdr{Gaussian approximation to reward-marginal KL}
Let $a_P,s_P^2$ and $a_q,s_q^2$ be the means and variances
of $R^*$ under $P$ and $q^*$, respectively. For positive
variances, the Gaussian approximation in the same direction is
\[
\KL_{\mathcal N}
=\log\frac{s_P}{s_q}
+\frac{s_q^2+(a_q-a_P)^2}{2s_P^2}-\frac12.
\]
Reference moments are estimated using $w_i^*$.

\xhdr{Standardized reward Wasserstein distances}
Let $F_P$ and $F_q$ be the CDFs of $z_*(X)$ under $P$
and $q^*$. For $p\in\{1,2\}$, define
\[
W_p^z
=\left[\int_0^1
|F_q^{-1}(u)-F_P^{-1}(u)|^p\,\mathrm du\right]^{1/p}.
\]
The same oracle standardization is applied to both
distributions. These compare reward marginals, not
distributions in the generation space.

\xhdr{Reward gain and gap}
For a generated law $P$, define
\[
\mathrm{gain}(P)=\E_P[R^*]-\E_{p_{1,\phi}}[R^*],
\qquad
\mathrm{gap}(P)=\E_{q^*}[R^*]-\E_P[R^*].
\]
The absolute gap is $|\mathrm{gap}(P)|$; a negative
signed gap indicates a mean reward above that of the
reference tilt. Matching the mean reward alone does not
establish distributional fidelity.

\xhdr{Language quality and diversity}
Let $\mathrm{NLL}_i$ be the total negative log-likelihood
of sequence $i$ under the evaluation language model and
$T_i$ its number of scored tokens. The per-sequence oracle
uses $R^*(x_i)=-\exp(\mathrm{NLL}_i/T_i)$, whereas
the aggregate quality metric is corpus perplexity,
\[
\mathrm{gPPL}(P)
=\exp\!\left(\frac{\sum_i\mathrm{NLL}_i}{\sum_iT_i}\right).
\]
Self-BLEU is the mean BLEU score of each generated sample
against the others, computed on the generator's token IDs.
Higher Self-BLEU indicates greater overlap across samples. We use \emph{diversity collapse} to describe outputs
dominated by repeated tokens or near-duplicate sequences,
assessed through within-sequence entropy and Self-BLEU.
We report the diversity-adjusted cost
\[
Q(P)=
\frac{\mathrm{gPPL}(P)}{\mathrm{gPPL}(p_{1,\phi})}
\max\!\left(1,
\frac{\mathrm{sBLEU}(P)}{\mathrm{sBLEU}(p_{1,\phi})}
\right).
\]
It equals $1$ for the prior, and lower is better. The
one-sided factor penalizes increased repetition without
rewarding reduced repetition; it does not ensure fluency
or rule out all forms of reward hacking.

\xhdr{Image quality and diversity}
Let $\mathbb{E}_P[R^*]$ be the mean oracle reward over the samples of method $P$, and
$\mathbb{E}_{q^*}[R^*]$ the target's, estimated by self-normalised importance weighting of the reference
pool. We report the reward error $|\mathrm{gap}|(P)=\lvert\mathbb{E}_{q^*}[R^*]-\mathbb{E}_P[R^*]\rvert$,
which penalises falling short of $q^*$ and overshooting it alike. Diversity $\mathrm{Div}(P)$ is the mean
pairwise cosine distance, $1-\cos(f_i,f_j)$ averaged over all pairs $i<j$, between the DINOv2 embeddings
(\texttt{dinov2-base}, images resized to $224\times224$) of the decoded samples; higher indicates greater
spread. For images, we combine standardized oracle-reward gain
with an analogous one-sided diversity penalty:
\[
Q(P)=
\exp\!\left(-\frac{\E_P[R^*]-\E_{p_{1,\phi}}[R^*]}{s_*}\right)
\max\!\left(1,
\frac{\mathrm{Div}(p_{1,\phi})}{\mathrm{Div}(P)}\right).
\]
where $s_*$ is the standard deviation of $R^*$ over the reference pool, the scale on which the tilt
$z(R^*)$ is defined. $Q$ equals $1$ for the prior and lower is better. Because it is always positive, the
one-sided factor always penalises lost diversity without rewarding added diversity. As for language, $Q$
rewards reward gain without limit, including gain beyond $\mathbb{E}_{q^*}[R^*]$, so fidelity to the target
is read from $|\mathrm{gap}|$ and the distributional metrics. $Q$ is computed per class--prompt pair and
then averaged.

\clearpage
\section{Composition selection and recovery}
\label{app:add-quan}
\subsection{Language generation}
We compare correlation, linear, and MLP compositions using
held-out fitting loss. MSE selects an MLP for value feedback,
and PL loss selects a linear composition for preferences.
The recovery and composition-KL comparisons appear
in~\cref{fig:lang-recovery,fig:lang-klg}.

\begin{figure}[h!]
  \centering
  \begin{subfigure}{\textwidth}
    \centering
    \includegraphics[width=\linewidth]{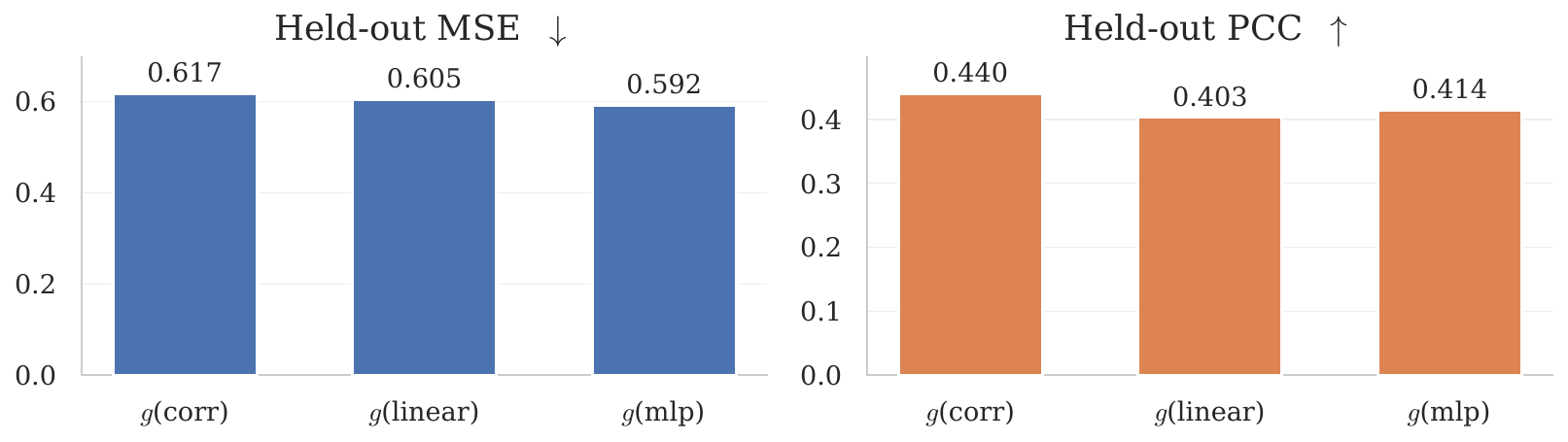}
    \caption{Value feedback.}
  \end{subfigure}

  \medskip
  \begin{subfigure}{\textwidth}
    \centering
    \includegraphics[width=\linewidth]{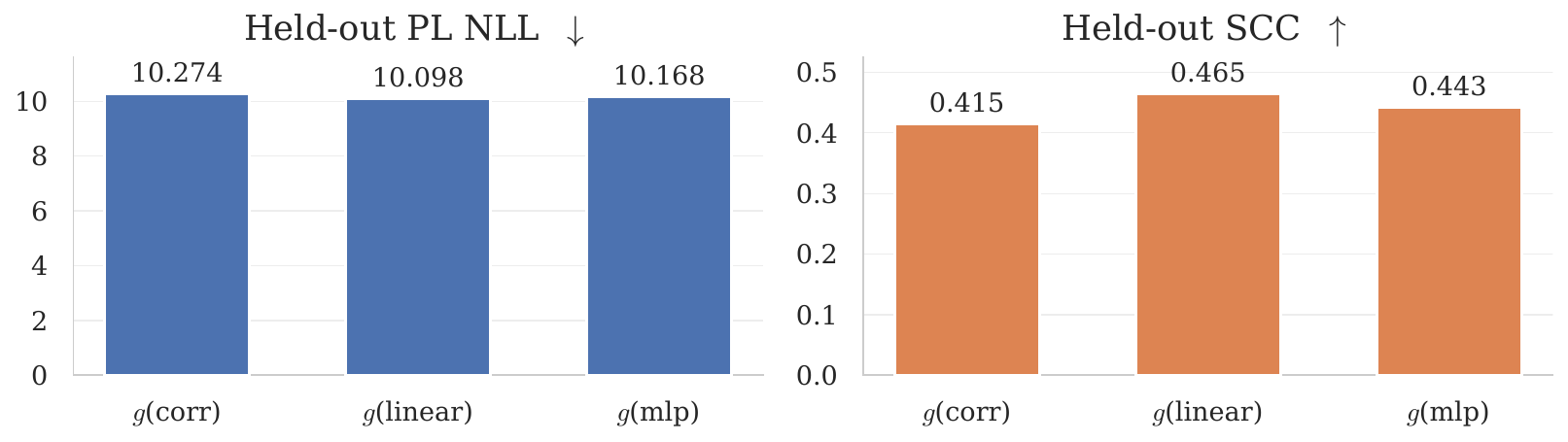}
    \caption{Preference feedback.}
  \end{subfigure}
    \caption{\small\looseness=-1
  \textbf{Language reward recovery.}
  Held-out losses and correlations for the three composition
  families. Lower loss and higher correlation are better.}
  \label{fig:lang-recovery}
\end{figure}

\begin{figure}[h!]
  \centering
  \begin{subfigure}{0.49\textwidth}
    \centering
    \includegraphics[width=\linewidth]{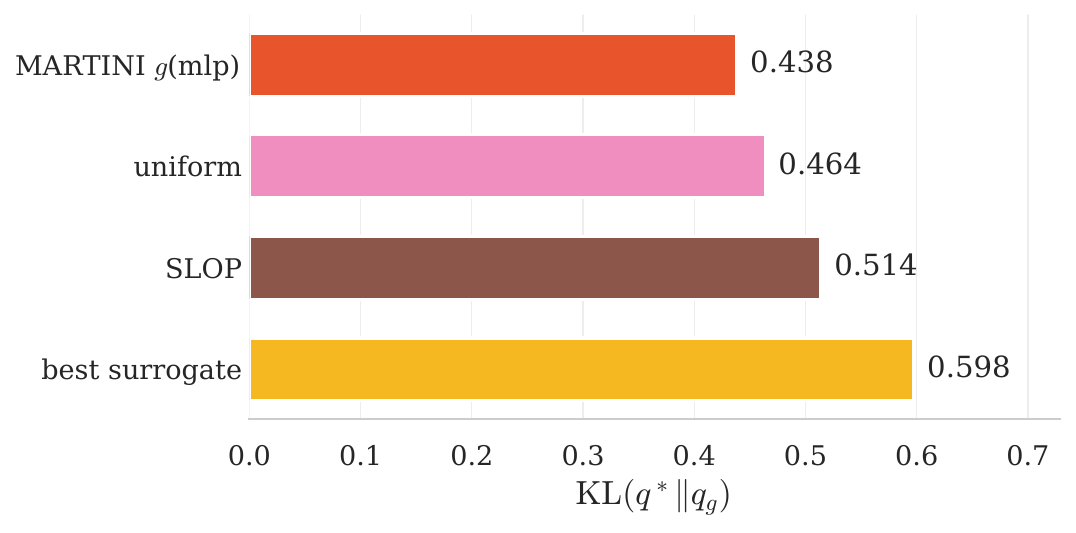}
    \caption{Value feedback.}
  \end{subfigure}
  \hfill
  \begin{subfigure}{0.49\textwidth}
    \centering
    \includegraphics[width=\linewidth]{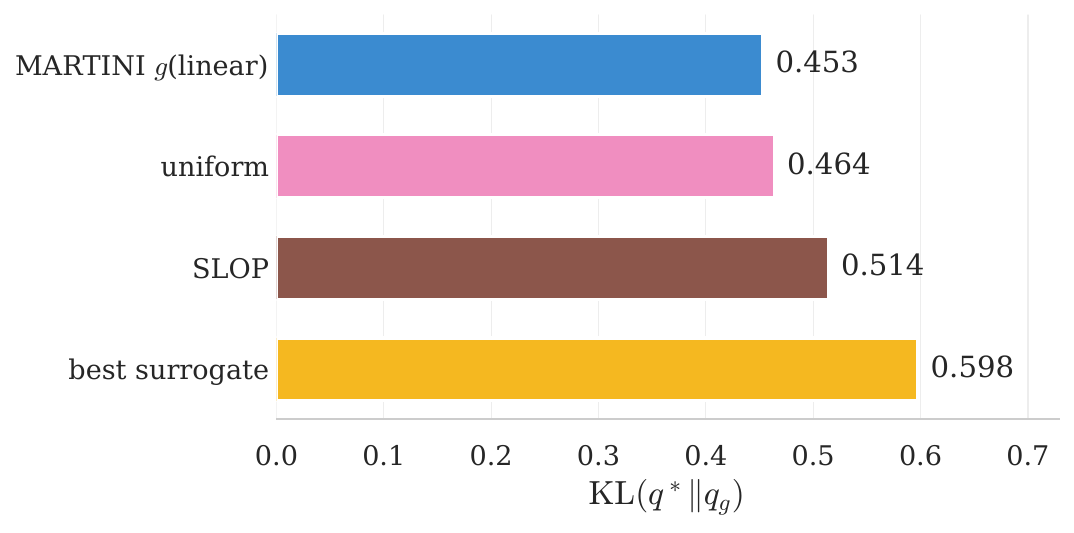}
    \caption{Preference feedback.}
  \end{subfigure}
    \caption{\small\looseness=-1
  \textbf{Language composition fidelity.}
  Estimated $\KL(q^*\|q_g)$ for the selected compositions
  and displayed baselines. Lower is better.}
  \label{fig:lang-klg}
\end{figure}

\clearpage
\subsection{ImageNet}
We average held-out fitting loss across all $19$
class--prompt pairs in~\cref{tab:in-classes}.
The linear family has the lowest loss in both regimes
and is used for steering. Coefficients are fitted
separately for each pair.

\begin{figure}[htbp]
  \centering
  \begin{subfigure}{\linewidth}
    \centering
    \includegraphics[width=0.85\linewidth]{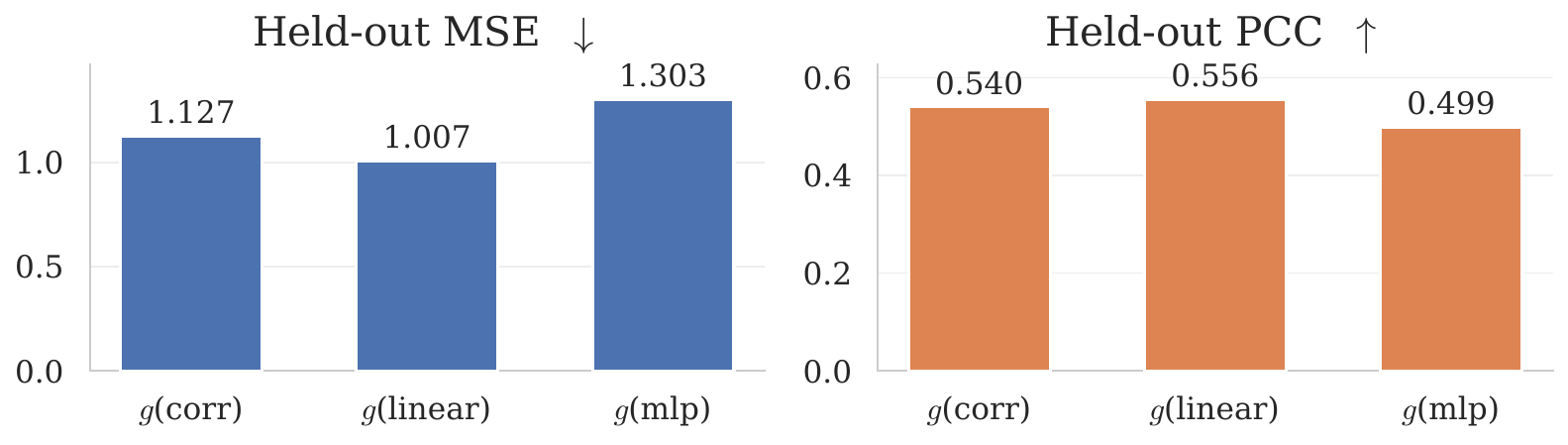}
    \caption{Value regime, $n{=}32$.}
    \label{fig:in-recovery-value}
  \end{subfigure}
  \vspace{4pt}
  \begin{subfigure}{\linewidth}
    \centering
    \includegraphics[width=0.85\linewidth]{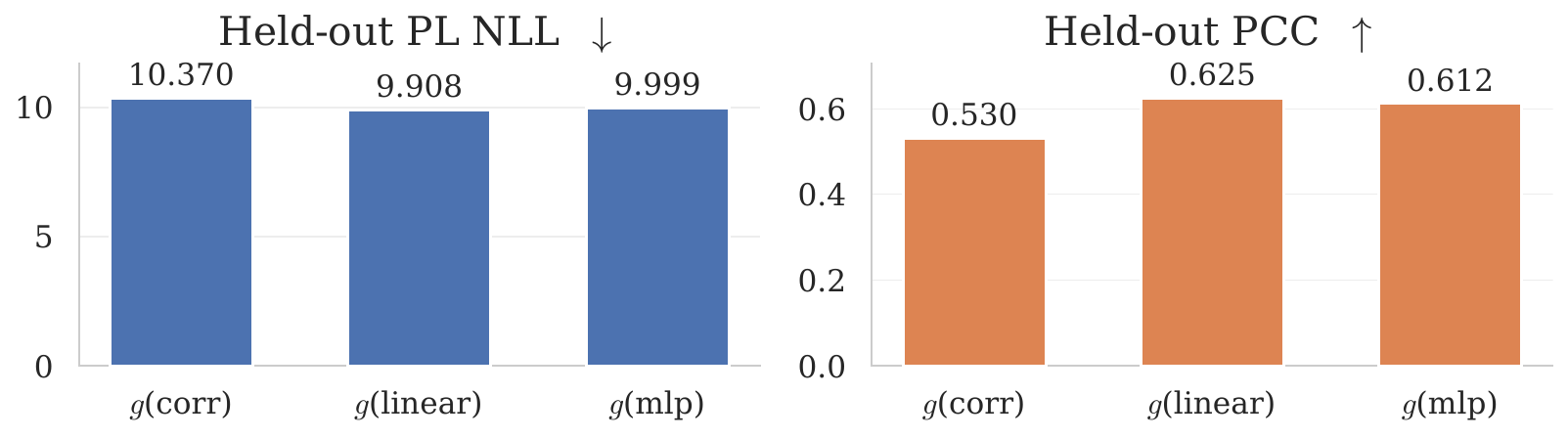}
    \caption{Preference regime, $M{=}128$.}
    \label{fig:in-recovery-pref}
  \end{subfigure}
   \caption{\small\looseness=-1
  \textbf{ImageNet reward recovery.}
  Left: held-out MSE for values and PL loss for preferences
  (lower is better). Right: held-out correlation between
  predicted and oracle rewards (higher is better).}
  \label{fig:in-recovery}
\end{figure}

\clearpage
\section{Additional results and sensitivity analyses}
\label{app:ablations}

\subsection{Language Generation}
\looseness=-1
\Cref{tab:lang-full} reports the complete language results,
including oracle baselines and displayed guidance strengths.

\begingroup
\footnotesize
\setlength{\tabcolsep}{2pt}
\renewcommand{\arraystretch}{1.0}
\setlength{\LTleft}{\fill}
\setlength{\LTright}{\fill}
\setlength{\LTcapwidth}{\linewidth}

\definecolor{AppMartiniVal}{RGB}{219,234,249}
\definecolor{AppMartiniPref}{RGB}{176,206,240}
\definecolor{AppOracle}{RGB}{253,229,203}

\begin{longtable}{@{}l*{10}{r}@{}}
\caption{\small \textbf{Language generation: complete results.} SB denotes Self-BLEU; Un. counts distinct sequences among $256$ outputs. Orange rows use inference-time oracle evaluations. Blue rows show \nameshort\ compositions.}
\phantomsection
\label{tab:lang-full}\\

\toprule
Method & Strength & Energy $\downarrow$ & gPPL $\downarrow$ & SB $\downarrow$ & Un. $\uparrow$ & $\mathrm{KL}_{\mathcal N}\!\downarrow$ & $\mathrm{KL}_{\mathrm{KDE}}\!\downarrow$ & $W_1^z\!\downarrow$ & $W_2^z\!\downarrow$ & $|\mathrm{gap}|\!\downarrow$ \\
\midrule
\endfirsthead
\multicolumn{11}{@{}l}{\small\textit{\tablename~\thetable\ (continued)}}\\
\toprule
Method & Strength & Energy $\downarrow$ & gPPL $\downarrow$ & SB $\downarrow$ & Un. $\uparrow$ & $\mathrm{KL}_{\mathcal N}\!\downarrow$ & $\mathrm{KL}_{\mathrm{KDE}}\!\downarrow$ & $W_1^z\!\downarrow$ & $W_2^z\!\downarrow$ & $|\mathrm{gap}|\!\downarrow$ \\
\midrule
\endhead
\midrule
\multicolumn{11}{r@{}}{\footnotesize Continued on next page}\\
\endfoot
\bottomrule
\endlastfoot
Prior $p_{1,\phi}$ & \textemdash & $5.84\!\times\!10^{-4}$ & $94.5$ & $0.232$ & $256$ & $0.380$ & $0.326$ & $0.790$ & $0.858$ & $20.33$ \\
\midrule
\rowcolor{AppOracle}
BoN ($N=8$) & \textemdash & $1.63\!\times\!10^{-2}$ & $64.6$ & $0.909$ & $82$ & $1.119$ & $1.517$ & $0.481$ & $0.582$ & $12.20$ \\
\rowcolor{AppOracle}
BoN ($N=64$) & \textemdash & $1.33\!\times\!10^{-1}$ & $49.4$ & $0.987$ & $17$ & $20.317$ & $14.635$ & $1.095$ & $1.219$ & $28.02$ \\
\rowcolor{AppOracle}
sBoN ($N=8$) & \textemdash & $5.40\!\times\!10^{-3}$ & $79.3$ & $0.773$ & $141$ & $0.024$ & $0.037$ & $0.162$ & $0.184$ & $3.87$ \\
\rowcolor{AppOracle}
sBoN ($N=64$) & \textemdash & $7.32\!\times\!10^{-3}$ & $76.0$ & $0.763$ & $137$ & $0.009$ & $0.034$ & $0.091$ & $0.125$ & $0.62$ \\
\midrule
Uniform & $0.00375$ & $5.56\!\times\!10^{-4}$ & $93.4$ & $0.233$ & $256$ & $0.350$ & $0.299$ & $0.743$ & $0.808$ & $19.12$ \\
Uniform & $0.015$ & $6.56\!\times\!10^{-4}$ & $93.6$ & $0.234$ & $256$ & $0.351$ & $0.303$ & $0.747$ & $0.810$ & $19.24$ \\
Uniform & $0.03$ & $6.54\!\times\!10^{-4}$ & $94.5$ & $0.233$ & $256$ & $0.376$ & $0.327$ & $0.784$ & $0.839$ & $20.18$ \\
Uniform & $0.06$ & $8.02\!\times\!10^{-4}$ & $95.8$ & $0.228$ & $256$ & $0.415$ & $0.356$ & $0.841$ & $0.906$ & $21.66$ \\
Uniform & $0.125$ & $1.30\!\times\!10^{-3}$ & $95.3$ & $0.225$ & $256$ & $0.403$ & $0.342$ & $0.828$ & $0.902$ & $21.31$ \\
\midrule
BSS & $0.00375$ & $5.20\!\times\!10^{-4}$ & $94.3$ & $0.233$ & $256$ & $0.374$ & $0.319$ & $0.779$ & $0.848$ & $20.06$ \\
BSS & $0.015$ & $8.72\!\times\!10^{-4}$ & $96.5$ & $0.232$ & $256$ & $0.437$ & $0.380$ & $0.873$ & $0.952$ & $22.49$ \\
BSS & $0.03$ & $8.30\!\times\!10^{-4}$ & $99.6$ & $0.230$ & $256$ & $0.523$ & $0.466$ & $0.995$ & $1.072$ & $25.61$ \\
BSS & $0.06$ & $1.28\!\times\!10^{-3}$ & $103.8$ & $0.224$ & $256$ & $0.650$ & $0.602$ & $1.164$ & $1.247$ & $29.98$ \\
BSS & $0.125$ & $2.52\!\times\!10^{-3}$ & $108.4$ & $0.216$ & $256$ & $0.782$ & $0.717$ & $1.367$ & $1.476$ & $35.20$ \\
\midrule
ParetoFlow & $0.00375$ & $7.65\!\times\!10^{-2}$ & $92.0$ & $0.975$ & $140$ & $0.340$ & $0.368$ & $0.707$ & $0.855$ & $18.01$ \\
ParetoFlow & $0.015$ & $6.92\!\times\!10^{-2}$ & $93.1$ & $0.973$ & $148$ & $0.380$ & $0.337$ & $0.771$ & $0.955$ & $19.70$ \\
ParetoFlow & $0.03$ & $5.81\!\times\!10^{-2}$ & $87.9$ & $0.972$ & $154$ & $0.258$ & $0.247$ & $0.575$ & $0.719$ & $14.02$ \\
ParetoFlow & $0.06$ & $6.15\!\times\!10^{-2}$ & $90.4$ & $0.969$ & $155$ & $0.335$ & $0.388$ & $0.652$ & $0.779$ & $16.41$ \\
ParetoFlow & $0.125$ & $8.38\!\times\!10^{-2}$ & $84.6$ & $0.970$ & $133$ & $0.272$ & $0.422$ & $0.479$ & $0.551$ & $9.76$ \\
\midrule
SLOP & \textemdash & $1.77\!\times\!10^{-2}$ & $81.4$ & $0.915$ & $75$ & $0.068$ & $0.069$ & $0.245$ & $0.326$ & $6.09$ \\
\midrule
\multicolumn{11}{@{}l}{\nameshort: \textit{values}} \\*
\rowcolor{AppMartiniVal}
$g(\mathrm{corr})$ & $0.00375$ & $5.56\!\times\!10^{-4}$ & $93.0$ & $0.234$ & $256$ & $0.337$ & $0.286$ & $0.726$ & $0.793$ & $18.69$ \\
\rowcolor{AppMartiniVal}
$g(\mathrm{corr})$ & $0.015$ & $5.56\!\times\!10^{-4}$ & $91.0$ & $0.239$ & $256$ & $0.281$ & $0.239$ & $0.642$ & $0.704$ & $16.52$ \\
\rowcolor{AppMartiniVal}
$g(\mathrm{corr})$ & $0.03$ & $4.78\!\times\!10^{-4}$ & $89.8$ & $0.242$ & $256$ & $0.250$ & $0.212$ & $0.593$ & $0.649$ & $15.26$ \\
\rowcolor{AppMartiniVal}
$g(\mathrm{corr})$ & $0.06$ & $6.80\!\times\!10^{-4}$ & $88.8$ & $0.243$ & $256$ & $0.226$ & $0.185$ & $0.555$ & $0.610$ & $14.28$ \\
\rowcolor{AppMartiniVal}
$g(\mathrm{corr})$ & $0.125$ & $9.72\!\times\!10^{-4}$ & $88.0$ & $0.246$ & $256$ & $0.200$ & $0.167$ & $0.517$ & $0.557$ & $13.29$ \\
\midrule
\rowcolor{AppMartiniVal}
$g(\mathrm{linear})$ & $0.00375$ & $5.02\!\times\!10^{-4}$ & $92.6$ & $0.235$ & $256$ & $0.323$ & $0.275$ & $0.707$ & $0.771$ & $18.21$ \\
\rowcolor{AppMartiniVal}
$g(\mathrm{linear})$ & $0.015$ & $3.82\!\times\!10^{-4}$ & $89.9$ & $0.239$ & $256$ & $0.249$ & $0.214$ & $0.596$ & $0.644$ & $15.33$ \\
\rowcolor{AppMartiniVal}
$g(\mathrm{linear})$ & $0.03$ & $4.66\!\times\!10^{-4}$ & $87.8$ & $0.245$ & $256$ & $0.196$ & $0.162$ & $0.508$ & $0.558$ & $13.09$ \\
\rowcolor{AppMartiniVal}
$g(\mathrm{linear})$ & $0.06$ & $9.68\!\times\!10^{-4}$ & $86.3$ & $0.250$ & $256$ & $0.160$ & $0.133$ & $0.447$ & $0.491$ & $11.50$ \\
\rowcolor{AppMartiniVal}
$g(\mathrm{linear})$ & $0.125$ & $1.69\!\times\!10^{-3}$ & $84.8$ & $0.256$ & $256$ & $0.124$ & $0.109$ & $0.385$ & $0.417$ & $9.90$ \\
\midrule
\rowcolor{AppMartiniVal}
$g(\mathrm{mlp})$ & $0.00375$ & $6.02\!\times\!10^{-4}$ & $92.9$ & $0.234$ & $256$ & $0.333$ & $0.282$ & $0.721$ & $0.784$ & $18.56$ \\
\rowcolor{AppMartiniVal}
$g(\mathrm{mlp})$ & $0.015$ & $6.28\!\times\!10^{-4}$ & $91.4$ & $0.239$ & $256$ & $0.292$ & $0.245$ & $0.661$ & $0.718$ & $17.01$ \\
\rowcolor{AppMartiniVal}
$g(\mathrm{mlp})$ & $0.03$ & $4.66\!\times\!10^{-4}$ & $90.2$ & $0.241$ & $256$ & $0.265$ & $0.209$ & $0.615$ & $0.681$ & $15.84$ \\
\rowcolor{AppMartiniVal}
$g(\mathrm{mlp})$ & $0.06$ & $7.72\!\times\!10^{-4}$ & $88.7$ & $0.245$ & $256$ & $0.225$ & $0.176$ & $0.552$ & $0.610$ & $14.21$ \\
\rowcolor{AppMartiniVal}
$g(\mathrm{mlp})$ & $0.125$ & $1.07\!\times\!10^{-3}$ & $88.1$ & $0.249$ & $256$ & $0.213$ & $0.171$ & $0.529$ & $0.590$ & $13.60$ \\
\midrule
\multicolumn{11}{@{}l}{\nameshort: \textit{preferences}} \\*
\rowcolor{AppMartiniPref}
$g(\mathrm{corr})$ & $0.00375$ & $5.14\!\times\!10^{-4}$ & $93.6$ & $0.234$ & $256$ & $0.356$ & $0.302$ & $0.753$ & $0.820$ & $19.37$ \\
\rowcolor{AppMartiniPref}
$g(\mathrm{corr})$ & $0.015$ & $5.58\!\times\!10^{-4}$ & $93.9$ & $0.235$ & $256$ & $0.362$ & $0.311$ & $0.763$ & $0.831$ & $19.63$ \\
\rowcolor{AppMartiniPref}
$g(\mathrm{corr})$ & $0.03$ & $5.88\!\times\!10^{-4}$ & $94.9$ & $0.235$ & $256$ & $0.390$ & $0.335$ & $0.806$ & $0.877$ & $20.74$ \\
\rowcolor{AppMartiniPref}
$g(\mathrm{corr})$ & $0.06$ & $8.28\!\times\!10^{-4}$ & $95.9$ & $0.233$ & $256$ & $0.420$ & $0.370$ & $0.843$ & $0.904$ & $21.71$ \\
\rowcolor{AppMartiniPref}
$g(\mathrm{corr})$ & $0.125$ & $1.00\!\times\!10^{-3}$ & $97.6$ & $0.228$ & $256$ & $0.469$ & $0.406$ & $0.920$ & $0.994$ & $23.69$ \\
\midrule
\rowcolor{AppMartiniPref}
$g(\mathrm{linear})$ & $0.00375$ & $5.14\!\times\!10^{-4}$ & $93.2$ & $0.234$ & $256$ & $0.343$ & $0.289$ & $0.736$ & $0.804$ & $18.94$ \\
\rowcolor{AppMartiniPref}
$g(\mathrm{linear})$ & $0.015$ & $6.24\!\times\!10^{-4}$ & $92.0$ & $0.236$ & $256$ & $0.306$ & $0.266$ & $0.682$ & $0.737$ & $17.55$ \\
\rowcolor{AppMartiniPref}
$g(\mathrm{linear})$ & $0.03$ & $5.02\!\times\!10^{-4}$ & $91.5$ & $0.238$ & $256$ & $0.296$ & $0.247$ & $0.666$ & $0.727$ & $17.15$ \\
\rowcolor{AppMartiniPref}
$g(\mathrm{linear})$ & $0.06$ & $7.68\!\times\!10^{-4}$ & $90.5$ & $0.238$ & $256$ & $0.269$ & $0.224$ & $0.625$ & $0.681$ & $16.09$ \\
\rowcolor{AppMartiniPref}
$g(\mathrm{linear})$ & $0.125$ & $9.92\!\times\!10^{-4}$ & $90.1$ & $0.238$ & $256$ & $0.258$ & $0.222$ & $0.609$ & $0.660$ & $15.66$ \\
\midrule
\rowcolor{AppMartiniPref}
$g(\mathrm{mlp})$ & $0.00375$ & $4.56\!\times\!10^{-4}$ & $93.4$ & $0.234$ & $256$ & $0.353$ & $0.295$ & $0.749$ & $0.826$ & $19.28$ \\
\rowcolor{AppMartiniPref}
$g(\mathrm{mlp})$ & $0.015$ & $5.88\!\times\!10^{-4}$ & $92.2$ & $0.237$ & $256$ & $0.313$ & $0.267$ & $0.692$ & $0.752$ & $17.80$ \\
\rowcolor{AppMartiniPref}
$g(\mathrm{mlp})$ & $0.03$ & $3.66\!\times\!10^{-4}$ & $91.8$ & $0.239$ & $256$ & $0.306$ & $0.250$ & $0.681$ & $0.749$ & $17.54$ \\
\rowcolor{AppMartiniPref}
$g(\mathrm{mlp})$ & $0.06$ & $7.14\!\times\!10^{-4}$ & $90.7$ & $0.244$ & $256$ & $0.276$ & $0.229$ & $0.636$ & $0.696$ & $16.36$ \\
\rowcolor{AppMartiniPref}
$g(\mathrm{mlp})$ & $0.125$ & $1.11\!\times\!10^{-3}$ & $91.1$ & $0.245$ & $256$ & $0.284$ & $0.240$ & $0.649$ & $0.703$ & $16.71$ \\
\end{longtable}
\endgroup

\subsection{ImageNet}
\looseness=-1
\Cref{tab:in-classes} lists the class--prompt pairs and reference
ESS values. Full guidance sweeps and per-pair \nameshort\
results appear in~\cref{tab:in-strength-sweep,tab:in-per-class}.

\begingroup
\footnotesize
\setlength{\tabcolsep}{3pt}
\renewcommand{\arraystretch}{1.15}
\setlength{\LTleft}{\fill}
\setlength{\LTright}{\fill}
\setlength{\LTcapwidth}{\linewidth}

\definecolor{AppMartiniVal}{RGB}{219,234,249}
\definecolor{AppMartiniPref}{RGB}{176,206,240}
\definecolor{AppOracle}{RGB}{253,229,203}

\begin{longtable}{@{}rl>{\raggedright\arraybackslash}p{0.48\linewidth}rr@{}}
\caption{\small \textbf{ImageNet class--prompt pairs.} ESS is computed from each regime\textquotesingle s independent $8192$-sample reference pool; displayed values are rounded. $\dagger$ marks the seven pairs with reference ESS above $1000$ in both regimes.}
\phantomsection
\label{tab:in-classes}\\
\toprule
ID & Class & Steering prompt & ESS (Val.) & ESS (Pref.) \\
\midrule
\endfirsthead
\multicolumn{5}{@{}l}{\small\textit{\tablename~\thetable\ (continued)}}\\
\toprule
ID & Class & Steering prompt & ESS (Val.) & ESS (Pref.) \\
\midrule
\endhead
\midrule
\multicolumn{5}{r@{}}{\footnotesize Continued on next page}\\
\endfoot
\bottomrule
\endlastfoot
555 & fire engine$^{\dagger}$ & a fire engine on a busy city street & 4043 & 4081 \\
294 & brown bear$^{\dagger}$ & a brown bear wading in a mountain river & 3903 & 3826 \\
84 & peacock$^{\dagger}$ & a peacock displaying in a lush garden & 3728 & 3757 \\
145 & king penguin$^{\dagger}$ & a king penguin standing on the ice & 2451 & 2418 \\
347 & bison$^{\dagger}$ & a bison grazing on the open prairie & 2199 & 2291 \\
386 & African elephant$^{\dagger}$ & an African elephant at a watering hole & 2031 & 2001 \\
779 & school bus$^{\dagger}$ & a school bus on a tree-lined suburban street & 1250 & 1304 \\
483 & castle & a castle on a hilltop at dawn & 993 & 828 \\
250 & Siberian husky & a Siberian husky in deep snow & 975 & 975 \\
130 & flamingo & a flamingo wading in shallow water at sunset & 935 & 908 \\
717 & pickup truck & a pickup truck on a dirt road in the countryside & 596 & 640 \\
144 & pelican & a pelican perched on a wooden pier by the sea & 500 & 349 \\
340 & zebra & a zebra grazing on the open savanna & 351 & 374 \\
9 & ostrich & an ostrich running across a dusty plain & 111 & 85 \\
291 & lion & a lion resting on the savanna at sunset & 90 & 44 \\
817 & sports car & a sports car parked on a city street at night & 48 & 34 \\
281 & tabby cat & a tabby cat sitting in the snow & 11 & 13 \\
354 & Arabian camel & an Arabian camel crossing desert dunes at sunset & 3 & 2 \\
207 & golden retriever & a golden retriever running on a sandy beach & 1 & 1 \\
\end{longtable}
\endgroup

\begingroup
\footnotesize
\setlength{\tabcolsep}{2pt}
\renewcommand{\arraystretch}{1.0}
\setlength{\LTleft}{\fill}
\setlength{\LTright}{\fill}
\setlength{\LTcapwidth}{\linewidth}

\definecolor{AppMartiniVal}{RGB}{219,234,249}
\definecolor{AppMartiniPref}{RGB}{176,206,240}
\definecolor{AppOracle}{RGB}{253,229,203}

\begin{longtable}{@{}l*{10}{r}@{}}
\caption{\small \textbf{ImageNet steering-strength sweep.} Means over the seven marked pairs in~\cref{tab:in-classes}. Setting gives guidance strength, except for SLOP, where it gives candidate count $N$. Composition KL is shown once per fitted composition. Lower is better for all metrics; SLOP is value-calibrated.}
\phantomsection
\label{tab:in-strength-sweep}\\
\toprule
Method & Setting & $\mathrm{KL}_g$ & Energy & Sl.\,$W_1$ & $\mathrm{KL}_{\mathcal N}$ & $\mathrm{KL}_{\mathrm{KDE}}$ & $W_1^z$ & $W_2^z$ & $|\mathrm{gap}|$ & $Q$ \\
\midrule
\endfirsthead
\multicolumn{11}{@{}l}{\small\textit{\tablename~\thetable\ (continued)}}\\
\toprule
Method & Setting & $\mathrm{KL}_g$ & Energy & Sl.\,$W_1$ & $\mathrm{KL}_{\mathcal N}$ & $\mathrm{KL}_{\mathrm{KDE}}$ & $W_1^z$ & $W_2^z$ & $|\mathrm{gap}|$ & $Q$ \\
\midrule
\endhead
\midrule
\multicolumn{11}{r@{}}{\footnotesize Continued on next page}\\
\endfoot
\bottomrule
\endlastfoot
Prior $p_{1,\phi}$ & \textemdash & \textemdash & $0.0683$ & $0.0080$ & $0.752$ & $0.658$ & $1.253$ & $1.380$ & $7.51$ & $1.000$ \\
\midrule
\rowcolor{AppOracle}
BoN ($N=8$) & \textemdash & \textemdash & $0.0654$ & $0.0078$ & $0.501$ & $0.671$ & $0.546$ & $0.658$ & $3.17$ & $0.193$ \\
\rowcolor{AppOracle}
BoN ($N=16$) & \textemdash & \textemdash & $0.0795$ & $0.0084$ & $2.918$ & $3.121$ & $0.951$ & $1.120$ & $5.53$ & $0.142$ \\
\rowcolor{AppOracle}
BoN ($N=32$) & \textemdash & \textemdash & $0.0844$ & $0.0085$ & $9.964$ & $7.721$ & $1.271$ & $1.473$ & $7.44$ & $0.106$ \\
\midrule
Uniform & $2$ & $0.436$ & $0.0832$ & $0.0086$ & $0.044$ & $0.089$ & $0.311$ & $0.374$ & $1.29$ & $0.418$ \\
Uniform & $3$ &  & $0.0688$ & $0.0078$ & $0.069$ & $0.091$ & $0.291$ & $0.332$ & $1.47$ & $0.291$ \\
Uniform & $5$ &  & $0.0648$ & $0.0076$ & $0.696$ & $0.646$ & $0.544$ & $0.627$ & $3.19$ & $0.222$ \\
Uniform & $8$ &  & $0.0813$ & $0.0081$ & $1.415$ & $1.123$ & $0.631$ & $0.725$ & $3.55$ & $0.207$ \\
\midrule
ReNO & $2$ & $0.553$ & $0.0749$ & $0.0082$ & $0.051$ & $0.090$ & $0.309$ & $0.379$ & $1.36$ & $0.400$ \\
ReNO & $3$ &  & $0.0578$ & $0.0074$ & $0.039$ & $0.080$ & $0.248$ & $0.289$ & $0.90$ & $0.314$ \\
ReNO & $5$ &  & $0.0527$ & $0.0070$ & $0.238$ & $0.196$ & $0.411$ & $0.465$ & $2.34$ & $0.235$ \\
ReNO & $8$ &  & $0.0812$ & $0.0081$ & $1.419$ & $0.893$ & $0.552$ & $0.634$ & $3.17$ & $0.218$ \\
\midrule
ParetoFlow & $2\ (t=0)$ & \textemdash & $0.1696$ & $0.0116$ & $0.490$ & $0.528$ & $0.932$ & $1.086$ & $5.19$ & $0.944$ \\
ParetoFlow & $3\ (t=0)$ &  & $0.1422$ & $0.0109$ & $0.404$ & $0.438$ & $0.917$ & $1.079$ & $5.08$ & $0.801$ \\
ParetoFlow & $5\ (t=0)$ &  & $0.1405$ & $0.0108$ & $0.550$ & $0.498$ & $0.885$ & $1.012$ & $4.96$ & $0.811$ \\
ParetoFlow & $8\ (t=0)$ &  & $0.1282$ & $0.0101$ & $0.145$ & $0.225$ & $0.465$ & $0.541$ & $1.92$ & $0.447$ \\
\midrule
SLOP & $N=8$ & $0.520$ & $0.1020$ & $0.0091$ & $0.077$ & $0.091$ & $0.328$ & $0.388$ & $1.38$ & $0.359$ \\
SLOP & $N=16$ &  & $0.1200$ & $0.0098$ & $0.168$ & $0.170$ & $0.365$ & $0.446$ & $1.98$ & $0.364$ \\
SLOP & $N=32$ &  & $0.1147$ & $0.0095$ & $0.420$ & $0.327$ & $0.551$ & $0.643$ & $3.08$ & $0.273$ \\
\midrule
\multicolumn{11}{@{}l}{\textit{Values ($n=32$)}} \\*
BSS & $2$ & $0.514$ & $0.0674$ & $0.0078$ & $0.131$ & $0.177$ & $0.539$ & $0.635$ & $2.79$ & $0.532$ \\
BSS & $3$ &  & $0.0535$ & $0.0070$ & $0.113$ & $0.142$ & $0.455$ & $0.548$ & $2.23$ & $0.468$ \\
BSS & $5$ &  & $0.0529$ & $0.0069$ & $0.162$ & $0.144$ & $0.385$ & $0.432$ & $1.94$ & $0.340$ \\
BSS & $8$ &  & $0.0960$ & $0.0088$ & $0.157$ & $0.161$ & $0.399$ & $0.456$ & $2.03$ & $0.330$ \\
\midrule
\rowcolor{AppMartiniVal}
$g(\mathrm{corr})$ & $2$ & $0.425$ & $0.0673$ & $0.0078$ & $0.059$ & $0.097$ & $0.348$ & $0.431$ & $1.83$ & $0.434$ \\
\rowcolor{AppMartiniVal}
$g(\mathrm{corr})$ & $3$ &  & $0.0618$ & $0.0075$ & $0.062$ & $0.084$ & $0.300$ & $0.339$ & $1.61$ & $0.314$ \\
\rowcolor{AppMartiniVal}
$g(\mathrm{corr})$ & $5$ &  & $0.0596$ & $0.0074$ & $0.774$ & $0.625$ & $0.531$ & $0.605$ & $2.98$ & $0.247$ \\
\rowcolor{AppMartiniVal}
$g(\mathrm{corr})$ & $8$ &  & $0.0778$ & $0.0080$ & $1.317$ & $0.991$ & $0.698$ & $0.800$ & $4.02$ & $0.241$ \\
\midrule
\rowcolor{AppMartiniVal}
$g(\mathrm{linear})$ & $2$ & $0.466$ & $0.0559$ & $0.0073$ & $0.104$ & $0.124$ & $0.471$ & $0.556$ & $2.69$ & $0.495$ \\
\rowcolor{AppMartiniVal}
$g(\mathrm{linear})$ & $3$ &  & $0.0409$ & $0.0064$ & $0.039$ & $0.069$ & $0.288$ & $0.352$ & $1.40$ & $0.388$ \\
\rowcolor{AppMartiniVal}
$g(\mathrm{linear})$ & $5$ &  & $0.0339$ & $0.0059$ & $0.061$ & $0.069$ & $0.319$ & $0.377$ & $1.64$ & $0.360$ \\
\rowcolor{AppMartiniVal}
$g(\mathrm{linear})$ & $8$ &  & $0.1010$ & $0.0091$ & $0.246$ & $0.312$ & $0.400$ & $0.475$ & $1.88$ & $0.431$ \\
\midrule
\rowcolor{AppMartiniVal}
$g(\mathrm{mlp})$ & $2$ & $0.488$ & $0.0525$ & $0.0071$ & $0.186$ & $0.204$ & $0.653$ & $0.744$ & $3.80$ & $0.594$ \\
\rowcolor{AppMartiniVal}
$g(\mathrm{mlp})$ & $3$ &  & $0.0355$ & $0.0060$ & $0.068$ & $0.091$ & $0.359$ & $0.433$ & $1.97$ & $0.433$ \\
\rowcolor{AppMartiniVal}
$g(\mathrm{mlp})$ & $5$ &  & $0.0319$ & $0.0058$ & $0.165$ & $0.168$ & $0.410$ & $0.492$ & $2.18$ & $0.436$ \\
\rowcolor{AppMartiniVal}
$g(\mathrm{mlp})$ & $8$ &  & $0.0956$ & $0.0090$ & $0.277$ & $0.316$ & $0.514$ & $0.600$ & $2.80$ & $0.434$ \\
\midrule
\multicolumn{11}{@{}l}{\textit{Preferences ($M=128$)}} \\*
BSS & $2$ & $0.440$ & $0.0628$ & $0.0077$ & $0.066$ & $0.102$ & $0.368$ & $0.463$ & $1.79$ & $0.429$ \\
BSS & $3$ &  & $0.0552$ & $0.0073$ & $0.035$ & $0.080$ & $0.258$ & $0.306$ & $0.90$ & $0.359$ \\
BSS & $5$ &  & $0.0425$ & $0.0065$ & $0.133$ & $0.134$ & $0.388$ & $0.433$ & $2.00$ & $0.312$ \\
BSS & $8$ &  & $0.0861$ & $0.0084$ & $0.264$ & $0.254$ & $0.427$ & $0.494$ & $2.08$ & $0.323$ \\
\midrule
\rowcolor{AppMartiniPref}
$g(\mathrm{corr})$ & $2$ & $0.417$ & $0.0752$ & $0.0082$ & $0.045$ & $0.089$ & $0.314$ & $0.366$ & $1.10$ & $0.385$ \\
\rowcolor{AppMartiniPref}
$g(\mathrm{corr})$ & $3$ &  & $0.0716$ & $0.0080$ & $0.048$ & $0.079$ & $0.314$ & $0.348$ & $1.57$ & $0.284$ \\
\rowcolor{AppMartiniPref}
$g(\mathrm{corr})$ & $5$ &  & $0.0638$ & $0.0075$ & $1.013$ & $0.866$ & $0.578$ & $0.648$ & $3.50$ & $0.199$ \\
\rowcolor{AppMartiniPref}
$g(\mathrm{corr})$ & $8$ &  & $0.0786$ & $0.0080$ & $3.831$ & $1.905$ & $0.670$ & $0.777$ & $4.02$ & $0.181$ \\
\midrule
\rowcolor{AppMartiniPref}
$g(\mathrm{linear})$ & $2$ & $0.423$ & $0.0580$ & $0.0074$ & $0.095$ & $0.132$ & $0.403$ & $0.490$ & $2.11$ & $0.441$ \\
\rowcolor{AppMartiniPref}
$g(\mathrm{linear})$ & $3$ &  & $0.0484$ & $0.0068$ & $0.053$ & $0.106$ & $0.345$ & $0.398$ & $1.54$ & $0.384$ \\
\rowcolor{AppMartiniPref}
$g(\mathrm{linear})$ & $5$ &  & $0.0380$ & $0.0061$ & $0.097$ & $0.123$ & $0.365$ & $0.423$ & $1.93$ & $0.287$ \\
\rowcolor{AppMartiniPref}
$g(\mathrm{linear})$ & $8$ &  & $0.0780$ & $0.0081$ & $0.400$ & $0.327$ & $0.389$ & $0.454$ & $2.13$ & $0.303$ \\
\midrule
\rowcolor{AppMartiniPref}
$g(\mathrm{mlp})$ & $2$ & $0.426$ & $0.0620$ & $0.0076$ & $0.081$ & $0.107$ & $0.408$ & $0.486$ & $2.11$ & $0.432$ \\
\rowcolor{AppMartiniPref}
$g(\mathrm{mlp})$ & $3$ &  & $0.0516$ & $0.0070$ & $0.070$ & $0.096$ & $0.313$ & $0.365$ & $1.51$ & $0.350$ \\
\rowcolor{AppMartiniPref}
$g(\mathrm{mlp})$ & $5$ &  & $0.0330$ & $0.0059$ & $0.525$ & $0.458$ & $0.384$ & $0.458$ & $1.86$ & $0.310$ \\
\rowcolor{AppMartiniPref}
$g(\mathrm{mlp})$ & $8$ &  & $0.0783$ & $0.0081$ & $0.263$ & $0.248$ & $0.503$ & $0.576$ & $2.56$ & $0.360$ \\
\end{longtable}
\endgroup

\begingroup
\footnotesize
\setlength{\tabcolsep}{1.6pt}
\renewcommand{\arraystretch}{1.05}
\setlength{\LTleft}{\fill}
\setlength{\LTright}{\fill}
\setlength{\LTcapwidth}{\linewidth}

\definecolor{AppMartiniVal}{RGB}{219,234,249}
\definecolor{AppMartiniPref}{RGB}{176,206,240}
\definecolor{AppOracle}{RGB}{253,229,203}

\begin{longtable}{@{}l*{11}{r}@{}}
\caption{\small \textbf{\nameshort\ by ImageNet class--prompt pair.} Linear compositions use each pair\textquotesingle s lowest-energy evaluated strength. ESS is rounded and refers to the regime-specific $8192$-sample pool. $\dagger$ marks pairs in the main comparison. Dashes indicate unreported estimates.}
\phantomsection
\label{tab:in-per-class}\\
\toprule
Class & ESS & Str. & $\mathrm{KL}_g$ & Energy & Sl.\,$W_1$ & $\mathrm{KL}_{\mathcal N}$ & $\mathrm{KL}_{\mathrm{KDE}}$ & $W_1^z$ & $W_2^z$ & $|\mathrm{gap}|$ & $Q$ \\
\midrule
\endfirsthead
\multicolumn{12}{@{}l}{\small\textit{\tablename~\thetable\ (continued)}}\\
\toprule
Class & ESS & Str. & $\mathrm{KL}_g$ & Energy & Sl.\,$W_1$ & $\mathrm{KL}_{\mathcal N}$ & $\mathrm{KL}_{\mathrm{KDE}}$ & $W_1^z$ & $W_2^z$ & $|\mathrm{gap}|$ & $Q$ \\
\midrule
\endhead
\midrule
\multicolumn{12}{r@{}}{\footnotesize Continued on next page}\\
\endfoot
\bottomrule
\endlastfoot
\multicolumn{12}{@{}l}{\textit{Values ($n=32$)}} \\*
fire engine$^{\dagger}$ & $4043$ & $5$ & $0.360$ & $0.0351$ & $0.0063$ & $0.028$ & $0.029$ & $0.187$ & $0.211$ & $1.10$ & $0.277$ \\
brown bear$^{\dagger}$ & $3903$ & $5$ & $0.257$ & $0.0289$ & $0.0056$ & $0.035$ & $0.054$ & $0.140$ & $0.234$ & $0.94$ & $0.241$ \\
peacock$^{\dagger}$ & $3728$ & $5$ & $0.416$ & $0.0229$ & $0.0052$ & $0.033$ & $0.068$ & $0.194$ & $0.229$ & $0.51$ & $0.449$ \\
king penguin$^{\dagger}$ & $2451$ & $3$ & $0.652$ & $0.0352$ & $0.0055$ & $0.011$ & $0.034$ & $0.200$ & $0.228$ & $0.81$ & $0.272$ \\
bison$^{\dagger}$ & $2199$ & $3$ & $0.462$ & $0.0238$ & $0.0052$ & $0.031$ & $0.042$ & $0.296$ & $0.363$ & $1.71$ & $0.546$ \\
African elephant$^{\dagger}$ & $2031$ & $5$ & $0.399$ & $0.0635$ & $0.0075$ & $0.012$ & $0.025$ & $0.172$ & $0.240$ & $0.95$ & $0.179$ \\
school bus$^{\dagger}$ & $1250$ & $2$ & $0.718$ & $0.0152$ & $0.0051$ & $0.086$ & $0.091$ & $0.593$ & $0.670$ & $2.88$ & $0.547$ \\
castle & $993$ & $5$ & $0.720$ & $0.0721$ & $0.0095$ & $0.013$ & $0.031$ & $0.282$ & $0.312$ & $0.78$ & $0.175$ \\
Siberian husky & $975$ & $5$ & $0.525$ & $0.0571$ & $0.0081$ & $2.349$ & $1.117$ & $2.203$ & $2.545$ & $12.48$ & $0.715$ \\
flamingo & $935$ & $3$ & $0.443$ & $0.0429$ & $0.0058$ & $0.021$ & $0.043$ & $0.332$ & $0.372$ & $1.00$ & $0.294$ \\
pickup truck & $596$ & $3$ & $0.959$ & $0.0162$ & $0.0063$ & $0.024$ & $0.077$ & $0.470$ & $0.567$ & $1.60$ & $0.072$ \\
pelican & $500$ & $3$ & $1.010$ & $0.0127$ & $0.0046$ & $0.466$ & $0.364$ & $1.682$ & $1.917$ & $7.27$ & $0.415$ \\
zebra & $351$ & $5$ & $0.916$ & $0.0308$ & $0.0058$ & $0.234$ & $0.174$ & $1.030$ & $1.134$ & $3.95$ & $0.366$ \\
ostrich & $111$ & $2$ & $1.278$ & $0.0302$ & $0.0060$ & $0.213$ & $0.159$ & $1.590$ & $1.803$ & $5.01$ & $0.208$ \\
lion & $90$ & $8$ & $1.828$ & $0.0583$ & $0.0077$ & $0.333$ & $0.285$ & $1.969$ & $2.081$ & $4.46$ & $0.003$ \\
sports car & $48$ & $5$ & $2.317$ & $0.0417$ & $0.0082$ & $0.528$ & $0.436$ & $2.982$ & $3.396$ & $7.54$ & $0.075$ \\
tabby cat & $11$ & $2$ & $1.363$ & $0.0773$ & $0.0130$ & $8.625$ & $2.676$ & $8.810$ & $9.189$ & $22.16$ & $0.526$ \\
Arabian camel & $3$ & $2$ & $3.820$ & $0.0532$ & $0.0140$ & $8.717$ & $3.435$ & $8.824$ & $9.073$ & $12.13$ & $0.717$ \\
golden retriever & $1$ & $5$ & $4.779$ & $0.0924$ & $0.0271$ & $1.815$ & $1.470$ & $12.072$ & $13.751$ & $13.90$ & $0.000$ \\
\midrule
Mean$^{\dagger}$ (7 pairs) &  &  & $0.466$ & $0.0321$ & $0.0058$ & $0.034$ & $0.049$ & $0.255$ & $0.310$ & $1.27$ & $0.359$ \\
\midrule
\multicolumn{12}{@{}l}{\textit{Preferences ($M=128$)}} \\*
fire engine$^{\dagger}$ & $4081$ & $5$ & $0.317$ & $0.0508$ & $0.0072$ & $0.390$ & $0.330$ & $0.615$ & $0.669$ & $3.90$ & $0.195$ \\
brown bear$^{\dagger}$ & $3826$ & $5$ & $0.266$ & $0.0409$ & $0.0064$ & $0.053$ & $0.130$ & $0.282$ & $0.305$ & $1.96$ & $0.209$ \\
peacock$^{\dagger}$ & $3757$ & $5$ & $0.286$ & $0.0208$ & $0.0050$ & $0.092$ & $0.168$ & $0.438$ & $0.500$ & $2.18$ & $0.282$ \\
king penguin$^{\dagger}$ & $2418$ & $5$ & $0.639$ & $0.0352$ & $0.0053$ & $0.068$ & $0.067$ & $0.260$ & $0.300$ & $0.68$ & $0.217$ \\
bison$^{\dagger}$ & $2291$ & $5$ & $0.423$ & $0.0287$ & $0.0056$ & $0.011$ & $0.030$ & $0.184$ & $0.220$ & $0.98$ & $0.492$ \\
African elephant$^{\dagger}$ & $2001$ & $3$ & $0.433$ & $0.0717$ & $0.0083$ & $0.001$ & $0.034$ & $0.131$ & $0.171$ & $0.33$ & $0.224$ \\
school bus$^{\dagger}$ & $1304$ & $5$ & $0.596$ & $0.0123$ & $0.0048$ & $0.034$ & $0.086$ & $0.455$ & $0.602$ & $1.96$ & $0.459$ \\
Siberian husky & $975$ & $5$ & $0.479$ & $0.0515$ & $0.0071$ & $1.233$ & $0.809$ & $1.835$ & $2.147$ & $10.39$ & $0.498$ \\
flamingo & $908$ & $3$ & $0.418$ & $0.0200$ & $0.0045$ & $0.021$ & $0.047$ & $0.370$ & $0.385$ & $1.05$ & $0.286$ \\
castle & $828$ & $3$ & $0.630$ & $0.0524$ & $0.0085$ & $0.008$ & $0.031$ & $0.278$ & $0.335$ & $0.56$ & $0.167$ \\
pickup truck & $640$ & $5$ & $0.797$ & $0.0181$ & $0.0064$ & $0.172$ & $0.173$ & $0.800$ & $0.885$ & $3.30$ & $0.025$ \\
zebra & $374$ & $3$ & $0.885$ & $0.0326$ & $0.0055$ & $0.500$ & $0.318$ & $1.470$ & $1.622$ & $5.51$ & $0.641$ \\
pelican & $349$ & $3$ & $0.720$ & $0.0156$ & $0.0048$ & $0.217$ & $0.162$ & $1.226$ & $1.352$ & $5.35$ & $0.194$ \\
ostrich & $85$ & $2$ & $1.174$ & $0.0167$ & $0.0059$ & $0.102$ & $0.074$ & $1.159$ & $1.333$ & $3.60$ & $0.113$ \\
lion & $44$ & $5$ & $2.179$ & $0.0333$ & $0.0068$ & $0.161$ & $0.224$ & $1.602$ & $1.818$ & $3.21$ & $0.043$ \\
sports car & $34$ & $5$ & $2.648$ & $0.0591$ & $0.0095$ & $0.904$ & $0.629$ & $3.627$ & $3.970$ & $8.75$ & $0.090$ \\
tabby cat & $13$ & $3$ & $2.062$ & $0.0566$ & $0.0124$ & $8.150$ & $3.535$ & $8.472$ & $8.796$ & $20.94$ & $0.339$ \\
Arabian camel & $2$ & $2$ & $4.655$ & \textemdash & $0.0155$ & $12.306$ & $10.363$ & $12.002$ & $12.207$ & $16.79$ & $0.502$ \\
golden retriever & $1$ & $3$ & $3.261$ & \textemdash & $0.0280$ & $6.451$ & $2.279$ & $17.331$ & $17.984$ & $22.68$ & $0.064$ \\
\midrule
Mean$^{\dagger}$ (7 pairs) &  &  & $0.423$ & $0.0372$ & $0.0061$ & $0.093$ & $0.121$ & $0.338$ & $0.395$ & $1.71$ & $0.297$ \\
\end{longtable}
\endgroup

\clearpage
\looseness=-1
\Cref{tab:in-ess} assesses sensitivity to the reference ESS
cutoff. For each cutoff, we aggregate each method's metrics
across retained pairs using the mean or median, then rank
the methods. The comparison includes \nameshort, BSS,
uniform weighting, ReNO, ParetoFlow, and SLOP; composition
KL excludes ParetoFlow. SLOP remains a value-calibrated
reference in the preference regime. Steering methods use
their per-pair lowest-energy evaluated strength. The Siberian
husky pair is retained whenever it meets the cutoff.

\begingroup
\footnotesize
\setlength{\tabcolsep}{2.5pt}
\renewcommand{\arraystretch}{1.05}
\setlength{\LTleft}{\fill}
\setlength{\LTright}{\fill}
\setlength{\LTcapwidth}{\linewidth}

\definecolor{AppMartiniVal}{RGB}{219,234,249}
\definecolor{AppMartiniPref}{RGB}{176,206,240}
\definecolor{AppOracle}{RGB}{253,229,203}

\providecommand{\ESSrank}[2]{%
  \ensuremath{\ifnum#1=1 \mathbf{1}\else #1\fi
  /\ifnum#2=1 \mathbf{1}\else #2\fi}%
}
\begin{longtable}{@{}lr*{9}{c}@{}}
\caption{\small \textbf{Sensitivity to reference ESS.} Each entry gives \nameshort\textquotesingle s rank using mean/median metrics across retained pairs; bold $1$ marks first place. Cutoffs apply separately to each regime\textquotesingle s $8192$-
sample reference pool.}
\phantomsection
\label{tab:in-ess}\\
\toprule
ESS cutoff & Pairs & $\mathrm{KL}_g$ & Energy & Sl.\,$W_1$ & $\mathrm{KL}_{\mathcal N}$ & $\mathrm{KL}_{\mathrm{KDE}}$ & $W_1^z$ & $W_2^z$ & $|\mathrm{gap}|$ & $Q$ \\
\midrule
\endfirsthead
\multicolumn{11}{@{}l}{\small\textit{\tablename~\thetable\ (continued)}}\\
\toprule
ESS cutoff & Pairs & $\mathrm{KL}_g$ & Energy & Sl.\,$W_1$ & $\mathrm{KL}_{\mathcal N}$ & $\mathrm{KL}_{\mathrm{KDE}}$ & $W_1^z$ & $W_2^z$ & $|\mathrm{gap}|$ & $Q$ \\
\midrule
\endhead
\midrule
\multicolumn{11}{r@{}}{\footnotesize Continued on next page}\\
\endfoot
\bottomrule
\endlastfoot
\multicolumn{11}{@{}l}{\textit{Values ($n=32$)}} \\*
$\geq 0$ & 19 & \ESSrank{1}{2} & \ESSrank{1}{1} & \ESSrank{1}{1} & \ESSrank{4}{1} & \ESSrank{2}{1} & \ESSrank{3}{1} & \ESSrank{3}{1} & \ESSrank{4}{2} & \ESSrank{5}{4} \\
$\geq 100$ & 14 & \ESSrank{2}{3} & \ESSrank{1}{1} & \ESSrank{1}{1} & \ESSrank{3}{1} & \ESSrank{2}{1} & \ESSrank{4}{1} & \ESSrank{4}{1} & \ESSrank{4}{1} & \ESSrank{4}{4} \\
$\geq 350$ & 13 & \ESSrank{2}{3} & \ESSrank{1}{1} & \ESSrank{1}{1} & \ESSrank{3}{1} & \ESSrank{2}{1} & \ESSrank{3}{1} & \ESSrank{2}{1} & \ESSrank{3}{1} & \ESSrank{4}{4} \\
$\geq 500$ & 11 & \ESSrank{2}{1} & \ESSrank{1}{1} & \ESSrank{1}{1} & \ESSrank{4}{1} & \ESSrank{1}{1} & \ESSrank{2}{1} & \ESSrank{2}{1} & \ESSrank{2}{1} & \ESSrank{5}{3} \\
$\geq 850$ & 10 & \ESSrank{2}{1} & \ESSrank{1}{1} & \ESSrank{1}{1} & \ESSrank{4}{1} & \ESSrank{2}{1} & \ESSrank{3}{1} & \ESSrank{3}{1} & \ESSrank{3}{1} & \ESSrank{5}{3} \\
$\geq 950$ & 9 & \ESSrank{2}{1} & \ESSrank{1}{1} & \ESSrank{1}{1} & \ESSrank{4}{1} & \ESSrank{3}{1} & \ESSrank{3}{1} & \ESSrank{3}{1} & \ESSrank{3}{1} & \ESSrank{5}{3} \\
$\geq 1000$ & 7 & \ESSrank{2}{2} & \ESSrank{1}{1} & \ESSrank{1}{1} & \ESSrank{1}{1} & \ESSrank{1}{1} & \ESSrank{1}{1} & \ESSrank{1}{1} & \ESSrank{1}{1} & \ESSrank{4}{3} \\
\midrule
\multicolumn{11}{@{}l}{\textit{Preferences ($M=128$)}} \\*
$\geq 0$ & 19 & \ESSrank{1}{1} & \ESSrank{1}{1} & \ESSrank{1}{1} & \ESSrank{3}{3} & \ESSrank{3}{4} & \ESSrank{4}{4} & \ESSrank{4}{4} & \ESSrank{4}{4} & \ESSrank{3}{3} \\
$\geq 100$ & 13 & \ESSrank{1}{1} & \ESSrank{1}{2} & \ESSrank{1}{1} & \ESSrank{4}{2} & \ESSrank{3}{4} & \ESSrank{4}{1} & \ESSrank{4}{1} & \ESSrank{5}{1} & \ESSrank{4}{3} \\
$\geq 350$ & 12 & \ESSrank{1}{1} & \ESSrank{1}{1} & \ESSrank{1}{1} & \ESSrank{4}{2} & \ESSrank{3}{2} & \ESSrank{3}{2} & \ESSrank{3}{1} & \ESSrank{4}{1} & \ESSrank{4}{3} \\
$\geq 500$ & 11 & \ESSrank{1}{1} & \ESSrank{1}{1} & \ESSrank{1}{1} & \ESSrank{3}{2} & \ESSrank{3}{2} & \ESSrank{3}{2} & \ESSrank{3}{1} & \ESSrank{3}{2} & \ESSrank{3}{3} \\
$\geq 850$ & 9 & \ESSrank{1}{1} & \ESSrank{1}{1} & \ESSrank{1}{1} & \ESSrank{4}{2} & \ESSrank{4}{2} & \ESSrank{4}{2} & \ESSrank{4}{1} & \ESSrank{4}{3} & \ESSrank{4}{3} \\
$\geq 950$ & 8 & \ESSrank{1}{1} & \ESSrank{1}{1} & \ESSrank{1}{1} & \ESSrank{4}{3} & \ESSrank{4}{4} & \ESSrank{5}{2} & \ESSrank{5}{2} & \ESSrank{5}{3} & \ESSrank{4}{3} \\
$\geq 1000$ & 7 & \ESSrank{1}{2} & \ESSrank{1}{1} & \ESSrank{1}{1} & \ESSrank{1}{2} & \ESSrank{1}{3} & \ESSrank{1}{1} & \ESSrank{1}{1} & \ESSrank{1}{3} & \ESSrank{3}{1} \\
\end{longtable}
\endgroup

\looseness=-1
Across all cutoffs, \nameshort\ ranks first on mean energy
and sliced Wasserstein distance. Reward-marginal ranks
depend on the cutoff and aggregation, with the Siberian
husky pair substantially affecting the means.


\subsection{Budget Analysis}
\label{app:budget}

\looseness=-1
We compare $n\in\{32,256\}$ numerical rewards and
$M\in\{8,128\}$ rankings of $m=8$ candidates. Increasing feedback improves held-out loss, PCC, and $\mathrm{KL}_g$ for all three families~(\cref{tab:in-budget-fit}).
Linear compositions remain selected at every budget, although the MLP's loss approaches that of the linear family at the larger budgets. For the selected compositions, $\mathrm{KL}_g$ decreases from $0.466$ to $0.415$ with values and from $0.475$ to $0.423$ with preferences.

\begin{table}[htbp]
\centering
\caption{\small\looseness=-1
\textbf{ImageNet composition fitting across feedback budgets.}
Loss is MSE on $z(R^*)$ for values and PL NLL per group
for preferences.
Bold marks the lowest mean selection loss.
BSS uses the same validation split to select its surrogate
and family; only its composition KL is reported.}
\label{tab:in-budget-fit}
\footnotesize
\setlength{\tabcolsep}{2.5pt}
\renewcommand{\arraystretch}{1.08}
\begin{tabular}{@{}ll*{10}{r}@{}}
\toprule
& & \multicolumn{3}{c}{Held-out loss $\downarrow$}
& \multicolumn{3}{c}{PCC $\uparrow$}
& \multicolumn{4}{c}{$\mathrm{KL}_g$ $\downarrow$} \\
\cmidrule(lr){3-5}\cmidrule(lr){6-8}\cmidrule(l){9-12}
Regime & Budget
& Corr. & Linear & MLP
& Corr. & Linear & MLP
& Corr. & Linear & MLP & BSS \\
\midrule
\multirow{2}{*}{Value}
& $n=32$
& 1.127 & \textbf{1.007} & 1.303
& 0.540 & 0.556 & 0.499
& 0.425 & 0.466 & 0.488 & 0.514 \\
& $n=256$
& 0.898 & \textbf{0.706} & 0.726
& 0.549 & 0.629 & 0.621
& 0.416 & 0.415 & 0.372 & 0.460 \\
\midrule
\multirow{2}{*}{Pref.}
& $M=8$
& 10.774 & \textbf{10.652} & 11.692
& 0.453 & 0.532 & 0.426
& 0.420 & 0.475 & 0.519 & 0.858 \\
& $M=128$
& 10.370 & \textbf{9.908} & 9.999
& 0.530 & 0.625 & 0.612
& 0.417 & 0.423 & 0.426 & 0.440 \\
\bottomrule
\end{tabular}
\end{table}

\looseness=-1
For each steering method and pair, all generated-sample
metrics use the strength with the lowest observed energy
distance to $q^*$.  Improved fitting does not uniformly improve realized
steering~(\cref{tab:in-budget-steer}). \nameshort\ achieves lower energy than BSS at every budget, but energy does not improve with additional feedback. Reward-marginal fidelity improves with more preferences, whereas it worsens at $n=256$: several pairs overshoot the target mean reward at the selected strengths, despite improved composition KL. This separates improvements in fitting from the effects of guidance and sampling.

\definecolor{BudgetMartiniVal}{RGB}{219,234,249}
\definecolor{BudgetMartiniPref}{RGB}{176,206,240}

\begin{table}[htbp]
\centering
\caption{\small\looseness=-1
\textbf{ImageNet steering across feedback budgets.}
Means over the seven evaluation pairs.
\nameshort\ uses the linear composition at every budget.
Lower is better; bold marks the better of each method's
two budgets. Fit-free baselines are shared.}
\label{tab:in-budget-steer}
\footnotesize
\setlength{\tabcolsep}{2.5pt}
\renewcommand{\arraystretch}{1.08}
\begin{tabular}{@{}ll*{9}{r}@{}}
\toprule
& & & \multicolumn{2}{c}{Embedding space}
& \multicolumn{4}{c}{$R^*$-marginal}
& \multicolumn{2}{c}{Reward} \\
\cmidrule(lr){4-5}
\cmidrule(lr){6-9}
\cmidrule(l){10-11}
Method & Budget & $\mathrm{KL}_g$
& Energy & Sl.\,$W_1$
& $\mathrm{KL}_{\mathcal N}$ & $\mathrm{KL}_{\mathrm{KDE}}$
& $W_1^z$ & $W_2^z$ & $|\mathrm{gap}|$ & $Q$ \\
\midrule
\multicolumn{11}{@{}l}{\textit{Fit-free baselines}} \\
Prior $p_{1,\phi}$ & --
& -- & 0.0683 & 0.0080 & 0.752 & 0.658
& 1.253 & 1.380 & 7.51 & 1.000 \\
ParetoFlow & --
& -- & 0.0885 & 0.0088 & 0.825 & 0.937
& 0.948 & 1.068 & 5.35 & 0.844 \\
ReNO & --
& 0.553 & 0.0503 & 0.0069 & 0.197 & 0.166
& 0.380 & 0.430 & 2.10 & 0.244 \\
Uniform & --
& 0.436 & 0.0603 & 0.0074 & 0.605 & 0.635
& 0.475 & 0.552 & 2.73 & 0.238 \\
\midrule
\multicolumn{11}{@{}l}{\textit{Value feedback}} \\
SLOP & $n=32$
& 0.520 & 0.1200 & 0.0098
& \textbf{0.168} & \textbf{0.170}
& \textbf{0.365} & \textbf{0.446} & \textbf{1.98} & 0.364 \\
SLOP & $n=256$
& \textbf{0.423} & \textbf{0.0867} & \textbf{0.0086}
& 0.376 & 0.265 & 0.532 & 0.623 & 2.99 & \textbf{0.250} \\
BSS & $n=32$
& 0.514 & 0.0498 & 0.0068 & 0.217 & 0.183
& 0.478 & 0.547 & 2.57 & 0.354 \\
BSS & $n=256$
& \textbf{0.460} & \textbf{0.0389} & \textbf{0.0063}
& \textbf{0.093} & \textbf{0.108}
& \textbf{0.291} & \textbf{0.348}
& \textbf{1.35} & \textbf{0.285} \\
\rowcolor{BudgetMartiniVal}
\nameshort & $n=32$
& 0.466 & \textbf{0.0321} & \textbf{0.0058}
& \textbf{0.034} & \textbf{0.049}
& \textbf{0.255} & \textbf{0.310} & \textbf{1.27} & 0.359 \\
\rowcolor{BudgetMartiniVal}
\nameshort & $n=256$
& \textbf{0.415} & 0.0377 & 0.0060 & 0.549 & 0.764
& 0.454 & 0.527 & 2.72 & \textbf{0.279} \\
\midrule
\multicolumn{11}{@{}l}{\textit{Preference feedback}} \\
BSS & $M=8$
& 0.858 & \textbf{0.0412} & \textbf{0.0063}
& 0.323 & 0.262 & 0.570 & 0.637 & 2.82 & 0.585 \\
BSS & $M=128$
& \textbf{0.440} & 0.0423 & 0.0065
& \textbf{0.135} & \textbf{0.133}
& \textbf{0.392} & \textbf{0.442}
& \textbf{2.04} & \textbf{0.316} \\
\rowcolor{BudgetMartiniPref}
\nameshort & $M=8$
& 0.475 & \textbf{0.0337} & \textbf{0.0059}
& 0.544 & 0.629 & 0.606 & 0.679 & 3.49 & 0.386 \\
\rowcolor{BudgetMartiniPref}
\nameshort & $M=128$
& \textbf{0.423} & 0.0372 & 0.0061
& \textbf{0.093} & \textbf{0.121}
& \textbf{0.338} & \textbf{0.395}
& \textbf{1.71} & \textbf{0.297} \\
\bottomrule
\end{tabular}
\end{table}

\clearpage
\section{Protein--ligand cofolding}
\label{app:protein}

\subsection{Surrogates}\label{app:protein-surrogates}
We fit the surrogate composition to experimental binding-affinity
feedback and evaluate its transfer to generated structures.
The five surrogates provide information about structural
quality, binding-site energetics, and predicted affinity.

\textit{SMINA}. SMINA is a widely used docking and scoring tool forked from AutoDock~Vina and optimised for high-throughput scoring and for user-specified scoring functions~\citep{koes2013smina}.

\textit{Frustration}. The FrustratometeR R package~\cite{rausch2021frustratometer} is used to calculate local energetic frustration in protein structures, where high frustration relates to unstable local regions and correlates with active sites.

\textit{MolProbity}. MolProbity ~\citep{chen2010molprobity} assesses the physical quality of molecular structures through all-atom steric contact analysis with explicit hydrogens, together with covalent-geometry and backbone and side-chain torsion-angle criteria.

\textit{Boltz Affinity}. The boltz affinity prediction head introduced in~\cite{passaro2025boltz2} is used here as an affinity surrogate. The Boltz-2 weights are used to featurise the complex, but the pose itself is taken as atomic coordinates from the steering run (e.g. with DeCAF) and then used to constrain the affinity prediction, so that the prediction head uses the steered poses. Default Boltz-2 settings were used for this surrogate.

\textit{Boltz-1x physical potentials.} We include the physical potential terms used by Boltz-1x~\citep{wohlwend2025boltz} to penalize stereochemical violations, clashes, and inconsistent chain geometry. The resampling coefficients are
\[
\begin{array}{lc}
\texttt{symmetric\_chain\_com} & 0.5\\
\texttt{connections} & 1.0\\
\texttt{posebusters} & 0.1\\
\texttt{chiral} & 1.0\\
\texttt{stereo} & 1.0\\
\texttt{planar} & 1.0\\
\texttt{vdw\_overlap} &
\begin{cases}0.01,&t<0.6,\\0,&t\ge0.6.\end{cases}
\end{array}
\]
Here $t$ is the resampling-time variable used by the
implementation. These seven terms include both within-chain
and between-chain contributions.

We normalize the surrogate reward values both for fitting $g$ and during steering. Figure \ref{fig:prot_ceil_bars_hiqbind} displays the correlation of the surrogate rewards with the experimental HiQBind $\mathrm{p}K_D$ as well as the theoretical recovery ceiling when using these surrogate rewards. SMINA is the best single surrogate based on the Pearson correlation. Boltz energy potentials are instrumental in ensuring physical validity of the generated poses; therefore, each baseline includes normalized energy potentials. 

\begin{figure}
    \centering
    \includegraphics[width=0.95\linewidth]{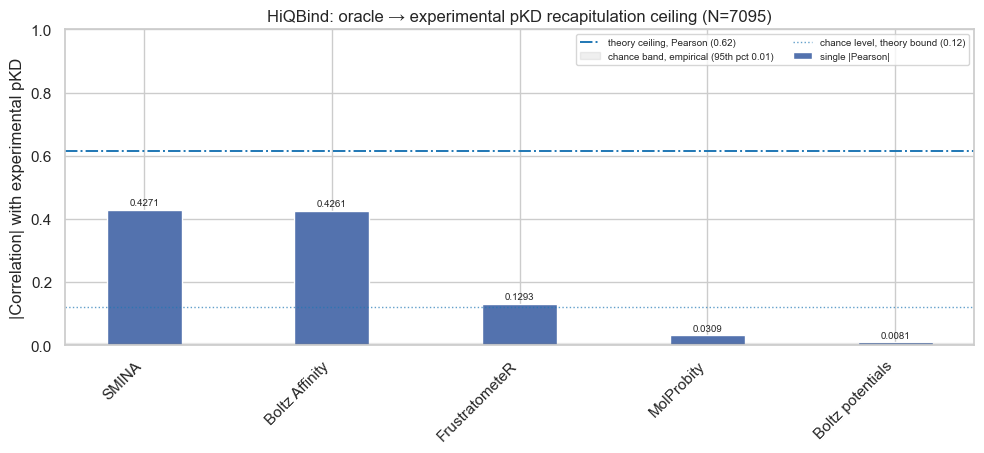}
    \caption{Correlation of HiQBind scores with the experimentally measured binding affinity values}
    \label{fig:prot_ceil_bars_hiqbind}
\end{figure}

The theory ceiling in~\cref{fig:prot_ceil_bars_hiqbind} is $\rho_{\mathrm{info}}=\sqrt{1-\bar v/\Var(R^*)}$, the highest PCC any function of the scores can reach, where $\bar v=\E[\Var(R^*\mid\mathbf R)]$ is the variance of the measured affinity that the HiQBind scores leave unexplained, estimated by leave-one-out local-linear $k$-nearest-neighbour regression of $R^*$ on $\mathbf R$.

\subsection{Feedback and evaluation datasets}
\label{app:protein-datasets}

\textit{HiQBind}. HiQBind is a curated protein--ligand binding dataset constructed by matching experimentally determined binding free energies from BioLiP, Binding~MOAD, and BindingDB to co-crystallized PDB complexes and reprocessing each entry through the open-source HiQBind-WF workflow~\citep{wang2025hiqbind}. The release comprises 18,160 PDB entries with binding-affinity annotations and 32,275 refined non-covalent complexes, annotated with affinity values, release year, resolution, ligand and protein identifiers, and derived ligand properties. We subset this data set to monomers and multimers that have experimental $\mathrm{p}K_D$, leaving us with N=7095 for the feedback set.

\textit{Runs N' Poses.} We evaluate on $186$ complexes
from the post-2023 subset of Runs N' Poses~\citep{skrinjar2026runsnposes},
beyond Boltz-1's training cutoff. These evaluation complexes
lack experimental affinity labels, so we measure pose
accuracy and physical validity. This tests whether a
composition fitted to affinity feedback transfers to
structure generation; it does not directly measure
recovery of the experimental affinity tilt.
\Cref{fig:prot_dist_zscore_oracle} compares the surrogate
marginals in the fitting and evaluation datasets.

\begin{figure}[h!]
    \centering
    \includegraphics[width=\linewidth]{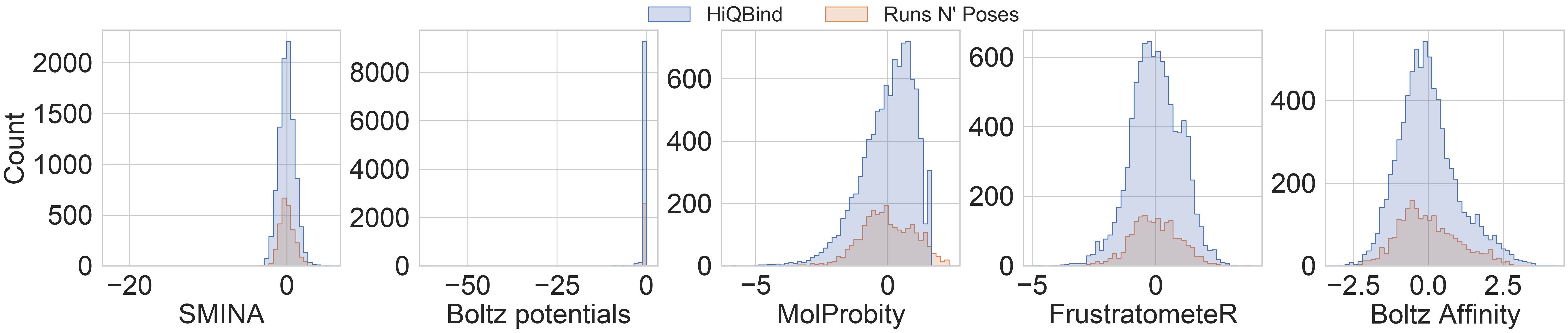}
    \vspace{-10pt}
    \caption{\small\looseness=-1 \textbf{Surrogate-score distributions.} Normalized marginals on HiQBind and Runs N' Poses.}
    \label{fig:prot_dist_zscore_oracle}
\end{figure}

\subsection{Protein success metrics}\label{app:protein-metrics}

We use the OpenStructure (\texttt{ost}) package for calculating both the RMSD and \textit{lDDT-PLI}~\cite{biasini2013openstructure}, with the same defaults as in~\cite{scarpellini2026fewstepcofoldingallatomflow}.

\looseness=-1
\textit{Diversity}. We report structural diversity, defined as one minus the similarity, where similarity is the mean pairwise TM-score~\cite{zhang2004tmscore} across the top $n_{seed}$ best poses (best within one seed across particles) where $n_{seed}$ is the number of seeds. 

We report best@1 metrics as defined in DeCAF~\cite{scarpellini2026fewstepcofoldingallatomflow} averaged over 5 seeds.


\subsection{Extended results and ablations}
\label{app:protein-results}

We compare FK steering with $10$ and $15$ particles
in~\cref{tab:prot-abl-num-particles} and vary the relative
weights of smina and the energy potentials
in~\cref{tab:prot-abl-smina-weights} to choose the strongest baseline combination. We report joint
success, physical validity, and lDDT-PLI to assess
sensitivity to these choices.

Ablating between the mixtures of fixed-weights shows that  $[0.1,0.9]$ gives the highest success rate ($61.8\%$), while $[0.9,0.1]$ gives the highest lDDT-PLI and lower PB-validity. We pick the mixture with the highest success rate as the BSS baseline reported in other tables, and use this mixture for BSS baseline experiment with 15 particles.
The learned linear mixture reaches $62.4\%$ success, which beats all of the ablated mixtures, demonstrating utility of \nameshort even with fewer surrogates.

Furthermore, we compare success rates (RMSD $\le2$\,\AA{} and PB validity) when selecting generated poses based on the surrogate reward in Table \ref{tab:prot_steering_appendix} (rewards unscaled). MARTINI-generated poses chosen based on the highest reward result in the highest success compared to baselines, with 66.8\% for 5-surrogate linear and 66.3\% for 3-surrogate linear, where the surrogates are [SMINA, GNINA, Frustration, MolProbity and Boltz Affinity] and [Frustration, MolProbity and Boltz Affinity], respectively. With un-normalized rewards, MARTINI with 3 surrogates achieves the highest success rate when choosing the best particle across seeds at 73\% compared to the BSS (Boltz Affinity) baseline's 72.4\%. The RMSD, PB valid, and lDDT-PLI metrics are evaluated for the pose that was selected based on the highest PB validity, then lowest RMSD, then highest lDDT-PLI. Diversity here is reported in terms of the effective number of poses remaining after steering.

 \begin{table}[htbp]
  \centering
  \caption{\textbf{Steering with smina and energy potentials, weight ablation.} 10 particles. Best@1 averaged over 5 seeds. \textbf{Bold} and \underline{underline} mark the best and runner-up value per column.}
  \label{tab:prot-abl-smina-weights}
\definecolor{martinirow}{RGB}{219,234,249}   
\footnotesize
\setlength{\tabcolsep}{4pt}
\begin{tabular}{lrrrr}
\toprule
$[w_{\text{smina}}, w_{\text{potentials}}]$ & Success Rate $\uparrow$ & RMSD $<2$\,\AA{} $\uparrow$ & PB Valid $\uparrow$ & lDDT-PLI $\uparrow$ \\
\midrule
$[0.1, 0.9]$ & $\underline{61.8}$ & $\underline{61.8}$ & $\underline{97.6}$ & ${65.1}$ \\
$[0.3, 0.7]$ & ${61.2}$ & ${61.2}$ & $\mathbf{97.8}$ & ${65.1}$ \\
$[0.5, 0.5]$ & ${61.1}$ & ${61.1}$ & ${97.4}$ & ${65.0}$ \\
$[0.7, 0.3]$ & ${61.3}$ & ${61.3}$ & ${96.5}$ & ${65.1}$ \\
$[0.9, 0.1]$ & ${61.5}$ & ${61.6}$ & ${94.4}$ & $\underline{65.4}$ \\
\midrule
\rowcolor{martinirow}$g$(linear) & $\mathbf{62.4}$ & $\mathbf{62.7}$ & ${93.5}$ & $\mathbf{66.3}$ \\
\bottomrule
\end{tabular}
\end{table}

 \begin{table}[htbp]
  \centering
  \caption{\textbf{Number of particles ablation, FK-Steering.} All surrogate rewards are normalized. Best@1 averaged over 5 seeds. \textbf{Bold} and \underline{underlined} mark the best and runner-up value per column.}
  \label{tab:prot-abl-num-particles}
\definecolor{martinirow}{RGB}{219,234,249}   
\footnotesize
\setlength{\tabcolsep}{4pt}
\begin{tabular}{lrrrrrr}
\toprule
& \multicolumn{3}{c}{Num Particles $=10$} & \multicolumn{3}{c}{Num Particles $=15$} \\
\cmidrule(lr){2-4} \cmidrule(lr){5-7}
Method & Success Rate $\uparrow$ & PB Valid $\uparrow$ & lDDT-PLI $\uparrow$ & Success Rate $\uparrow$ & PB Valid $\uparrow$ & lDDT-PLI $\uparrow$ \\
\midrule
\multicolumn{7}{l}{\textbf{Baselines}} \\
\midrule
DeCAF potentials & ${61.8}$ & $\mathbf{97.9}$ & ${65.2}$ & ${61.8}$ & $\mathbf{98.3}$ & ${65.7}$ \\
smina + potentials & ${61.8}$ & $\underline{97.6}$ & ${65.1}$ & ${61.8}$ & $\underline{98.2}$ & ${65.4}$ \\
\midrule
\multicolumn{7}{l}{\textbf{\nameshort}} \\
\midrule
\rowcolor{martinirow}$g$(lin: smina, poten.) & ${62.4}$ & ${93.5}$ & $\underline{66.3}$ & ${62.7}$ & ${93.3}$ & ${66.5}$ \\
\rowcolor{martinirow}$g$(corr: all 5) & $\mathbf{63.3}$ & ${91.9}$ & $\mathbf{66.4}$ & $\mathbf{64.5}$ & ${91.8}$ & $\mathbf{67.6}$ \\
\rowcolor{martinirow}$g$(lin: all 5) & $\underline{62.8}$ & ${95.1}$ & ${65.6}$ & ${63.4}$ & ${94.7}$ & ${66.5}$ \\
\rowcolor{martinirow}$g$(mlp: all 5) & ${62.2}$ & ${83.0}$ & ${65.5}$ & $\underline{63.7}$ & ${82.7}$ & $\underline{67.3}$ \\
\bottomrule
\end{tabular}
    
\end{table}

\begin{table*}
    \centering
    \caption{\small \looseness=-1 \textbf{Reward-selected poses vs. actual best poses, un-normalized rewards}. Steering on pose-2023 Runs N' Poses, 15 particles, 5 seeds, rewards un-normalised, best@5. \textbf{Bold} and \underline{underlined} mark the best and runner-up value per column.}
        \vspace{-10pt}
    \label{tab:prot_steering_appendix}
    
\definecolor{martinirow}{RGB}{219,234,249}   
\definecolor{refrow}{gray}{0.92}   

\centering
\setlength{\tabcolsep}{4.5pt}
\resizebox{\ifdim\width>\linewidth\linewidth\else\width\fi}{!}{
\begin{tabular}{lrrrrrr}
    \toprule
    & \multicolumn{1}{c}{Top reward} & \multicolumn{4}{c}{Best particle$^{\dagger}$} & \multicolumn{1}{c}{Diversity} \\
    \cmidrule(lr){2-2} \cmidrule(lr){3-6} \cmidrule(l){7-7}
    Method & Success $\uparrow$ & Success $\uparrow$ & RMSD$\le2$\,\AA{} $\uparrow$ & PB Valid $\uparrow$ & lDDT-PLI $\uparrow$ & Eff.~poses $\uparrow$ \\
    \midrule
    \rowcolor{refrow}Boltz-1x & $66.7$ & $67.3$ & $67.3$ & $100.0$ & $69.9$ & $2.9$ \\
    \midrule
    DeCAF & $65.3$ & $66.3$ & $66.3$ & $\mathbf{100.0}$ & $69.2$ & $2.6$ \\
    DeCAF potentials (FK, $K{=}15$) & $65.8$ & $69.4$ & $69.9$ & $99.5$ & $70.7$ & $6.2$ \\
    \midrule
    DeCAF BSS (affinity) & $66.1$ & $\underline{72.4}$ & $\underline{72.4}$ & $\mathbf{100.0}$ & $\mathbf{72.9}$ & $\underline{7.9}$ \\
    \rowcolor{martinirow}DeCAF-\nameshort(3 surrogates) & $\underline{66.3}$ & $\mathbf{73.0}$ & $\mathbf{73.0}$ & $\mathbf{100.0}$ & $\mathbf{72.9}$ & $\mathbf{8.0}$ \\
    \rowcolor{martinirow}DeCAF-\nameshort(5 surrogates) & $\mathbf{66.8}$ & $69.9$ & $69.9$ & $\mathbf{100.0}$ & $71.1$ & $6.0$ \\
    \bottomrule
\end{tabular}
}
\end{table*}


\clearpage
\section{Additional theoretical results}
\label{app:extra_theory}

This section gives the assumptions and supporting results
for the two main theorems in~\S\ref{sec:theory}.
Proofs appear in~\S\ref{app:proofs}.

\subsection{Assumptions}
\label{app:assumptions}

We analyze population quantities and ideal tilted distributions,
with $p_{1,\phi}=p_1$. The space $\mathcal X$ carries its Borel
$\sigma$-algebra, $p_1$ is a probability law, and the oracle,
surrogates, and compositions are measurable. Conditional
statements hold almost surely under the specified law.
Compositions in the main theorems are fixed as $\beta\downarrow0$;
the explicitly defined $\gopt$ in
Proposition~\ref{prop:klopt} may depend on $\beta$.

\begin{mdframed}[style=MyFrame2]
\begin{assumption}[Finite second moments]
\label{ass:moments}
The oracle and every compared score satisfy
\[
\E_{p_1}[(R^*)^2]<\infty,
\qquad
\E_{p_1}[(g\circ\mathbf R)^2]<\infty.
\]
\end{assumption}
\end{mdframed}

\begin{mdframed}[style=MyFrame2]
\begin{assumption}[Essential boundedness]
\label{ass:bounded}
There exist constants $M_*,M_g<\infty$, independent of $\beta$,
such that
\[
p_1(|R^*|\le M_*)=1,
\qquad
p_1(|g\circ\mathbf R|\le M_g)=1.
\]
The score bound is required for each composition used in an
expansion; $M_g$ need not be uniform over all compositions.
\end{assumption}
\end{mdframed}

Assumption~\ref{ass:bounded} implies
Assumption~\ref{ass:moments} and ensures finite, positive tilt
normalizers. It concerns the oracle and composed scores
$p_1$-almost surely, rather than bounds on every raw surrogate
globally on $\mathcal X$.

\begin{mdframed}[style=MyFrame2]
\begin{assumption}[Mean-aligned representatives]
\label{ass:means}
Every compared composition is represented so that
\[
\E_{p_1}[g\circ\mathbf R]=\E_{p_1}[R^*].
\]
In the preference analysis, this holds for every $g\in\gG$,
including the population minimizer $g_{\PL}$.
\end{assumption}
\end{mdframed}

This is a normalization convention for analysis. For an unshifted
composition $g$, define
\[
\bar g(\mathbf r):=g(\mathbf r)
+\E_{p_1}[R^*]-\E_{p_1}[g\circ\mathbf R].
\]
Then $q_{\bar g}=q_g$ and $\PL(\bar g)=\PL(g)$, whereas
\[
\mathrm{MSE}(\bar g)=\mathrm{MSE}(g)
-\bigl(
\E_{p_1}[g\circ\mathbf R]-\E_{p_1}[R^*]
\bigr)^2.
\]
Thus mean alignment preserves tilts and PL probabilities, but
can reduce MSE. In the preference analysis, $\gG$ denotes the
mean-aligned image of the original class; this affine
transformation preserves convexity, score-range widths, PL
risks, and induced tilts. The theoretical bounds use these
representatives and require no additional fitting step or
oracle access. The aligned MSE is no larger than the unshifted
population MSE estimated by~\eqref{eq:mse-loss}.

\begin{mdframed}[style=MyFrame2]
\begin{assumption}[Preference observations and composition class]
\label{ass:preferences}
\begin{enumerate}[
    label=(\roman*),
    leftmargin=2em,
    itemsep=3pt,
    topsep=1pt
]
\item Each query contains $m\ge2$ independent candidates
$S\sim p_1^{\otimes m}$, with ranking
\[
\pi\mid S\sim p^*_{\PL}(\cdot\mid S)
:=p_{\PL}(\cdot\mid R^*,S).
\]
Observed queries are independent.

\item The class $\gG$ is nonempty and convex, and there is a
common $\Delta<\infty$ such that, for every $g\in\gG$,
\[
\operatorname*{ess\,sup}_{p_1}(g\circ\mathbf R)
-\operatorname*{ess\,inf}_{p_1}(g\circ\mathbf R)
\le\Delta.
\]

\item The population loss
\[
\PL(g):=\E_{S,\pi}
[-\log p_{\PL}(\pi\mid g\circ\mathbf R,S)]
\]
attains its minimum over $\gG$ at some
$g_{\PL}\in\argmin_{g\in\gG}\PL(g)$.
\end{enumerate}
\end{assumption}
\end{mdframed}

Existence of $g_{\PL}$ is assumed; convexity and a range bound
alone do not establish it. Together with
Assumption~\ref{ass:means}, the common width implies
\[
|g\circ\mathbf R-\E_{p_1}[R^*]|\le\Delta
\quad p_1\text{-almost surely},
\qquad g\in\gG.
\]
Hence the preference scores are already bounded. Boundedness
of $R^*$ is additionally imposed for the KL and reward
expansions, but is not needed for the exact recovery bound
when $R^*\in L^2(p_1)$.

\paragraph{Relation to fitted models.}
The common range bound applies to the entire admissible score class, not merely to observed scores. For example, if $\|\mathbf R(x)\|_2\le B$ $p_1$-almost surely, the convex linear class $\{g_{\mathbf w}:\|\mathbf w\|_2\le W\}$ has score-range width at most $2BW$. Ridge regularization and finite empirical ranges alone do not establish this condition for an unrestricted class. Theorem~\ref{thm:pref} remains conditional on the stated class assumptions.

\paragraph{Which assumptions each conclusion uses.}
The following conditions suffice for the individual conclusions;
the main statements collect them for convenience.

\begin{center}
\small
\begin{tabular}{@{}p{0.36\linewidth}p{0.59\linewidth}@{}}
\toprule
Conclusion & Assumptions used \\
\midrule
Theorem~\ref{thm:value}(i)
& \ref{ass:bounded}; also \ref{ass:means} for the equality
with MSE. \\

Theorem~\ref{thm:value}(ii), exact identity
& \ref{ass:moments} for $R^*$ and $R\in L^\infty(p_1)$;
no small-$\beta$ limit. \\

Theorem~\ref{thm:value}(ii)--(iii), asymptotic bounds
& \ref{ass:bounded}, with $g$ fixed. \\

Theorem~\ref{thm:value}(iv)
& \ref{ass:moments} only. \\

Theorem~\ref{thm:pref}, $\mathrm{MSE}(g)\le B(g)$ and (iv)
& \ref{ass:moments}, \ref{ass:means}, and
\ref{ass:preferences}. \\

Theorem~\ref{thm:pref}(i)--(iii)
& The preceding conditions and the oracle bound in
\ref{ass:bounded}, with $g$ fixed. \\
\bottomrule
\end{tabular}
\end{center}

Mean alignment is unnecessary for the value gain and gap bounds:
$\Var_{p_1}(g\circ\mathbf R-\mu)
\le\mathrm{MSE}(g)-\bar v$ suffices.
The exact recovery identities and the preference recovery bound
hold without taking $\beta\downarrow0$. For fixed-score expansions,
the $O(\beta^j)$ constants may depend on the problem and score
and need not be uniform over the fitted class. The expansions
involving $\gopt$ account explicitly for its dependence on
$\beta$.

\begin{mdframed}[style=MyFrame2]
\begin{lemma}[Likelihood controls recovery]
\label{lem:pl-sc}
Under Assumptions~\ref{ass:moments}, \ref{ass:means},
and~\ref{ass:preferences}, for every $g\in\gG$, with
$c_\Delta=(1-\tfrac1m)e^{-\Delta}$,
\[
\Var_{p_1}(g\circ\mathbf R-g_{\PL}\circ\mathbf R)
\le
\frac{2}{c_\Delta}
\bigl(\PL(g)-\PL(g_{\PL})\bigr).
\]
The variance inequality is invariant to additive shifts.
\end{lemma}
\end{mdframed}

\subsection{Theoretical floor}

The exact recovery floor is $\bar v$. The exact KL floor is
$\mathcal K_{\min}$, whose leading term is $\beta^2\bar v/2$.
For a law $q$ on $\mathcal X$, write $q(\mathbf R)$ for
the distribution of $\mathbf R(X)$ when $X\sim q$.

\begin{mdframed}[style=MyFrame2]
\begin{lemma}[Level-set invariance]
\label{lem:invariance}
For every measurable composition $g$ with
$0<\E_{p_1}[e^{\beta g\circ\mathbf R}]<\infty$,
\[
q_g(\cdot\mid\mathbf R=\mathbf r)
=p_1(\cdot\mid\mathbf R=\mathbf r)
\quad\text{for }p_1(\mathbf R)\text{-almost every }\mathbf r.
\]
The tilt reweights surrogate level sets without changing their
conditional laws. This statement also applies when individual
level sets have zero mass.
\end{lemma}
\end{mdframed}

\begin{mdframed}[style=MyFrame2]
\begin{proposition}[Exact $\KL$ floor and optimal composition]
\label{prop:klopt}
Fix $\beta>0$. Under Assumption~\ref{ass:bounded}, for every
compared composition $g$,
\[
\KL(q^*\|q_g)
=\mathcal K_{\min}
+\KL(q^*(\mathbf R)\|q_g(\mathbf R))
\ge\mathcal K_{\min},
\]
where
\[
\mathcal K_{\min}
:=
\E_{\mathbf r\sim q^*(\mathbf R)}
\left[
\KL\bigl(
q^*(\cdot\mid\mathbf r)\|p_1(\cdot\mid\mathbf r)
\bigr)
\right]
\]
does not depend on $g$. Over all admissible measurable
compositions, the minimum is attained at
\[
\gopt(\mathbf r)
=\frac1\beta
\log\E_{p_1}
[e^{\beta R^*}\mid\mathbf R=\mathbf r],
\]
uniquely up to an additive constant,
$p_1(\mathbf R)$-almost surely; it need not belong to a
restricted fitted family.

Moreover, $\mathcal K_{\min}=0$ if and only if $\bar v=0$.
As $\beta\downarrow0$, with $g$ fixed in the second expansion,
\[
\mathcal K_{\min}
=\frac{\beta^2}{2}\bar v+O(\beta^3),
\]
\[
\KL(q^*\|q_g)
=\mathcal K_{\min}
+\frac{\beta^2}{2}\Var_{p_1}(g\circ\mathbf R-\mu)
+O(\beta^3).
\]
\end{proposition}
\end{mdframed}

\begin{mdframed}[style=MyFrame2]
\begin{corollary}[The regression solution is leading-order optimal]
\label{cor:mu}
Under Assumption~\ref{ass:bounded}, for the displayed $\gopt$
in Proposition~\ref{prop:klopt},
\[
\gopt\circ\mathbf R
=\mu+\tfrac{\beta}{2}v+O(\beta^2)
\]
uniformly almost surely. As $\beta\downarrow0$,
\[
\KL(q^*\|q_\mu)-\mathcal K_{\min}
=\frac{\beta^4}{8}\Var_{p_1}(v)+O(\beta^5).
\]
\end{corollary}
\end{mdframed}

Here $\Var_{p_1}(v)$ measures variation in conditional
oracle variance across surrogate values. Thus the conditional
mean induces a tilt whose excess KL above the exact optimum
is fourth order in $\beta$.

\subsection{Reward error controls steering}

\begin{mdframed}[style=MyFrame2]
\begin{lemma}[Reward error controls steering]
\label{lem:err2kl}
Under Assumption~\ref{ass:bounded}, with $g$ fixed,
\[
\KL(q^*\|q_g)
=\frac{\beta^2}{2}
\Var_{p_1}(R^*-g\circ\mathbf R)+O(\beta^3)
\quad(\beta\downarrow0).
\]
No mean or variance matching is required for this expansion.
\end{lemma}
\end{mdframed}

The orthogonal decomposition
$R^*-g\circ\mathbf R=(R^*-\mu)+(\mu-g\circ\mathbf R)$
gives the recovery floor plus approximation error in
Theorem~\ref{thm:value}.

\begin{mdframed}[style=MyFrame2]
\begin{corollary}[Positive gain from imperfect recovery]
\label{cor:positive-gain}
Under the assumptions of Theorem~\ref{thm:value}, suppose
\[
\mathrm{MSE}(g)<\Var_{p_1}(R^*),
\]
or equivalently,
$\|g\circ\mathbf R-\mu\|_{L^2(p_1)}^2
<\Var_{p_1}(\mu)$.
Then there exists $\beta_0>0$ such that
\[
\E_{q_g}[R^*]>\E_{p_1}[R^*]
\qquad\text{for all }0<\beta<\beta_0.
\]
Thus, predicting the oracle more accurately than the best
constant predictor suffices for positive reward gain at
sufficiently small guidance strength.
\end{corollary}
\end{mdframed}

\subsection{Correlation ceiling}

\begin{mdframed}[style=MyFrame2]
\begin{corollary}[Correlation ceiling]
\label{cor:corr}
Under Assumption~\ref{ass:moments}, if $\Var_{p_1}(R^*)>0$,
every nonconstant score $g\circ\mathbf R$ satisfies
\[
\left|\PCC_{p_1}(g\circ\mathbf R,R^*)\right|
\le\sqrt{1-\bar v/\Var_{p_1}(R^*)}.
\]
If $\Var_{p_1}(\mu)>0$, equality is attained at
$g\circ\mathbf R=\mu$. If $\Var_{p_1}(\mu)=0$, the correlations
are zero and PCC with the constant $\mu$ is undefined.
\end{corollary}
\end{mdframed}

The identity
$\bar v=\Var_{p_1}(R^*)(1-\rho_{\mathrm{ceil}}^2)$
involves the population ceiling. An arbitrary predictor's
held-out correlation need not attain it or determine $\bar v$.
A fitted linear predictor estimates performance within the linear
class and need not attain this unrestricted ceiling.
The best nonlinear function of a single surrogate $R_j$
instead has Pearson ceiling
\[
\sqrt{
\frac{
\Var_{p_1}(\E_{p_1}[R^*\mid R_j])
}{
\Var_{p_1}(R^*)
}
}.
\]

\begin{mdframed}[style=MyFrame2]
\begin{corollary}[Standardized surrogate tilts]
\label{cor:standardized}
Under Assumptions~\ref{ass:bounded} and~\ref{ass:means},
let $f=g\circ\mathbf R$, with
$s_*^2=\Var_{p_1}(R^*)>0$ and $s_f^2=\Var_{p_1}(f)>0$.
For a nonconstant score $F$, define
\[
z_F:=\frac{F-\E_{p_1}[F]}{\sqrt{\Var_{p_1}(F)}},
\qquad
q_F^{\mathrm{std}}\propto p_1e^{\beta z_F}.
\]
For fixed $g$, as $\beta\downarrow0$,
\[
\KL(q_{R^*}^{\mathrm{std}}\|q_f^{\mathrm{std}})
=\beta^2\bigl[1-\PCC_{p_1}(R^*,f)\bigr]+O(\beta^3).
\]
If $s_\mu^2:=\Var_{p_1}(\mu)>0$, the leading coefficient
is minimized over bounded, nonconstant surrogate scores
by $f=\mu$. Moreover, with $\rho_{\mathrm{ceil}}=s_\mu/s_*$,
\[
0\le \rho_{\mathrm{ceil}}-\PCC_{p_1}(R^*,f)
\le \frac{2\bigl(\mathrm{MSE}(g)-\bar v\bigr)}{s_*s_\mu}.
\]
\end{corollary}
\end{mdframed}

\subsection{Surrogates and floors}

\begin{mdframed}[style=MyFrame2]
\begin{corollary}[More surrogates can improve the optimal floors]
\label{cor:mono}
Let $\mathbf R_S$ be any sub-vector of $\mathbf R$, and let
$\bar v_S$ and $\mathcal K_{\min}(S)$ be its optimal recovery
and KL floors. Then
\[
\bar v\le\bar v_S,
\qquad
\mathcal K_{\min}\le\mathcal K_{\min}(S).
\]
The recovery comparison uses Assumption~\ref{ass:moments};
the KL comparison holds for every $\beta>0$ under
Assumption~\ref{ass:bounded}. If $\Var_{p_1}(R^*)>0$, the
Pearson ceiling weakly increases. Under
Assumption~\ref{ass:bounded}, the conditional-mean leading
reward-gain coefficient $\Var_{p_1}(R^*)-\bar v$ weakly
increases, and its leading mean-gap coefficient $\bar v$
weakly decreases.
\end{corollary}
\end{mdframed}

The unrestricted best composition of the full vector can
match or improve on that of any sub-vector. Redundant
surrogates need not give strict improvement. These comparisons
do not guarantee that a finite-data fit improves, or impose
an unrestricted reward ceiling.

\subsection{Plackett--Luce curvature}
The following derivatives give the curvature used in the
preference recovery bound.

\begin{mdframed}[style=MyFrame2]
\begin{proposition}[Plackett--Luce curvature]
\label{prop:plhess}
For a query with $m\ge2$ candidates, finite score vector
$f\in\mathbb R^m$, and observed order $\pi$, let
$A_i=\{\pi(i),\ldots,\pi(m)\}$ and let $p^{(i)}$ be
the softmax on $A_i$, zero elsewhere. The query NLL satisfies
\[
\nabla_f\mathcal L_S
=\sum_{i=1}^{m-1}(p^{(i)}-e_{\pi(i)}),
\qquad
\nabla_f^2\mathcal L_S
=\sum_{i=1}^{m-1}
\left[\diag(p^{(i)})-p^{(i)}p^{(i)\top}\right]\succeq0.
\]
Its Hessian has null space $\mathrm{span}\{\mathbf1\}$
and rank $m-1$, where $e_j$ is the $j$th standard basis vector.
\end{proposition}
\end{mdframed}

\clearpage
\section{Proofs}
\label{app:proofs}

Proofs are ordered by dependency. All expectations below are
under the stated base or tilted laws, and all assertions
about conditional moments hold almost surely.

\paragraph{Lemma~\ref{lem:invariance} (level-set invariance).}
For a measurable event $A$, conditional change of measure gives
\[
\E_{q_g}[\mathbf1_A\mid\mathbf R]
=
\frac{
\E_{p_1}[\omega_{g\circ\mathbf R}\mathbf1_A\mid\mathbf R]
}{
\E_{p_1}[\omega_{g\circ\mathbf R}\mid\mathbf R]
}
=\E_{p_1}[\mathbf1_A\mid\mathbf R].
\]
The density ratio is positive and $\mathbf R$-measurable, so
it cancels. The base and tilted surrogate marginals have the
same null sets. No division by the probability of a level
set is used.\hfill$\square$

\paragraph{Lemma~\ref{lem:err2kl} (reward error controls steering).}
Let
$Z_g=\E_{p_1}[e^{\beta g\circ\mathbf R}]$
and $Z^*=\E_{p_1}[e^{\beta R^*}]$.
The density ratio gives
\[
\KL(q^*\|q_g)
=\beta\E_{q^*}[R^*-g\circ\mathbf R]
+\log Z_g-\log Z^*.
\]
Assumption~\ref{ass:bounded} justifies the Taylor expansions:
\begin{align*}
\log Z_g
&=\beta\E_{p_1}[g\circ\mathbf R]
+\tfrac{\beta^2}{2}\Var_{p_1}(g\circ\mathbf R)
+O(\beta^3),\\
\log Z^*
&=\beta\E_{p_1}[R^*]
+\tfrac{\beta^2}{2}\Var_{p_1}(R^*)
+O(\beta^3),\\
\E_{q^*}[R^*-g\circ\mathbf R]
&=\E_{p_1}[R^*-g\circ\mathbf R]
+\beta\Cov_{p_1}(R^*,R^*-g\circ\mathbf R)
+O(\beta^2).
\end{align*}
The linear terms cancel, leaving
\[
\tfrac{\beta^2}{2}
\left[
\Var_{p_1}(R^*)
-2\Cov_{p_1}(R^*,g\circ\mathbf R)
+\Var_{p_1}(g\circ\mathbf R)
\right]+O(\beta^3),
\]
which is the claimed variance expansion.\hfill$\square$

\begin{mdframed}[style=MyFrame2]
\thmvalue*
\end{mdframed}

\begin{proof}[Proof of Theorem~\ref{thm:value}]
\emph{(iv), using Assumption~\ref{ass:moments}.}
Write $\varepsilon=R^*-\mu$. Then
$\E_{p_1}[\varepsilon\mid\mathbf R]=0$ and
$\E_{p_1}[\varepsilon^2]=\bar v$.
Thus $\varepsilon$ is orthogonal in $L^2(p_1)$ to every
square-integrable function of $\mathbf R$.
Expanding the square gives
\[
\mathrm{MSE}(g)
=\E_{p_1}
[(\varepsilon+\mu-g\circ\mathbf R)^2]
=\bar v+\|g\circ\mathbf R-\mu\|_{L^2(p_1)}^2.
\]
Equality with $\bar v$ holds exactly when
$g\circ\mathbf R=\mu$. Taking the constant score
$\E_{p_1}[R^*]$ proves
$\Var_{p_1}(R^*)=\Var_{p_1}(\mu)+\bar v$.

\emph{(i), using Assumptions~\ref{ass:bounded}
and~\ref{ass:means}.}
The same orthogonality yields
\[
\Var_{p_1}(R^*-g\circ\mathbf R)
=\bar v+\Var_{p_1}(g\circ\mathbf R-\mu).
\]
By mean alignment the error has mean zero, so this equals
$\mathrm{MSE}(g)$. Lemma~\ref{lem:err2kl} proves (i).

\emph{(ii).}
For $R\in L^\infty(p_1)$ and $R^*\in L^2(p_1)$, the density
ratio satisfies $\E_{p_1}[\omega_R]=1$, giving the exact identity
\[
\E_{q_R}[R^*]-\E_{p_1}[R^*]
=
\E_{p_1}[R^*\omega_R]
-\E_{p_1}[R^*]\E_{p_1}[\omega_R]
=\Cov_{p_1}(R^*,\omega_R).
\]
This identity also appears in~\citet{paes2026theoretical}.
For a composition, $\omega_{g\circ\mathbf R}$ is
$\mathbf R$-measurable, so the covariance equals
\[
\Cov_{p_1}(\mu,\omega_{g\circ\mathbf R}).
\]
Under Assumption~\ref{ass:bounded}, expanding the density ratio
gives
\[
\E_{q_g}[R^*]-\E_{p_1}[R^*]
=\beta\Cov_{p_1}(\mu,g\circ\mathbf R)+O(\beta^2).
\]
Part~(iv) implies
\[
\Var_{p_1}(g\circ\mathbf R-\mu)
\le\mathrm{MSE}(g)-\bar v,
\]
with equality under Assumption~\ref{ass:means}.
Consequently,
\[
|\Cov_{p_1}(\mu,\mu-g\circ\mathbf R)|
\le
\sqrt{
\Var_{p_1}(\mu)(\mathrm{MSE}(g)-\bar v)
}
\]
by Cauchy--Schwarz. Substituting
\[
\Cov_{p_1}(\mu,g\circ\mathbf R)
=\Var_{p_1}(\mu)
-\Cov_{p_1}(\mu,\mu-g\circ\mathbf R)
\]
proves (ii).

\emph{(iii), using Assumption~\ref{ass:bounded}.}
The oracle tilt's gain is
$\beta\Var_{p_1}(R^*)+O(\beta^2)$, so
\begin{equation}
\label{eq:reward-leading}
\E_{q^*}[R^*]-\E_{q_g}[R^*]
=
\beta\left[
\bar v+\Cov_{p_1}(\mu,\mu-g\circ\mathbf R)
\right]+O(\beta^2).
\end{equation}
The same Cauchy--Schwarz bound proves (iii).
\end{proof}

The conditional mean minimizes recovery error and the
leading KL coefficient. It need not maximize expected
reward over arbitrary score scalings: replacing $\mu$
by $a\mu$ changes the leading gain to
$a\Var_{p_1}(\mu)$.

\begin{proof}[Proof of Corollary~\ref{cor:positive-gain}]
By Theorem~\ref{thm:value}(iv),
\[
\mathrm{MSE}(g)-\bar v<\Var_{p_1}(\mu).
\]
Consequently, the coefficient of $\beta$ in the gain lower
bound of Theorem~\ref{thm:value}(ii) is strictly positive.
For fixed $g$, this term dominates the $O(\beta^2)$ remainder
for all sufficiently small positive $\beta$.
\end{proof}

\paragraph{Proposition~\ref{prop:plhess} (PL curvature).}
For a fixed query and order,
\[
\mathcal L_S(f)
=\sum_{i=1}^{m-1}
\left[\log\sum_{l\in A_i}e^{f_l}-f_{\pi(i)}\right].
\]
Differentiating log-sum-exp gives the stated gradient
and Hessian. For any $w\in\mathbb R^m$,
\[
w^\top\left[\diag(p^{(i)})-p^{(i)}p^{(i)\top}\right]w
=\Var_{p^{(i)}}(w)\ge0.
\]
The first stage assigns positive probability to every
candidate, so this quadratic form vanishes only for
constant $w$. Every stage vanishes on constant vectors;
hence the summed Hessian has the stated null space
and rank.\hfill$\square$

\paragraph{Lemma~\ref{lem:pl-sc} (likelihood controls recovery).}
Under Assumptions~\ref{ass:moments}, \ref{ass:means},
and~\ref{ass:preferences}, set
$w=(g-g_{\PL})\circ\mathbf R$ and
$g_t=g_{\PL}+t(g-g_{\PL})$, $0\le t\le1$.
The segment is feasible by convexity of $\gG$. Minimality gives
\[
\left.\frac{d}{dt}\PL(g_t)\right|_{t=0}\ge0,
\]
not necessarily equality at a constrained boundary.
For $S=(x_1,\ldots,x_m)$ let $w_{S,j}=w(x_j)$ and
$\bar w_S=m^{-1}\sum_jw_{S,j}$.

Proposition~\ref{prop:plhess} bounds the likelihood's second
directional derivative below by its first-stage variance.
The range bound gives $p_j^{(1)}\ge e^{-\Delta}/m$, and hence
\[
\Var_{p^{(1)}}(w_S)
=\min_c\sum_jp_j^{(1)}(w_{S,j}-c)^2
\ge
\frac{e^{-\Delta}}m
\sum_j(w_{S,j}-\bar w_S)^2.
\]
For independent items,
\[
\E_S\left[
m^{-1}\sum_j(w_{S,j}-\bar w_S)^2
\right]
=(1-1/m)\Var_{p_1}(w).
\]
Taking expectations over both $S$ and $\pi$ gives
\[
\frac{d^2}{dt^2}\PL(g_t)
\ge c_\Delta\Var_{p_1}(w).
\]
The bounded scores and directions justify differentiating
under these expectations. Taylor's integral formula then yields
\begin{align*}
\PL(g)-\PL(g_{\PL})
&=
\left.\frac{d}{dt}\PL(g_t)\right|_{t=0}
+\int_0^1(1-t)\frac{d^2}{dt^2}\PL(g_t)\,dt\\
&\ge
\tfrac{c_\Delta}{2}\Var_{p_1}(w),
\end{align*}
proving the claim.\hfill$\square$

\begin{mdframed}[style=MyFrame2]
\thmpref*
\end{mdframed}

\begin{proof}[Proof of Theorem~\ref{thm:pref}]
First use Assumptions~\ref{ass:moments}, \ref{ass:means},
and~\ref{ass:preferences}.
Theorem~\ref{thm:value}(iv), the $L^2$ triangle inequality,
mean alignment, and Lemma~\ref{lem:pl-sc} give
\begin{align*}
\sqrt{\mathrm{MSE}(g)-\bar v}
&=\|g\circ\mathbf R-\mu\|_{L^2(p_1)}\\
&\le
\|g\circ\mathbf R-g_{\PL}\circ\mathbf R\|_{L^2(p_1)}
+\|g_{\PL}\circ\mathbf R-\mu\|_{L^2(p_1)}\\
&\le
\sqrt{
\frac{2}{c_\Delta}
\bigl(\PL(g)-\PL(g_{\PL})\bigr)
}
+\sqrt{\Var_{p_1}(g_{\PL}\circ\mathbf R-\mu)}.
\end{align*}
Squaring and adding $\bar v$ proves
$\mathrm{MSE}(g)\le B(g)$.
The nonnegative excess PL risk is zero at $g_{\PL}$, so
$B$ is minimized there, with
\[
B(g_{\PL})
=\bar v+\Var_{p_1}(g_{\PL}\circ\mathbf R-\mu)
=\mathrm{MSE}(g_{\PL}).
\]
This proves (iv), not minimization of the actual quantities
bounded by $B$.

For (i)--(iii), additionally assume the oracle bound in
Assumption~\ref{ass:bounded}. The common finite mean and range
width already imply
\[
|g\circ\mathbf R|
\le|\E_{p_1}[R^*]|+\Delta
\]
almost surely for every $g\in\gG$.
Theorem~\ref{thm:value} therefore applies to each fixed $g$.
Substituting $\mathrm{MSE}(g)\le B(g)$ into its bounds proves
(i)--(iii): the gain lower bound decreases, and the KL and gap
upper bounds increase, with recovery error.
\end{proof}

The discrepancy need not vanish under conditional oracle
variation: mixing PL laws over oracle values need not produce
PL probabilities with utility $\mu$. A sufficient realizable
case is $\bar v=0$ and a conditional-mean composition in
$\gG$. Then $R^*=\mu$ almost surely, and the excess conditional
cross-entropy over that score is a KL divergence between ranking
laws. Thus $\mu$ minimizes PL risk;
Lemma~\ref{lem:pl-sc} and mean alignment give
$g_{\PL}\circ\mathbf R=\mu$.

\paragraph{Proposition~\ref{prop:klopt} (exact KL floor).}
The KL chain rule and Lemma~\ref{lem:invariance} give
\[
\KL(q^*\|q_g)
=
\KL(q^*(\mathbf R)\|q_g(\mathbf R))
+
\E_{\mathbf r\sim q^*(\mathbf R)}
\left[
\KL\bigl(
q^*(\cdot\mid\mathbf r)\|p_1(\cdot\mid\mathbf r)
\bigr)
\right].
\]
The second term is $\mathcal K_{\min}$.
The surrogate marginals have density ratios relative to
$p_1(\mathbf R)$ proportional to
$\E_{p_1}[e^{\beta R^*}\mid\mathbf R]$ and
$e^{\beta g\circ\mathbf R}$, respectively.
They agree exactly for the displayed $\gopt$, up to an
additive constant, proving optimality.

The conditional density ratio of $q^*$ relative to $p_1$ is
\[
\frac{
e^{\beta R^*}
}{
\E_{p_1}[e^{\beta R^*}\mid\mathbf R]
}.
\]
Therefore $\mathcal K_{\min}=0$ if and only if this ratio is
one almost surely, equivalently $R^*$ is
$\mathbf R$-measurable, or $\bar v=0$.

Assumption~\ref{ass:bounded} gives the uniform conditional
cumulant expansion
\[
\log\E_{p_1}[e^{\beta R^*}\mid\mathbf R]
=
\beta\mu+\tfrac{\beta^2}{2}v+O(\beta^3).
\]
Thus $\gopt\circ\mathbf R=\mu+O(\beta)$.
Since
$|\gopt\circ\mathbf R|\le\|R^*\|_\infty$,
the constants in the Taylor remainder bounds of
Lemma~\ref{lem:err2kl} are independent of $\beta$, so that
expansion also applies to $\gopt$.
Orthogonality gives
\[
\Var_{p_1}(R^*-\gopt\circ\mathbf R)
=\bar v+O(\beta^2),
\]
so
$\mathcal K_{\min}=\beta^2\bar v/2+O(\beta^3)$.
For a fixed $g$, the same orthogonal decomposition and
Lemma~\ref{lem:err2kl} give the final expansion.
\hfill$\square$

\paragraph{Corollary~\ref{cor:mu} (fourth-order excess).}
Dividing the conditional cumulant expansion by $\beta$ proves
\[
\gopt\circ\mathbf R
=\mu+\beta v/2+O(\beta^2)
\]
uniformly. The exact decomposition, equality of the surrogate
marginals of $q^*$ and $q_{\gopt}$, and level-set invariance
for $q_{\gopt}$ and $q_\mu$ give
\[
\KL(q^*\|q_\mu)-\mathcal K_{\min}
=\KL(q_{\gopt}\|q_\mu).
\]
Twice differentiating the log normalizer along the
interpolating score, and integrating, yields the exact identity
\[
\KL(q_{\gopt}\|q_\mu)
=
\beta^2\int_0^1 t\,
\Var_{q_{(1-t)\mu+t(\gopt\circ\mathbf R)}}
(\gopt\circ\mathbf R-\mu)\,dt.
\]
For completeness, if the log normalizer along this path is
$F(t)$, the KL is
\[
F'(1)-F(1)+F(0)=\int_0^1tF''(t)\,dt,
\]
and $F''$ is $\beta^2$ times the displayed variance.
The tilted density is $1+O(\beta)$ relative to $p_1$
uniformly in $t$. Since
$\gopt\circ\mathbf R-\mu=\beta v/2+O(\beta^2)$
uniformly, the integrand's variance equals
\[
\beta^2\Var_{p_1}(v)/4+O(\beta^3).
\]
Integration gives
$\beta^4\Var_{p_1}(v)/8+O(\beta^5)$.
This argument controls the remainder for the shrinking score
difference; the generic $O(\beta^3)$ remainder of
Lemma~\ref{lem:err2kl} alone would not suffice.
\hfill$\square$

\paragraph{Corollary~\ref{cor:corr} (correlation ceiling).}
Conditional expectation gives
\[
\Cov_{p_1}(g\circ\mathbf R,R^*)
=\Cov_{p_1}(g\circ\mathbf R,\mu).
\]
Cauchy--Schwarz and total variance imply
\[
|\PCC_{p_1}(g\circ\mathbf R,R^*)|
\le
\sqrt{
\frac{\Var_{p_1}(\mu)}{\Var_{p_1}(R^*)}
}
=
\sqrt{
1-\frac{\bar v}{\Var_{p_1}(R^*)}
}.
\]
If $\mu$ is nonconstant, its own score attains equality.
If it is constant, the covariance vanishes for every
admissible nonconstant score. The single-surrogate ceiling
follows by the same argument conditioning on $R_j$.
\hfill$\square$

\paragraph{Corollary~\ref{cor:standardized}.}
Lemma~\ref{lem:err2kl} applied to the standardized scores
gives the expansion, since
$\Var_{p_1}(z_{R^*}-z_f)=2[1-\PCC_{p_1}(R^*,f)]$.
Optimality follows from Corollary~\ref{cor:corr}.
Conditional expectation gives
\[
\rho_{\mathrm{ceil}}-\PCC_{p_1}(R^*,f)
=\frac{\rho_{\mathrm{ceil}}}{2}
\|z_f-z_\mu\|_{L^2(p_1)}^2.
\]
By mean alignment and the reverse triangle inequality,
\[
\|z_f-z_\mu\|_{L^2(p_1)}
\le \frac{2}{s_\mu}\|f-\mu\|_{L^2(p_1)}.
\]
The recovery decomposition in Theorem~\ref{thm:value}
completes the bound.

\paragraph{Corollary~\ref{cor:mono}.}
Every composition of $\mathbf R_S$ is also a composition of
$\mathbf R$. The full-vector population MSE infimum therefore
cannot be larger; the two infima are $\bar v$ and $\bar v_S$
by Theorem~\ref{thm:value}(iv).
The same inclusion for the exact KL infima in
Proposition~\ref{prop:klopt} proves
$\mathcal K_{\min}\le\mathcal K_{\min}(S)$
for every $\beta>0$.
The remaining weak comparisons follow by substitution in
the correlation ceiling and the conditional-mean reward
coefficients.\hfill$\square$

\clearpage
\section{Toy Examples}
\label{app:toy}

\subsection{Value feedback}

\looseness=-1
\xhdr{Setup}
We use a two-dimensional example to isolate the benefit
of a nonlinear composition. Let $x=(u,v)$,
$p_1=\mathcal N(0,4I)$, $p_0=\mathcal N(0,I)$,
and $R_1(x)=u$, $R_2(x)=v$. The oracle is
$R^*(x)=|u|-|v|$, giving
\begin{equation}
q^*(x)\propto p_1(x)e^{\beta R^*(x)},\qquad \beta=1.
\label{eq:toy-qstar}
\end{equation}
The target is bimodal along the $u$-axis, and no affine
composition of the surrogates represents $R^*$ exactly.
Although the oracle is unbounded, the Gaussian prior
gives a finite tilt normalizer. This example therefore
extends beyond the bounded-reward assumptions of the
main results.

\looseness=-1
\xhdr{Fitting}
We use $n=400$ prior samples to fit linear and MLP
compositions with the value loss in~\cref{eq:mse-loss}.
The linear composition uses ridge regression on
standardized surrogates. The MLP has two hidden layers
of $64$ GELU units and is trained with Adam, using a
held-out quarter of the feedback set for early stopping.
The predicted reward fields are compared with the
oracle in~\cref{fig:toy-recovery}.

\begin{figure}[htbp]
  \centering
  \includegraphics[width=\linewidth]
    {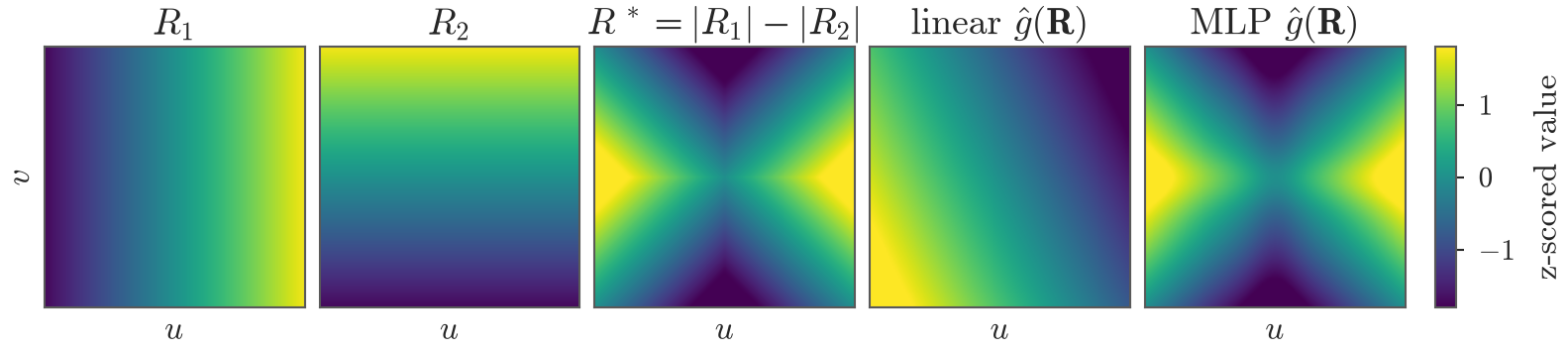}
  \caption{\small\looseness=-1
  \textbf{Value-toy reward recovery.}
  Surrogate fields, oracle reward, and fitted linear
  and MLP compositions.}
  \label{fig:toy-recovery}
\end{figure}

\looseness=-1
\xhdr{Steering}
We steer the distilled meta flow map with MFM using
each fitted composition or an individual surrogate.
The MLP composition gives lower estimated $W_1$ and KL
to $q^*$ ($0.51$ and $0.05$) than the linear composition
($1.49$ and $0.60$), $R_1$ ($1.85$ and $1.06$), or
$R_2$ ($2.32$ and $0.91$), consistent with the
distributions in~\cref{fig:toy-steer}.

\begin{figure}[htbp]
  \centering
  \includegraphics[width=\linewidth]
    {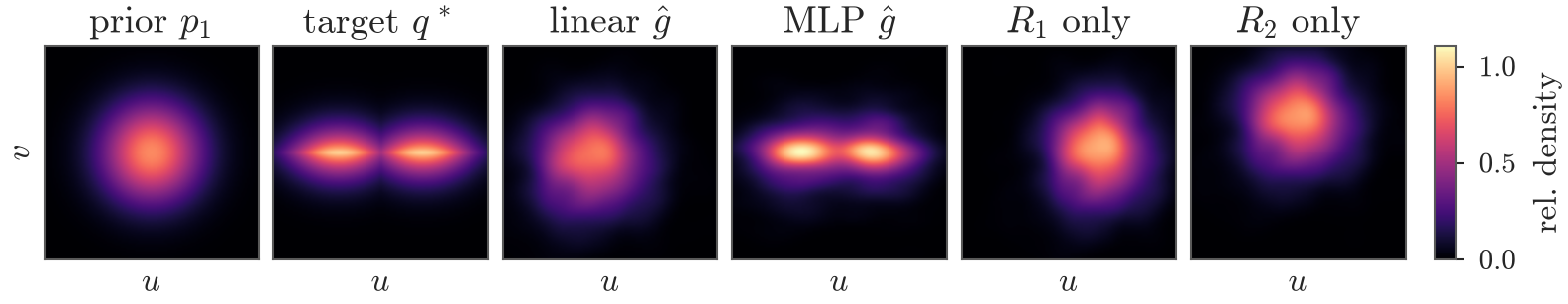}
  \caption{\small\looseness=-1
  \textbf{Value-toy steering with MFM.}
  Prior and target distributions, followed by samples
  generated using the linear composition, MLP composition,
  and individual surrogates.}
  \label{fig:toy-steer}
\end{figure}

\subsection{Preference feedback}

\looseness=-1
We next learn from PL rankings for a multimodal oracle
and compare FKS, FMRG, MFM, and It\^o steering.

\looseness=-1
\xhdr{Setup}
On $\mathcal X=\mathbb R^2$, with coordinates $x=(u,v)$
and prior $p_1=\mathcal N(0,4I)$, the oracle combines
three Gaussian bumps:
\[
R^*(x)=\sum_{c=1}^{3}h_c\,\phi(x;\mu_c,w),
\]
where $\mu_c$ are the centers, $h_c$ control their
amplitudes, and $w$ is the width parameter.
The resulting target
$q^*(x)\propto p_1(x)\exp(\beta R^*(x))$ at $\beta=1$
is multimodal~(\cref{fig:gmm-context}).
Reference samples are obtained by rejection sampling.
The three surrogates $R_1,R_2,R_3$ are corrupted views
of the oracle, with shifted centers, rescaled heights,
and noise localized around the modes.

\begin{figure}[htbp]
  \centering
  \includegraphics[width=\linewidth]
    {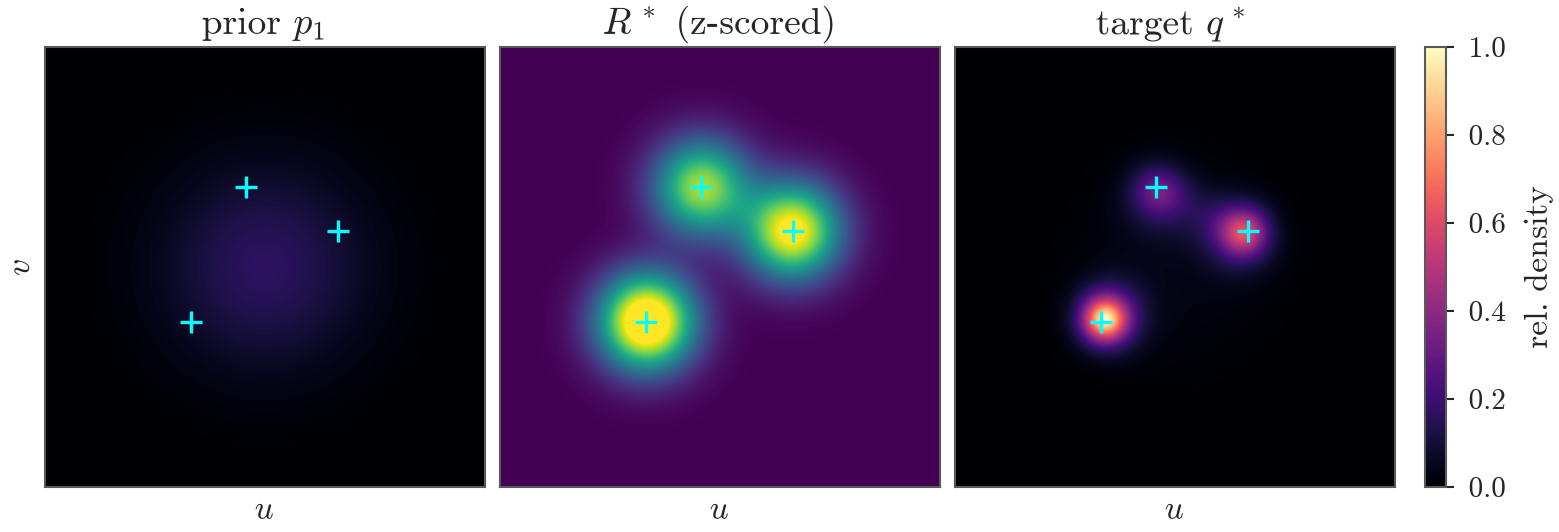}
  \caption{\small\looseness=-1
  \textbf{Preference-toy setup.}
  Prior, oracle reward, and target tilt.
  Reward-mode centers are marked ``$+$''.
  Density panels here and in~\cref{fig:gmm-steered,fig:gmm-paired}
  use a common scale relative to the peak of $q^*$.}
  \label{fig:gmm-context}
\end{figure}

\looseness=-1
\xhdr{Fitting}
Preference feedback consists of a PL ranking of
$m=50$ independent prior samples, generated with
utilities $\lambda R^*(x)$ and $\lambda=10$.
Linear and MLP compositions are fitted using the
PL loss in~\cref{eq:pl-loss}; correlation weights are
computed from the observed ranking.
Fitting uses these rankings rather than numerical
oracle rewards. The fitted fields are shown
in~\cref{fig:gmm-recovery}.

The fitted PL utility is on the $\lambda R^*$ scale.
To express the corresponding composition in oracle-reward
units, it is divided by $\lambda$; we denote this
calibrated composition by $\hat g$.
We repeat the experiment over five ranking draws.

\begin{figure}[htbp]
  \centering
  \includegraphics[width=\linewidth]
    {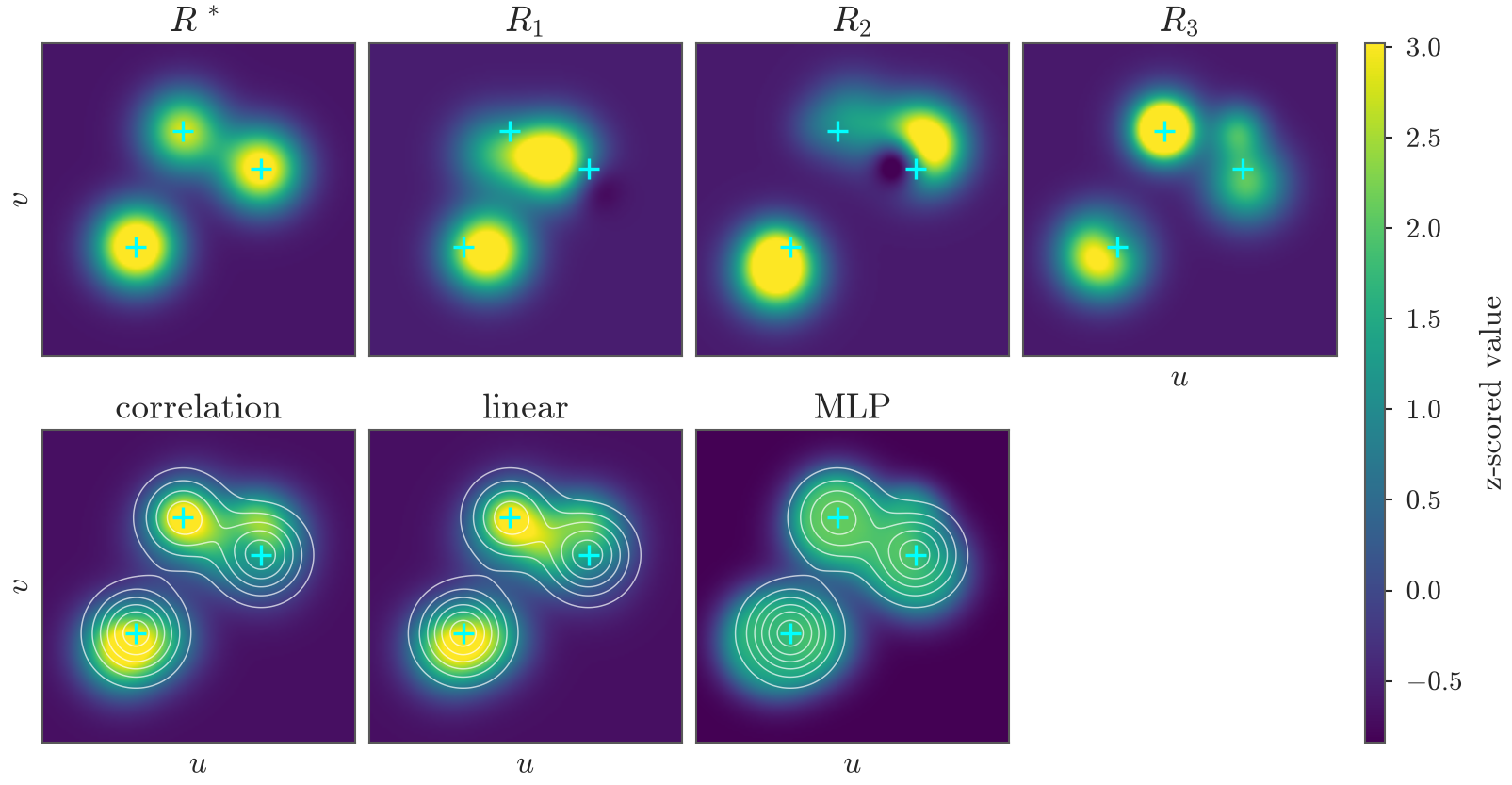}
  \caption{\small\looseness=-1
  \textbf{Preference-toy reward recovery.}
  Oracle and surrogate fields, together with correlation,
  linear, and MLP compositions. Oracle contours are
  overlaid on the fitted fields.}
  \label{fig:gmm-recovery}
\end{figure}

\looseness=-1
\xhdr{Steering}
We supply the same fitted composition to FKS, FMRG,
MFM, and It\^o steering~(\cref{fig:gmm-steered}).
Its ideal tilt $q_{\hat g}$ can differ from $q^*$,
while look-ahead, discretization, and particle
approximations introduce additional sampling error.
FMRG, MFM, and It\^o are reported at the evaluated
setting with the lowest reference-based KL; this is
retrospective sweep selection. FKS uses a fixed
configuration. The target remains fixed at $\beta=1$
throughout the guidance sweeps.

The metrics rank the steerers differently
in~\cref{tab:gmm-steer}. FMRG and achieves the lowest
KL ($0.21$), while FKS achieves the lowest $W_1$
($0.55$). FMRG also gives the largest mean oracle
reward ($0.55$), which alone does not establish
closer agreement with the full target distribution.

\begin{figure*}[htbp]
  \centering
  \includegraphics[width=\linewidth]
    {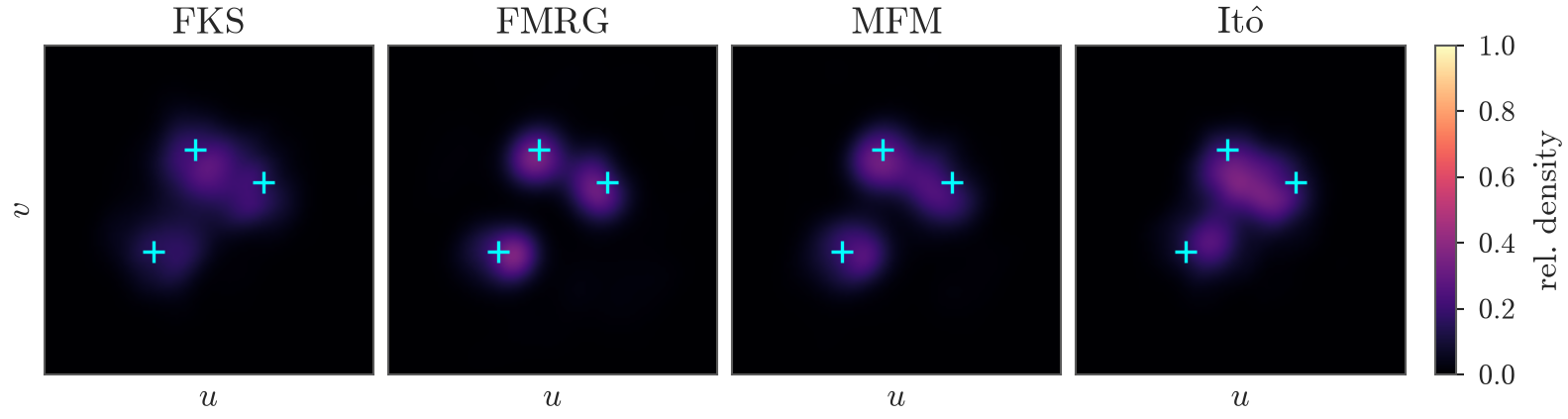}
  \caption{\small\looseness=-1
  \textbf{Preference-toy steering.}
  FKS, FMRG, MFM, and It\^o outputs under the fitted
  composition, using the density scale
  of~\cref{fig:gmm-context}.}
  \label{fig:gmm-steered}
\end{figure*}

\begin{table}[htbp]
  \centering
  \caption{\small\looseness=-1
  \textbf{Preference-toy steering.}
  Estimated distances to $q^*$ and mean oracle reward,
  averaged over five ranking draws.
  Bold marks the best value in each column.}
  \label{tab:gmm-steer}
  \begin{tabular}{lccc}
    \toprule
    Steerer
    & $W_1\downarrow$
    & $\KL(q^*\|\cdot)\downarrow$
    & Mean $R^*\uparrow$ \\
    \midrule
    FKS   & $\mathbf{0.55}$ & $0.32$          & $0.44$ \\
    FMRG  & $0.69$          & $\mathbf{0.21}$ & $\mathbf{0.55}$ \\
    MFM   & $0.57$          & $0.26$ & $0.49$ \\
    It\^o & $0.68$          & $0.43$          & $0.48$ \\
    \bottomrule
  \end{tabular}
\end{table}

\looseness=-1
\xhdr{Comparison with individual surrogates}
Using FMRG, the fitted composition achieves lower
$W_1$ and KL than any individual surrogate
in~\cref{fig:gmm-paired,tab:gmm-compose}.
Each reward uses its evaluated KL-minimizing guidance
strength, with the oracle target held fixed.
This comparison demonstrates the benefit of composition
within the evaluated sweeps.

\begin{figure}[htbp]
  \centering
  \includegraphics[width=\linewidth]
    {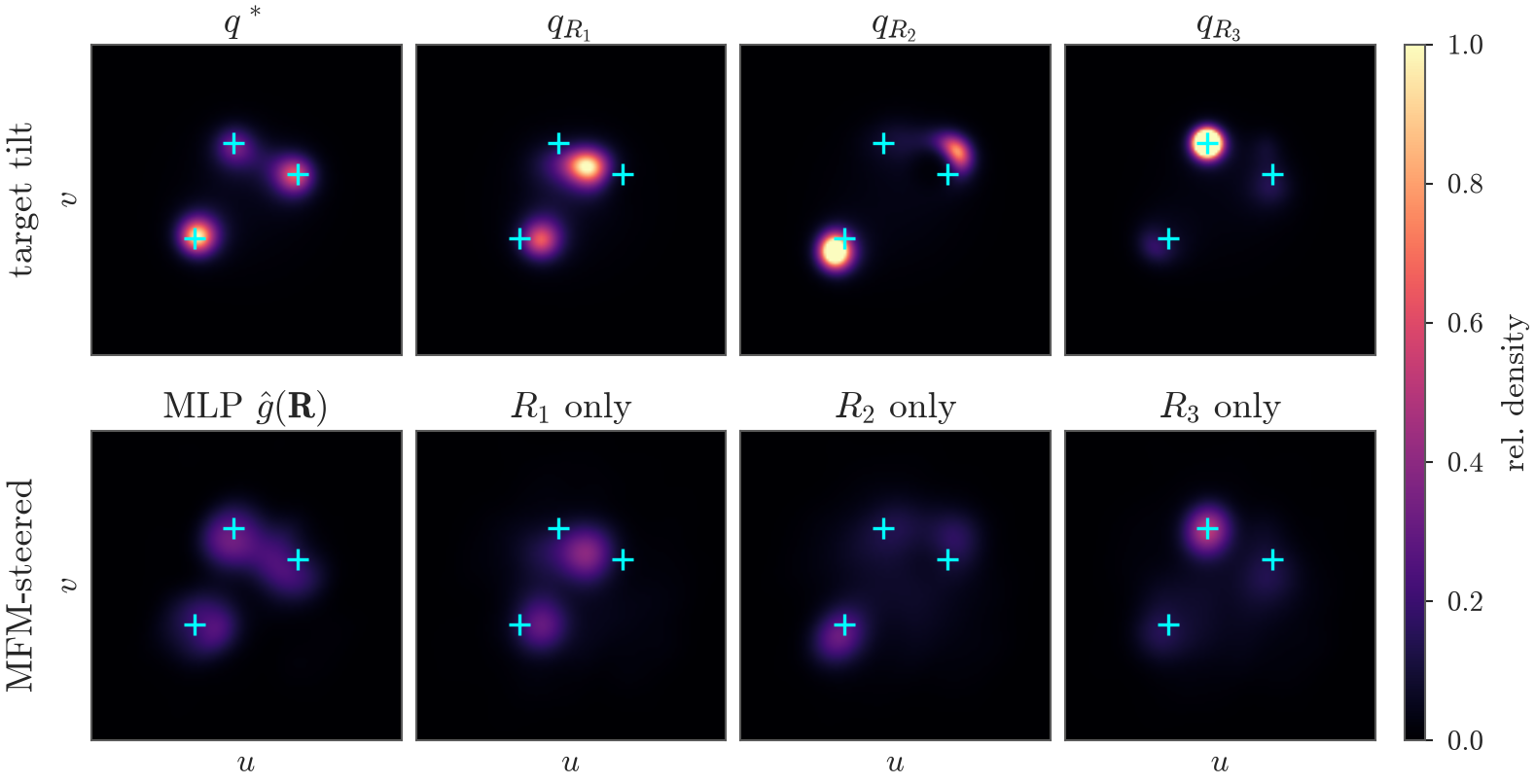}
  \caption{\small\looseness=-1
  \textbf{Composition and individual surrogates.}
  Top: oracle and single-surrogate tilts.
  Bottom: FMRG outputs using the fitted composition
  and each surrogate separately.
  Densities use the scale of~\cref{fig:gmm-context}.}
  \label{fig:gmm-paired}
\end{figure}

\begin{table}[htbp]
  \centering
  \caption{\small\looseness=-1
  \textbf{Preference-toy reward comparison.}
  FMRG using the fitted composition or an individual
  surrogate, averaged over five ranking draws.
  Lower is better; bold marks the best result.}
  \label{tab:gmm-compose}
  \begin{tabular}{lcc}
    \toprule
    Steering reward
    & $W_1\downarrow$
    & $\KL(q^*\|\cdot)\downarrow$ \\
    \midrule
    $\hat g(\mathbf R)$
    & $\mathbf{0.69}$ & $\mathbf{0.21}$ \\
    $R_1$                             & $0.79$          & $0.52$          \\
    $R_2$                              & $0.73$          & $0.47$          \\
    $R_3$                             & $0.99$          & $0.46$          \\
    \bottomrule
  \end{tabular}
\end{table}

\clearpage
\section{Qualitative examples}
\label{app:qual}

\subsection{Language Generation}

\newcommand{\iss}[1]{\textcolor{red}{#1}}
\begin{table}[h!]
\centering
\footnotesize
\setlength{\tabcolsep}{6pt}
\renewcommand{\arraystretch}{1.3}
\begin{tabular}{@{}lp{0.855\linewidth}@{}}
\toprule
\textbf{Method} & \textbf{Generated text} \\
\midrule

prior $p_{1,\phi}$
& [CLS] june. [CLS] for the moment, some muslims in america will be told to \iss{thinker} the
\iss{united states, the united states} also of course for the same for being 4, 000. [CLS] set a
hands - on date for one local resident after this three of each college that, in spades, she is
the only way when i move it to show\ldots \\
\midrule

BoN
& \iss{[CLS] [CLS]} nine officers involved in the trade said they expected 7 online, up from
their peak at \pounds 25, 500 for up to 1, 000 people to be arrested in 2006. [CLS] in the same,
it was spent about 10 minutes in the previous week\ldots \\
\addlinespace

sBoN
& [CLS] are going though not to get " women to help out " is always just thousands of childcare
- - and questions about whether orwell should not have set up but if there are simply no social
issues, we co - authors of thousands\ldots \\
\midrule

uniform
& [CLS] to \iss{[CLS] [CLS]} this \iss{[CLS]} ' s better earnings \iss{[CLS]}, or \$ 1, 000
\iss{[CLS]}. to match \iss{[CLS]}. [CLS] and will he be \iss{[CLS]} to \iss{[CLS]} once this
weekend, which will \iss{[CLS]} that it along further. [CLS] japan ' s idea \iss{[CLS]} than
\iss{[CLS]} in the \iss{[CLS] [CLS]} city \iss{withinus}\ldots \\
\addlinespace

BSS
& [CLS] leaders. [CLS] in the senate, the 2010 parliament road will be political by long, the
united states, the gas and also of announced for the public - private distribution of food.
[CLS] i last emailed on thursday on new zealand for being " six of report " that, in
\iss{bads}, she is the only survivor when i move it to li\ldots \\
\addlinespace

SLOP \emph{(draw 1)}
& \iss{[CLS] himself up at the masters again to show a time and can work on training this week.
[CLS] the end of the recession, he signed up to cutting 2, 000 jobs from 2030 by \$ 2. 1. 5.
[CLS] so might actually see your loan based on interest rates\ldots} \\
\addlinespace[2pt]

SLOP \emph{(draw 2)}
& \iss{[CLS] himself up at the masters again to show a time and can work on training this week.
[CLS] the end of the recession, he signed up to cutting 2, 000 jobs from 2030 by \$ 2. 1. 5.
[CLS] so might actually see your loan based on interest rates\ldots} \\
\addlinespace

ParetoFlow
& [CLS] a year. [CLS] " \iss{there is none of} the very how they, " she will. [CLS]
\iss{there is none of} this one, but the fact that that 100 lobby phone ' s office is not going
back to war, what did, have that of the president? [CLS] about 19s out of the then back of two -
time champion\ldots \\
\midrule

MARTINI
& [CLS] parents. [CLS] for the times, some parents in staff will be told to do, the has
investment, the act and also of out for the world for last 4 of england. [CLS] ( did this
blacks ) this one is a almost and industry of the college that, in spades, she is the only
business when i move it to off\ldots \\
\addlinespace

\bottomrule
\end{tabular}
\caption{\small\textbf{Language-generation examples.}
Red highlights repetitions and malformed spans in the
displayed excerpts. The two SLOP outputs are identical;
ParetoFlow repeats a phrase within a sample.}
\label{tab:lang-qualitative}
\end{table}

\subsection{ImageNet}

\label{app:qual-imagenet}
\Cref{fig:in-qual-c145-011,fig:in-qual-c294-008,fig:in-qual-c347-012,fig:in-qual-c555-016,fig:in-qual-c779-015}
show five of the seven class--prompt pairs in the main table.
Rows compare the unsteered prior with FMRG steering using
\nameshort's linear composition, ReNO weights, and BSS
in the value regime at $n=32$.
Each row contains that method's highest-oracle-reward samples.
This oracle-based selection is used only for illustration;
the grids do not describe typical output quality or an
oracle-free deployment selection rule.

\begin{figure}[h!]
  \centering
  \includegraphics[width=0.85\linewidth]{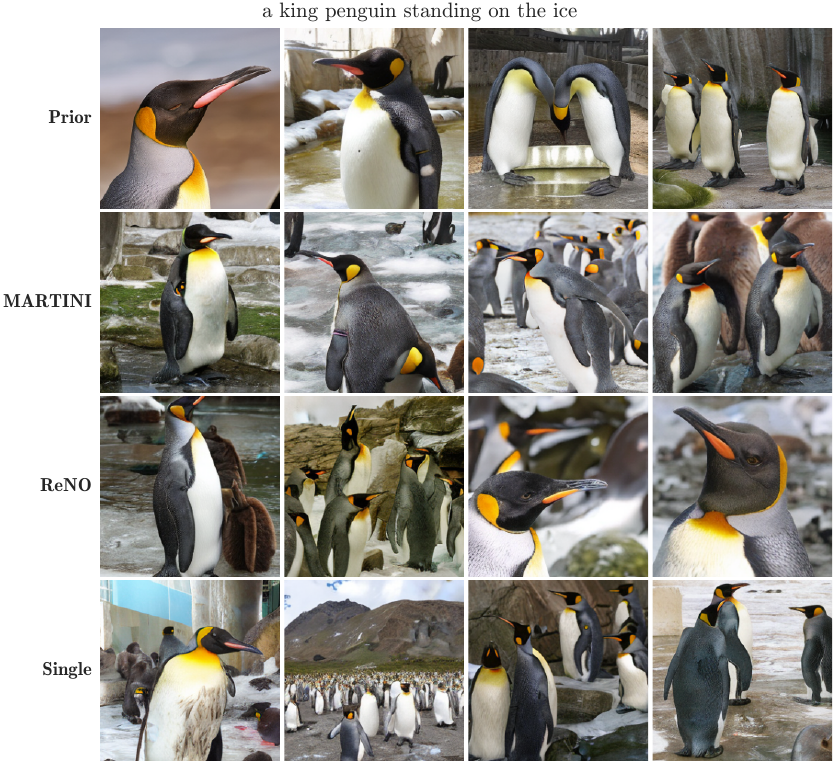}
  \caption{\small\textbf{ImageNet: king penguin.}
Prompt: ``a king penguin standing on the ice.''}
  \label{fig:in-qual-c145-011}
\end{figure}

\begin{figure}[h!]
  \centering
  \includegraphics[width=0.85\linewidth]{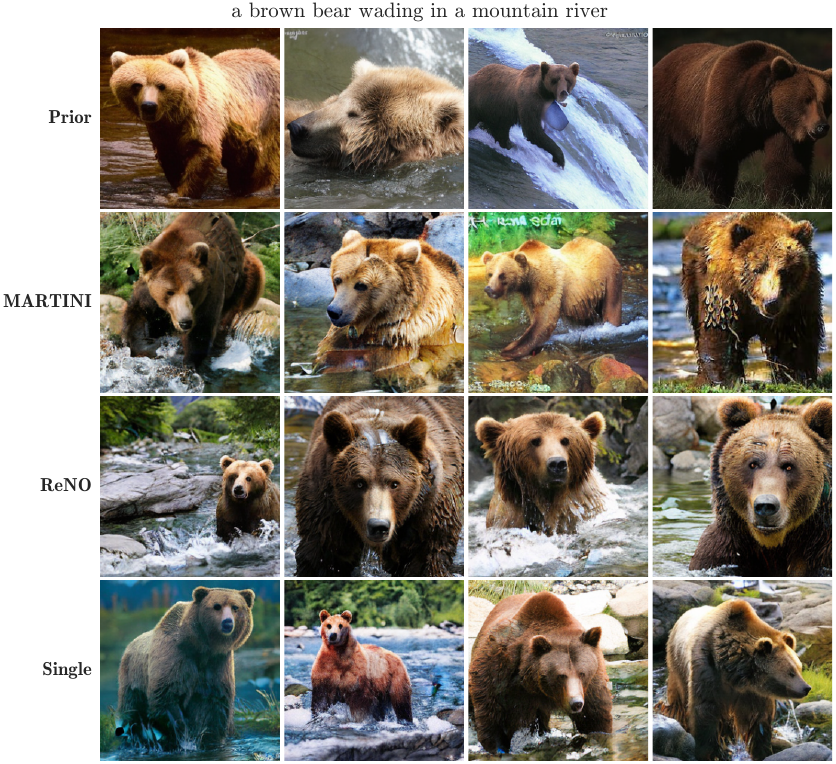}
  \caption{\small\textbf{ImageNet: brown bear.}
Prompt: ``a brown bear wading in a mountain river.''}
  \label{fig:in-qual-c294-008}
\end{figure}

\begin{figure}[h!]
  \centering
  \includegraphics[width=0.85\linewidth]{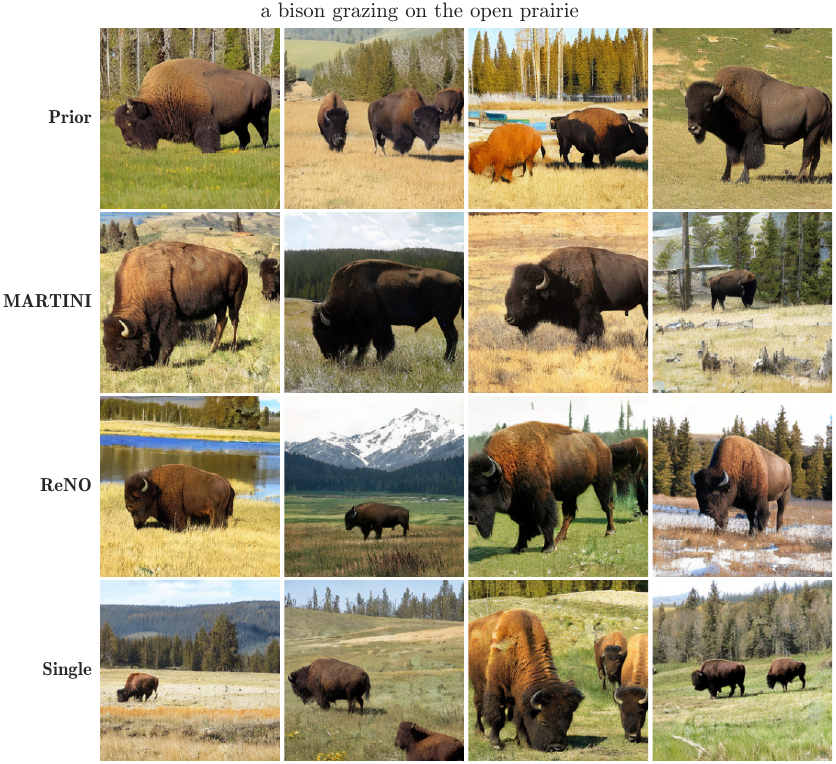}
  \caption{\small\textbf{ImageNet: bison.}
Prompt: ``a bison grazing on the open prairie.''}
  \label{fig:in-qual-c347-012}
\end{figure}

\begin{figure}[h!]
  \centering
  \includegraphics[width=0.85\linewidth]{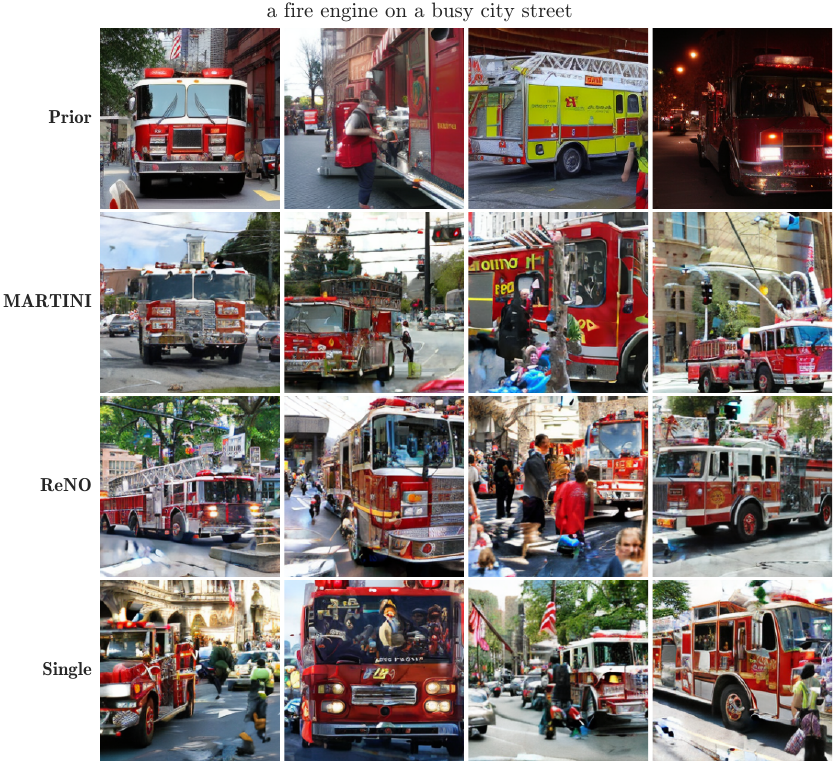}
  \caption{\small\textbf{ImageNet: fire engine.}
Prompt: ``a fire engine on a busy city street.''}
  \label{fig:in-qual-c555-016}
\end{figure}

\begin{figure}[h!]
  \centering
  \includegraphics[width=0.85\linewidth]{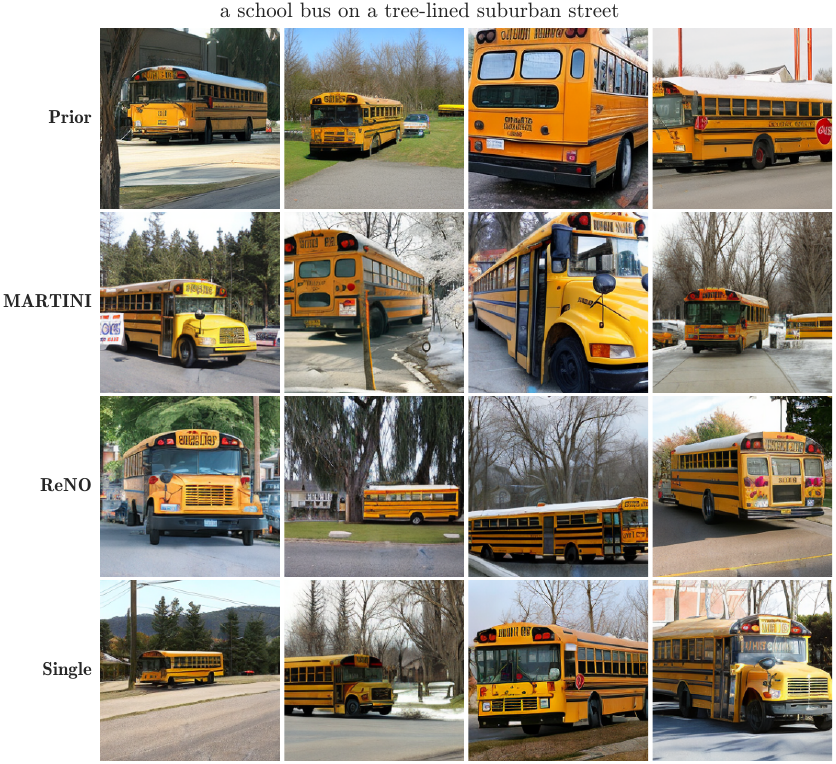}
  \caption{\small\textbf{ImageNet: school bus.}
Prompt: ``a school bus on a tree-lined suburban street.''}
  \label{fig:in-qual-c779-015}
\end{figure}

\clearpage
\subsection{Protein cofolding}

\Cref{fig:prot_gcorr15_top20_barebones} shows 20 complexes, comparing ground truth, ligand crystal, the baseline DeCAF, and the MARTINI-best [gCorrelation]. 

\begin{figure}[h!]
    \centering
    \includegraphics[width=0.98\linewidth]{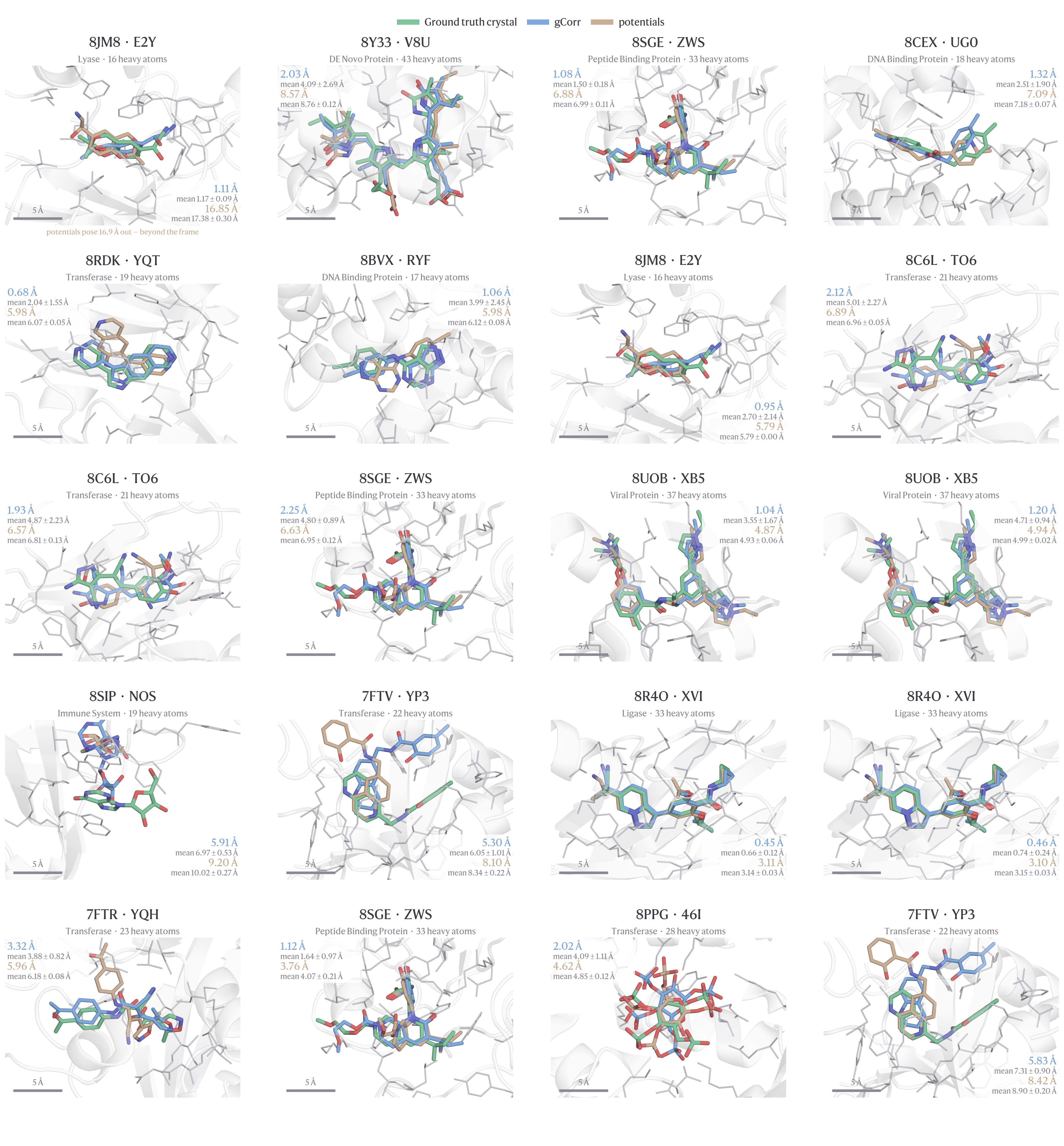}
    \caption{Top 20 complexes that outperform the baseline based on RMSD from the ground truth ligand crystal (green), the baseline DeCAF (orange), and the MARTINI-best [gCorrelation] (blue). The mean and standard deviation of the RMSD for that complex is show across the 5 seeds used to generate the particles. All steering runs here used $k=15$ particles.}
    \label{fig:prot_gcorr15_top20_barebones}
\end{figure}

\clearpage
\section{Relation to multi-objective optimisation}
\label{app:pareto}

Multi-objective optimization treats $R_1,\ldots,R_k$ as
objectives to be balanced. Here they provide imperfect
information about a single oracle reward $R^*$.
We formalize why surrogate Pareto efficiency need not
identify candidates preferred by the oracle.

For rewards to be maximized, $y$ dominates $x$ if
$R_j(y)\ge R_j(x)$ for every $j$, with at least one
strict inequality. The nondominated feasible candidates
form the Pareto set $\mathcal P$; their reward vectors
form the Pareto front $\mathbf R(\mathcal P)$.

\begin{mdframed}[style=MyFrame2]
\begin{lemma}[Increasing compositions are Pareto efficient]
\label{lem:monotone_g}
If $g$ is coordinatewise strictly increasing, every
maximizer of $g\circ\mathbf R$ over the feasible domain
belongs to the surrogate Pareto set $\mathcal P$.
\end{lemma}
\end{mdframed}

This ordering need not agree with the oracle:

\begin{mdframed}[style=MyFrame2]
\begin{lemma}[The oracle optimum can be dominated]
\label{lem:offront}
There exist a base law and rewards with $\bar v>0$ whose
unique oracle maximizer is strictly dominated in every surrogate.
\end{lemma}
\end{mdframed}

\xhdr{Example}
Let $x=(a,b)\in[-1,1]^2$ and
$p_1=\operatorname{Unif}([-1,1]^2)$, with
\[
R_1(x)=-(a+1)^2-b^2,\qquad
R_2(x)=-(a-1)^2-b^2,\qquad
R^*(x)=-3(a^2+b^2)+2b-3.
\]
The surrogate Pareto set is the segment
$\mathcal P=[-1,1]\times\{0\}$, whereas the oracle
has the unique maximizer $x^\star=(0,1/3)$.
The point $(0,0)$ strictly dominates $x^\star$ on both
surrogates, but has lower oracle reward:
\[
R_1(0,0)=R_2(0,0)=-1>-10/9
=R_1(x^\star)=R_2(x^\star),
\]
\[
R^*(0,0)=-3<-8/3=R^*(x^\star).
\]
\Cref{fig:moo-comparison} illustrates the distinction.

Removing the monotonicity restriction allows compositions
to favor points outside the surrogate Pareto set, but does
not recover missing information. Here
$\mathbf R(a,b)=\mathbf R(a,-b)$, so every composition
assigns the same score to $x^\star$ and $(0,-1/3)$ even
though their oracle rewards differ.

\subsection{Proofs}
\paragraph{Lemma~\ref{lem:monotone_g}.}
If a maximizer $x$ were dominated by $y$, coordinatewise
strict increase would give
$g(\mathbf R(y))>g(\mathbf R(x))$, a contradiction.
\hfill$\square$

\paragraph{Lemma~\ref{lem:offront}.}
In the example, the surrogate values determine
\[
a=(R_2-R_1)/4,\qquad
b^2=-(R_1+R_2)/2-1-a^2.
\]
Under the uniform base law, the sign of $b$ is conditionally
symmetric. Hence
\[
\mu(x)=-3(a^2+b^2)-3,\qquad
R^*(x)-\mu(x)=2b,\qquad
\bar v=4\E_{p_1}[b^2]=4/3>0.
\]
Any point with $b\ne0$ is strictly dominated by $(a,0)$.
For a point on the segment, its distances to $(-1,0)$
and $(1,0)$ sum to $2$. A dominating point would have
a strictly smaller sum, contradicting the triangle
inequality. Thus $\mathcal P$ is exactly this segment.
Finally,
\[
R^*(a,b)=-3a^2-3(b-1/3)^2-8/3
\]
has the unique maximizer stated above, which $(0,0)$
strictly dominates in both surrogates.\hfill$\square$


\end{document}